\documentclass[twoside]{article}

\usepackage[preprint]{aistats2026}
\usepackage[round]{natbib}

\usepackage{algorithm}
\usepackage{algpseudocode}
\usepackage{hyperref}
\usepackage{url}
\usepackage{graphicx}
\usepackage{amsthm} 
\usepackage{amsmath,amssymb,thmtools}
\usepackage{xcolor}
\newtheorem{theorem}{Theorem}
\newtheorem{lemma}{Lemma}
\newtheorem{assumption}{Assumption}

\DeclareMathOperator*{\tr}{tr}
\let\bs\boldsymbol

\DeclareMathOperator*{\argmin}{argmin}

\usepackage{booktabs}
\usepackage{pifont}
\newcommand{\cmark}{\ding{51}}
\newcommand{\xmark}{\ding{55}}

\newcommand{\Wsig}[1]{\cmark\,(#1)}
\newcommand{\Wmar}[1]{\ensuremath{\approx}\,(#1)}
\newcommand{\Wns}[1]{\xmark\,(#1)}
\newcommand{\Wna}{--}

\begin{document}

%

%

\runningtitle{NeST-BO}

\twocolumn[

\aistatstitle{NeST-BO: Fast Local Bayesian Optimization via Newton-Step Targeting of Gradient and Hessian Information}

\aistatsauthor{ Wei-Ting Tang \And Akshay Kudva \And Joel A. Paulson }

\aistatsaddress{ University of Wisconsin-Madison \And The Ohio State University \And University of Wisconsin-Madison } ]

\begin{abstract}
Bayesian optimization (BO) is effective for expensive black-box problems but remains challenging in high dimensions. We propose NeST-BO, a curvature-aware local BO method that targets a (modified) Newton step by jointly learning gradient and Hessian information with Gaussian process (GP) surrogates, and selecting evaluations via a one-step lookahead bound on the Newton-step error. We show that this bound contracts with batch size, so NeST-BO drives the step error to zero; in well-behaved neighborhoods it recovers the fast local convergence behavior of inexact/modified Newton methods, while standard safeguards support global convergence to stationary points. To improve scaling with problem dimension, we optimize the acquisition in low-dimensional embedded subspaces (random or learned), reducing the dominant cost of learning curvature from $O(d^2)$ to $O(m^2)$ with $m \ll d$ while preserving step targeting. Across high-dimensional synthetic and real-world problems, including cases with thousands of variables and unknown active subspaces, NeST-BO consistently yields faster convergence and better final values than state-of-the-art local and high-dimensional BO baselines.
\end{abstract}

\section{INTRODUCTION}
\label{sec:intro}

Bayesian optimization (BO) is a popular framework for optimizing expensive black-box functions because it often needs far fewer evaluations than alternative derivative-free methods \citep{jones1998efficient, frazier2018tutorial, garnett2023bayesian}. BO has been applied successfully in automated machine learning \citep{snoek2012practical, lindauer2022smac3}, prompt optimization for LLMs \citep{sabbatella2024bayesian}, robotics and control \citep{berkenkamp2023bayesian, paulson2023tutorial}, process optimization \citep{kudva2025multi}, materials design \citep{frazier2015bayesian, tang2024beacon}, 
and more. However, as dimensionality grows, performance often deteriorates, with recent studies attributing this decline to degeneracies such as vanishing or uninformative gradients in the Gaussian process (GP) surrogate \citep{williams2006gaussian} that make acquisition optimization brittle when length scales are poorly chosen \citep{papenmeier2025understanding}.

This work develops a curvature-aware \textit{local} BO approach and a practical way to make it scale to very high dimensions. We introduce \textbf{NeST-BO} (\textbf{Ne}wton-\textbf{S}tep-\textbf{T}argeted \textbf{BO}), which uses the GP surrogate model to \textit{jointly} learn gradient and Hessian information, and chooses new evaluations to shrink a one-step lookahead bound on the Newton-step error.
Conceptually, NeST-BO targets the step rather than the derivatives themselves -- an approach that we find can learn the Newton direction with fewer samples than, e.g., finite difference methods would require.

NeST-BO is not motivated by the idea that a vanilla Newton methods is a robust \textit{global} solver for nonconvex objectives. Rather, it is designed as a local refinement routine: once we are in a neighborhood where the objective is reasonably smooth and the curvature is informative, modified Newton schemes can provide much faster local convergence than purely first-order updates. In practice, this ``hybrid'' behavior is typically enforced by standard safeguards (e.g., damping, line search), so that the method behaves like a gradient step when curvature is unreliable, while retaining rapid local convergence when it is helpful \citep{nocedal2006numerical}.
Because second-order (Newton-type) methods can be highly effective as \textit{local} accelerators even on large, nonconvex objectives when suitably modified, e.g., \citep{martens2010hessianfree}, we conjecture targeting the step can substantially accelerate local BO.

An important obstacle is the cost of Hessian-based terms, which grows as $O(d^2)$ with input dimension $d$. To address this, we instantiate NeST-BO inside lower-dimensional subspaces, using a nested subspace expansion strategy similar to BAxUS \citep{papenmeier2022increasing}. This collapses the dominant cost to $O(m^2)$ for subspace dimension $m \ll d$ while preserving the benefits of Newton-step targeting. We find that the local Newton step is also naturally robust to the non-stationarity in the mapping from the subspace to the objective function that can be introduced by subspace embeddings, which helps explain why our approach continues to perform well for some problems even as the ambient dimension reaches thousands or more. 

Finally, we find that our acquisition includes a scale factor that balances gradient and Hessian learning whose optimal value must be estimated by, e.g., Monte Carlo sampling; our empirical results show that performance is robust to the precise choice of this factor, and we use a simple default that avoids this extra computation. 
We prove a ``vanishing power-function condition'' showing that NeST-BO drives the Newton-step error to zero as batch size increases. As a result, the method fits squarely within the classical inexact/modified Newton viewpoint: once the step errors are sufficiently small, one recovers the standard fast local convergence behavior of Newton-type methods, while globalization via, e.g., dampling/line search safeguards yields the usual notion of global convergence to stationary points \citep{nocedal2006numerical}.

In summary, our key contributions are:
\begin{itemize}
    \item A curvature-aware local BO algorithm that explicitly targets the Newton step via a tractable and theoretically-sound acquisition function.
    \item A scalable instantiation that runs NeST-BO in enlarging subspaces, reducing computation from $O(d^2)$ to $O(m^2)$, where $d$ and $m$ are the ambient and maximum subspace dimensions. 
    \item Theoretical guarantees showing that NeST-BO drives Newton-step error to zero and thus inherits the convergence behavior of modified Newton methods.
    \item Extensive empirical results on more than 12 synthetic and real-world problems (ranging from 20d to $>$7000d), where NeST-BO variants yield large performance improvements over six state-of-the-art high-dimensional BO baselines.
\end{itemize}




\section{RELATED WORK}
\label{sec:related-work}

\paragraph{Linear subspaces and embeddings.} A common strategy for high-dimensional BO is to assume the objective varies mainly in a lower-dimensional subspace and reduce model complexity accordingly. REMBO projects the search into a random linear subspace and optimizes there, with guarantees when the effective dimension is small \citep{wang2016bayesian}.
ALEBO improves robustness by using a Mahalanobis kernel and linear constraints on the acquisition to respect the original box \citep{letham2020re}. HeSBO replaces dense projections with count-sketch-style \textit{sparse} embeddings that preserve structure with negligible overhead. BAxUS introduces \textit{nested} random subspaces that expand during the run and a mechanism to carry observations across enlargements, providing improved success probabilities and practical robustness \citep{papenmeier2022increasing}. In this paper, to improve the scalability of our method, we adopt the BAxUS embedding and enlargement schedule but \textit{replace} its trust-region-based optimizer with NeST-BO.

\paragraph{Learning sparse structure.} Another interesting strategy for tackling high-dimensional problems is to adaptively learn space substructure. SAASBO places a sparsity-promoting prior on inverse GP length-scales \citep{eriksson2021high}. This can be very effective when the active set is small and axis-aligned, but the fully Bayesian inference of the kernel hyperparameters makes the inference cost scale cubically with the number of evaluations, which greatly limit its ability to scale beyond fairly small sampling budgets. 

\paragraph{Local BO.} Local BO restricts search to neighborhoods around the incumbent to mitigate the curse of dimensionality. TuRBO, a trust-region variant, maintains multiple local regions with adaptive sizes \citep{eriksson2019scalable}. Another line of work is \textit{directional} local BO, which uses a GP to define a local step rule. GIBO reduces gradient posterior uncertainty and moves along the mean gradient \citep{muller2021local}, while MPD chooses the direction maximizing the posterior probability of descent \citep{nguyen2022local}. MinUCB forgoes direct gradient inference and instead minimizes the upper confidence bound (UCB) objective as a local step. Our approach differs by explicitly \textit{targeting the Newton step}, which uses both gradient and Hessian predictions from the GP. 

\paragraph{``Vanilla BO works'' in high dimensions.} Several recent works show that standard BO can be competitive in high dimensions when properly designed. Hvarfner et al. scale a log-normal length-scale prior with dimension, yielding a strong ``D-scaled'' LogEI baseline \citep{hvarfner2024vanilla}. Xu et al. argue that poor length-scale initialization induces vanishing gradients in GPs with squared exponential (SE) kernels and show Mat\'ern kernels or robust initialization can avoid this pathology \citep{xu2024standard}. Papenmeier et al. analyze why such settings succeed, attributing gains to low effective dimensionality or benign benchmark structure rather than a general cure for high-dimensional problems \citep{papenmeier2025understanding}. 
We include such strong ``vanilla'' baselines in our experiments to reflect these practices.


\section{PRELIMINARIES}
\label{sec:preliminaries}

\subsection{Problem Setup \& Bayesian Optimization}

We consider the following zeroth-order optimization problem over a $d$-dimensional space:
\begin{align} \label{eq:true-opt}
  \bs{x}^\star \in \argmin_{\bs{x} \in \mathcal{X}} f(\bs{x}), \qquad \mathcal{X} \subseteq \mathbb{R}^d,
\end{align}
where the expensive function $f$ can only be accessed through noisy queries $y = f(\bs{x}) + \epsilon$ with i.i.d. Gaussian noise $\epsilon \sim \mathcal{N}(0,\sigma^2)$. Bayesian optimization (BO) tackles \eqref{eq:true-opt} by fitting a probabilistic surrogate $p(f|\mathcal{D})$ over available data $\mathcal{D}$ and uses it to select new evaluations by maximizing an \textit{acquisition function} $\alpha(\bs{x}|\mathcal{D})$ that trades off exploration and exploitation. After exhausting the budget, the recommended solution is either the best observed point or the minimizer of the surrogate mean.
Many acquisitions have been proposed including expected improvement (EI) \citep{jones1998efficient},
upper confidence bound (UCB) \citep{srinivas2010gaussian}, knowledge gradient (KG) \citep{frazier2008knowledge}, and entropy-based methods \citep{hennig2012entropy}. We refer readers to \citet{garnett2023bayesian} for a full tutorial.

\subsection{Gaussian Processes and their Derivatives}

We use Gaussian processes (GPs) \citep{williams2006gaussian} as surrogates, which is the most popular surrogate model class in BO. A GP prior $f \sim \mathcal{GP}(\mu,k)$ induces a joint Gaussian belief over any finite set of inputs; conditioning on the dataset $\mathcal{D}$ yields a posterior GP $f | \mathcal{D} \sim \mathcal{GP}( \mu_\mathcal{D}, k_\mathcal{D} )$ with closed-form posterior mean $\mu_\mathcal{D}$ and covariance (or kernel) function $k_\mathcal{D}$. A key property we repeatedly use in this work is that \textit{derivatives of a GP remain GPs} because differentiation is a linear operator \citep{de2021high}. Thus, the gradient $\bs{g}(\bs{x}) = \nabla f(\bs{x})$ and Hessian $\bs{H}(\bs{x}) = \nabla^2 f(\bs{x})$ have analytic posterior means and covariances obtained by differentiating the kernel in the appropriate arguments. We collect the explicit formulas for $f$, $\bs{g}$, and $\bs{H}$ posteriors in Appendix \ref{app:gp-expressions}.

\subsection{Local BO using Gradients}

Learning a globally accurate surrogate becomes data-hungry as $d$ grows because regret bounds for global BO (e.g., GP-UCB) scale exponentially with dimension unless strong structural assumptions hold \citep{srinivas2010gaussian}. Local BO addresses this by focusing search near the incumbent and updating the model with more locally collected data. Gradient-informed methods refine this idea by explicitly selecting evaluations that reduce uncertainty about the local descent direction.
A prominent example is the Gradient Information (GI) acquisition that underlies GIBO \citep{muller2021local}. Let $\bs{x}_t$ be the current iterate and $\bs{Z} \in \mathbb{R}^{b_t\times d}$ a batch of candidates. GI selects $\bs{Z}$ to maximally reduce the expected posterior covariance of the gradient at $\bs{x}_t$ given current data $\mathcal{D}$ after observing a new batch of points $(\bs{Z}, \bs{y})$:
\begin{align} \label{eq:gi}
& \alpha_{\mathrm{GI}}(\bs{Z}|\bs{x}_t,\mathcal{D}) \\\notag
& =\mathbb{E}_{\bs{y} | \mathcal{D},\bs{Z}}
\!\big\{\tr\!\big(\Sigma^{\bs g}_{\mathcal{D}}(\bs{x}_t)\big)
-\tr\!\big(\Sigma^{\bs g}_{\mathcal{D}\cup(\bs{Z},\bs{y})}(\bs{x}_t)\big)\big\}.
\end{align}
The key observation is, since posterior covariances do not depend on the observed targets $\bs{y}$, this simplifies to minimizing a (squared) ``power function'' for the gradient defined as the trace of the posterior covariance:
\begin{align} \label{eq:gi-rearranged}
\tilde{\alpha}_{\mathrm{GI}}(\bs{Z}|\bs{x}_t,\mathcal{D})
=\tr\!\big(\Sigma^{\bs g}_{\mathcal{D}\cup \bs{Z}}(\bs{x}_t)\big)
=\pi^{\bs g}_{\mathcal{D}\cup\bs{Z}}(\bs{x}_t).
\end{align}
This criterion encourages sampling along directions that most reduce gradient uncertainty, after which the algorithm steps along the GP posterior mean gradient. Theoretical analysis of GIBO showed that the convergence rate to a stationary point scales linearly in $d$, which is better than rates at which global optimization can find the global optimum \citep{wu2023behavior}. 

Nguyen et al.~\citep{nguyen2022local} revisit the GIBO template and show that the GP posterior mean gradient is not, in general, the direction that maximizes the \textit{posterior probability of descent}, which motivates MPD: an acquisition that targets an upper bound on the one-step maximum descent probability together with updates along the most-probable descent direction.
We include MPD as a strong baseline; NeST-BO builds on the same local gradient-learning viewpoint as GIBO but targets a \textit{second-order} Newton step rather than a first-order descent direction. 

\subsection{Newton's Method for Optimization}

Newton’s method (NM) is a classical algorithm for (unconstrained) local minimization of twice continuously differentiable objectives. Starting from $\bs{x}_0 \in \mathbb{R}^d$, it makes the following updates:
\begin{align}
  \bs{x}_{t+1} = \bs{x}_t - \gamma_t \bs{H}(\bs{x}_t)^{-1} \bs{g}(\bs{x}_t), \qquad t=0,1,\ldots,
\end{align}
where $\gamma_t > 0$ denotes the step size at iteration $t$. NM re-scales and rotates the gradient using local curvature. 
On well-behaved landscapes (e.g., in a neighborhood where $\bs{H}$ is positive definite), this makes progress far less sensitive to conditioning than first-order methods and yields \textit{local quadratic} convergence near a (nondegenerate) minimizer, in contrast to the linear or sublinear rates typical of gradient schemes under comparable assumptions \citep{nocedal2006numerical}.

In the BO context, we do not get to directly observe $\bs{g}$ or $\bs{H}$, but their GP posteriors (implicitly) define a distribution over the Newton step $\bs{d}(\bs{x})=\bs{H}(\bs{x})^{-1}\bs{g}(\bs{x})$. NeST-BO is built around {actively reducing the posterior uncertainty of the step $\bs{d}(\bs{x})$, rather than estimating $\bs{g}(\bs{x})$ and $\bs{H}(\bs{x})$ separately. The thought is that, when the Newton step gets close to a local solution, the iteration advantage of NM can outweigh the per-step cost of learning curvature, yielding stronger sample efficiency than gradient-only methods like GIBO.

\section{NEWTON-STEP-TARGETED BAYESIAN OPTIMIZATION}
\label{sec:nestbo}

\subsection{An Acquisition for the Newton Step}

Let $\bs{x}_t$ be the current iterate and $\bs{Z}\!\in\!\mathbb{R}^{b_t \times d}$ a batch of candidates at which we might evaluate $f$. Our goal is to learn the Newton step $\bs d(\bs x)=\bs H(\bs x)^{-1}\bs g(\bs x)$ at $\bs{x}_t$ under the GP posterior. Because $\bs{d}(\bs{x})$ is non-Gaussian, its posterior covariance is not available in closed form; instead we derive a Reproducing Kernel Hilbert Space (RKHS) error bound (inspired by the results derived in \citep{wu2023behavior}). Let the gradient and (vectorized) Hessian (squared) \textit{power functions} at $\bs{x}$ be denoted by $\pi^{\bs g}_{\mathcal D}(\bs x)=\tr\left( \Sigma^{\bs g}_{\mathcal D}(\bs x)\right)$ and
$\pi^{\bs H}_{\mathcal D}(\bs x)=\tr\big(\Sigma^{\mathrm{vec}(\bs H)}_{\mathcal D}(\bs x)\big)$. Also, let $\mathcal{H}$ be the RKHS on $\mathcal{X}$ equipped with a reproducing kernel $k(\cdot,\cdot)$ and RKHS norm $\| \cdot \|_\mathcal{H}$. We can establish the following bound (see Appendix \ref{app:proof-newton} for the proof):

\begin{theorem}[Newton-step error bound]\label{thm:newton-bound}
Let $\varepsilon_\mathcal{D}(\bs{x}) = \|\bs d(\bs{x})-\widehat{\bs d}_{\mathcal{D}}(\bs{x}) \|$ denote the Newton-step error at $\bs{x}$ given data $\mathcal{D}$ with $\widehat{\bs d}_{\mathcal{D}}(\bs{x}) = \widehat{\bs{H}}_\mathcal{D}(\bs{x})^{-1} \widehat{\bs{g}}_\mathcal{D}( \bs{x} )$ where $\| \cdot \|$ is the standard Euclidean norm. 
Assume the kernel $k$ is stationary and four-times differentiable and that $\| \bs{H}(\bs{x})^{-1} \|$ exists and is finite.
For any $f\in\mathcal H$ (from the RKHS with reproducing kernel $k$) with $\|f\|_{\mathcal H}\le B$, $\bs x\in\mathcal X$, and $\mathcal D$, the following bound holds:
\begin{align}\label{eq:newton-step-bound}
\varepsilon_\mathcal{D}(\bs{x}) \le C_{\bs{x}} \sqrt{ \pi^{\bs g}_{\mathcal D}(\bs x)+
s_{\mathcal D}(\bs x)\,\pi^{\bs H}_{\mathcal D}(\bs x)},
\end{align}
where $s_{\mathcal D}(\bs x)=\|\widehat{\bs H}_{\mathcal D}(\bs x)^{-1}\|^2\,\|\widehat{\bs g}_{\mathcal D}(\bs x)\|^2$ is a scale factor that trades off gradient and Hessian information and $C_{\bs{x}} = \sqrt{2}B\,\|\bs H(\bs x)^{-1}\|$ is independent of $\mathcal{D}$.
\end{theorem}

Motivated by the Newton-step error bound in Theorem~\ref{thm:newton-bound}, our design goal at iteration $t$ is simple: choose a batch $\bs Z$ that will make the \textit{post-batch} bound on the error as small as possible (in expectation) over the unobserved outcomes at $\bs Z$. Concretely, let $B_{\mathcal D}(\bs x_t) =
\pi^{\bs g}_{\mathcal D}(\bs x_t) + s_{\mathcal D}(\bs x_t)\,\pi^{\bs H}_{\mathcal D}(\bs x_t)$, so that \eqref{eq:newton-step-bound} is $\varepsilon_{\mathcal D}(\bs x_t)\le C_{\bs x_t}\sqrt{B_{\mathcal D}(\bs x_t)}$.
Since $C_{\bs x_t}$ is independent of the data, minimizing an upper bound on the expected step error is naturally achieved by:
\begin{align}\label{eq:nest-objective-start}
\bs Z_t \in \argmin_{\bs Z}\;
\mathbb{E}_{\bs y \mid \mathcal D,\bs Z}
\lbrace
B_{\mathcal D\cup(\bs Z,\bs y)}(\bs x_t)
\rbrace.
\end{align}
This is the direct analogue of GI, which chooses points to minimize a lookahead bound on the gradient error, as shown in \citep{wu2023behavior}. The objective in \eqref{eq:nest-objective-start} is doing the same thing but on a lookahead bound for the error in the Newton step. 

The remaining step is to simplify \eqref{eq:nest-objective-start}. Under a GP with Gaussian observation noise, the posterior \textit{covariances} of derivatives at $\bs x_t$ after conditioning on $\mathcal D\cup(\bs Z,\bs y)$ do not depend on the realized values $\bs y$; they depend only on the locations $\bs Z$,
so the only term in \eqref{eq:nest-objective-start} that requires a lookahead expectation is
$s_{\mathcal D\cup(\bs Z,\bs y)}(\bs x_t)$, which depends on posterior \textit{means} through $\widehat{\bs g}$ and $\widehat{\bs H}$.
Using this fact, \eqref{eq:nest-objective-start} reduces to our following proposed Newton Step Targeting (NeST) acquisition function:
\begin{align}\label{eq:nest-rearranged}
& \tilde{\alpha}_{\mathrm{NeST}}(\bs Z | \bs x_t,\mathcal D) \\\notag
& = \pi^{\bs g}_{\mathcal D\cup\bs Z}(\bs x_t)
+ \mathbb{E}_{\bs y \mid \mathcal D,\bs Z}\left\lbrace s_{\mathcal D\cup(\bs Z,\bs y)}(\bs x_t) \right\rbrace
  \pi^{\bs H}_{\mathcal D\cup\bs Z}(\bs x_t).
\end{align}
NeST is composed of three main terms: (i) $\pi^{\bs g}_{\mathcal D\cup\bs Z}(\bs x_t)$ is the trace of the covariance of the one-step lookahead gradient at $\bs{x}_t$, (ii) $\pi^{\bs H}_{\mathcal D\cup\bs Z}(\bs x_t)$ is the trace of the covariance of the one-step lookahead vectorized Hessian, and (iii) $\mathbb{E}_{\bs y \mid \mathcal D,\bs Z}\left\lbrace s_{\mathcal D\cup(\bs Z,\bs y)}(\bs x_t) \right\rbrace$ is a scale factor that weights these two terms. We can interpret (i) and (ii) as measures of the information content about quantities of interest (mainly gradient and Hessian at our current point) and (iii) as a term that adaptively balances gradient and Hessian information. Note that, since $\| \bs{H}^{-1}(\bs{x}) \bs{g}(\bs{x}) \| \leq \| \bs{H}^{-1}(\bs{x}) \| \| \bs{g}(\bs{x}) \|$, we can also think of this term as a conservative proxy for the squared step length and thus upweights Hessian information when NM may take large steps. 

\paragraph{A practical family.}
The expectation over the scale factor in \eqref{eq:nest-rearranged} does not have a simple closed-form expression and thus would need to be estimated using Monte Carlo (MC) sampling in practice (see Appendix \ref{app:mc-est-scale} for details). This is mainly due to the fact that $s_{\mathcal D\cup(\bs Z,\bs y)}$ is a nonlinear function of $\bs{y}$ and thus non-Gaussian in general. Through some empirical testing shown in Appendix \ref{app:scale-factor}, we observed relatively low sensitivity in performance to the precise scale factor, which motivates the following computationally cheaper acquisition:
\begin{align}\label{eq:nest-approx}
\widehat{\alpha}_{\mathrm{NeST}}(\bs Z | \bs x_t,\mathcal D, \widehat{s}_t)
= \pi^{\bs g}_{\mathcal D\cup\bs Z}(\bs x_t) + \widehat{s}_t \,\pi^{\bs H}_{\mathcal D\cup\bs Z}(\bs x_t),
\end{align}
with a pre-determined $\widehat{s}_t > 0$ that is independent of $\bs{Z}$.
This recovers GI when $\widehat{s}_t = 0$ and focuses more on curvature as $\widehat{s}_t$ increases. In Section \ref{subsec:theoretical-analysis}, we show the NeST samples can drive the bound in \eqref{eq:newton-step-bound} to zero as the batch size increases for any choice of scale factor. 

\subsection{The NeST-BO Algorithm}

Algorithm~\ref{alg:NeSTBO} summarizes the simple version of the loop. At iterate $\bs{x}_t$, choose a batch $\bs{X}_t$ by minimizing \eqref{eq:nest-approx} at $\bs{x}_t$; update the GP with the new observations; then move along the predicted Newton step with some step size. Note there are a number of practical implementation details, which we discuss more in the next section. 

Since we are attempting to learn both the gradient and Hessian at $\bs{x}_t$, one can in fact easily adapt this algorithm to use a step direction from any form of, e.g., damped NM \citep{hanzely2022damped} or NM with line search \citep{shea2025greedy}. We do not attempt to systematically compare different options in this work; instead we go with a standard ``hybrid'' implementation that uses Newton steps when they well-behaved and otherwise revert to gradient steps.

We provide an illustration and visual comparison of NeST-BO (top) to GIBO (bottom) in Figure \ref{fig:Illustrative}, which demonstrates the advantages of simultaneously learning gradient and Hessian (curvature) information. 

\begin{algorithm}[tb!]
\caption{NeST-BO}
\label{alg:NeSTBO}
\textbf{Inputs:} initial iterate $\bs{x}_0\!\in\!\mathbb{R}^d$; initial data $\mathcal{D}_0$; batch sizes $\{b_t\}$; step sizes $\{\gamma_t\}$; scale factors $\{\widehat{s}_t\}$; GP hyperparameters; and total number of iterations $T$.
\begin{algorithmic}[1]
\State \textbf{Fit} GP surrogate on $\mathcal{D}_0$.
\For{$t=0,\dots,T-1$}
  \State $\bs{X}_t \in \textstyle\argmin_{\bs{Z} \in \mathbb{R}^{b_t\times d}} \widehat{\alpha}_{\mathrm{NeST}}(\bs{Z} | \bs{x}_t,\mathcal{D}_t,\widehat{s}_t)$.
  \State Evaluate $f$ at $\bs{X}_t$ to obtain $\bs{y}_t$.
  \State Augment dataset $\mathcal{D}_{t+1}\leftarrow\mathcal{D}_t\cup(\bs{X}_t,\bs{y}_t)$.
  \State Update GP posterior with $\mathcal{D}_{t+1}$.
  \State Compute $\widehat{\bs{g}} \leftarrow \widehat{\bs g}_{\mathcal{D}_{t+1}}(\bs{x}_t)$ and $\widehat{\bs{H}} \leftarrow \widehat{\bs H}_{\mathcal{D}_{t+1}}(\bs{x}_t)$. 
  \State Solve $\widehat{\bs{H}}\bs{v}=\widehat{\bs{g}}$ for $\widehat{\bs{d}}_{\mathcal{D}_{t+1}}(\bs{x}_t) \leftarrow \bs{v}$.
  \State Update iterate: $\bs{x}_{t+1}=\bs{x}_t-\gamma_t\,\widehat{\bs{d}}_{\mathcal{D}_{t+1}} (\bs{x}_t)$. \label{alg:newton-direction}
\EndFor
\end{algorithmic}
\end{algorithm}

\begin{figure*}[tb!]
  \centering
  \includegraphics[width=0.9\textwidth]{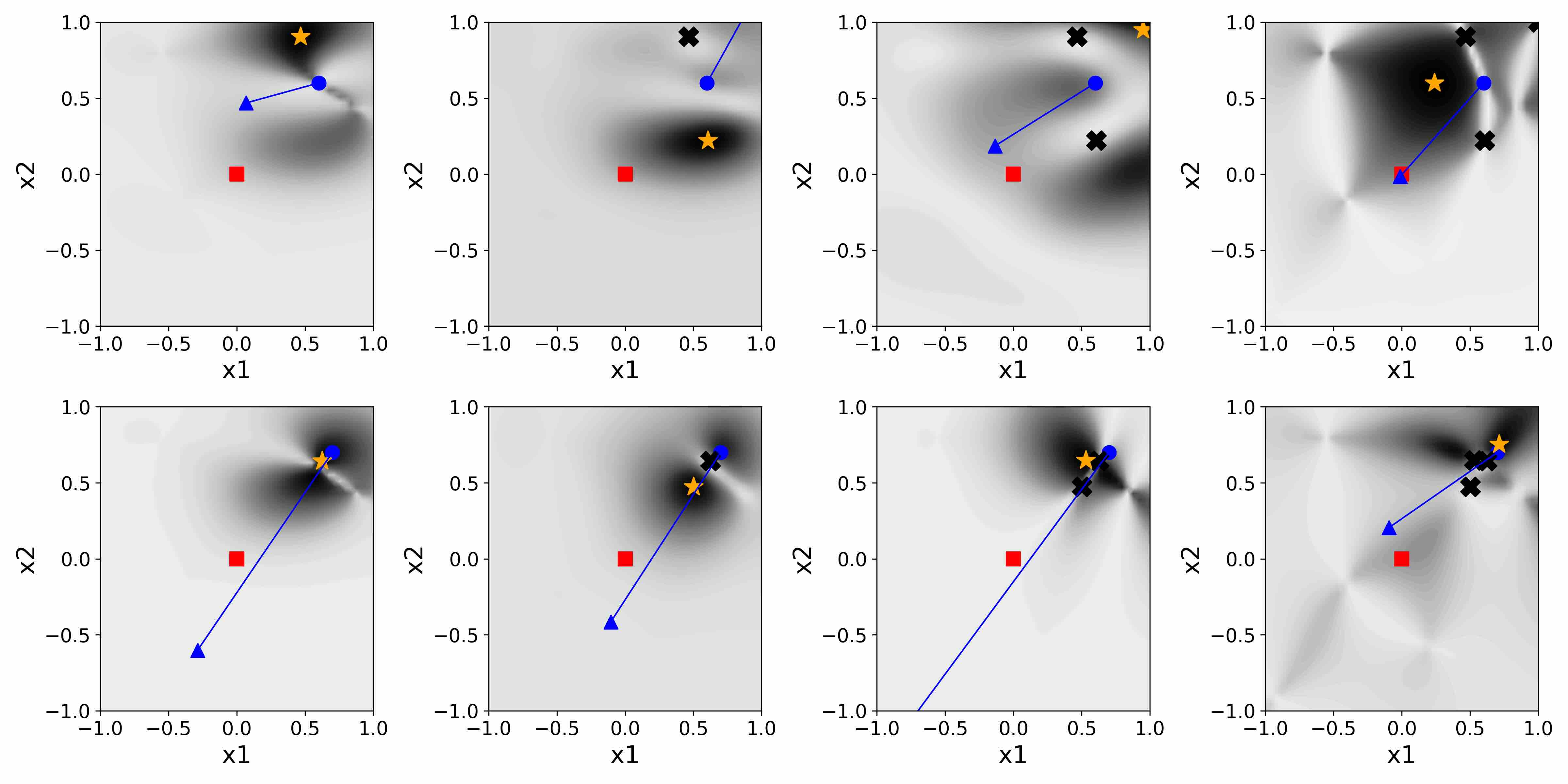}  
  \caption{\textbf{Top:} NeST-BO’s acquisition $\tilde{\alpha}_{\mathrm{NeST}}$; \textbf{Bottom:} GIBO's acquisition $\tilde{\alpha}_{\mathrm{GI}}$ on a 2D test function at iterate $\bs{x}_t$ (blue circle). Darker background indicates larger acquisition value. Red square: location of the true Newton step. Blue triangle: GP-predicted Newton step. Orange star: acquisition minimizer. Black crosses: Past evaluation points. NeST-BO places samples away from $\bs{x}_t$ along directions informative for curvature, rapidly shrinking the Newton-step error bound; GI tends to oversample near $\bs{x}_t$, slowing curvature identification.}
  \label{fig:Illustrative}
\end{figure*}

\subsection{Implementation Details}

Algorithm~\ref{alg:NeSTBO} describes an idealized NeST-BO loop used for exposition. In practice, just as vanilla NM can be brittle without safeguards, we found it helpful to use a slightly ``hardened'' variant for all experiments. The key point is that NeST-BO’s value is in \textit{where} it samples at each iteration (i.e., producing high-quality local gradient/Hessian information), so we can borrow standard Newton-style safeguards without changing the core idea. For completeness, the full pseudocode for our practical implementation is given in Appendix~\ref{app:practical-nestbo}. We summarize the main choices here.

\paragraph{GP hyperparameters.} 
In most applications, kernel hyperparameters are not known \textit{a priori}. We thus need to fit the GP by maximizing the marginal log likelihood, which is standard in the BO literature. Our method does not rely on this specific estimator; any hyperparameter learning procedure can be used.

\paragraph{Moving direction.}
NeST-BO is meant to behave like a local Newton-type method once it enters a well-behaved neighborhood. 
When $\widehat{\bs H}_{\mathcal D}(\bs x_t)$ is positive definite (and not too ill-conditioned), this is a descent direction for the GP mean and typically behaves as expected. If $\widehat{\bs H}_{\mathcal D}(\bs x_t)$ is indefinite or numerically unstable to invert, we interpret this as a signal that we are not yet in a regime where Newton steps are reliable, and we revert to a length-scale-normalized gradient step \citep[Appendix~A.4]{muller2021local}. This is a simple and robust fallback; many other Newton safeguards (e.g., damping/regularization or trust-region updates) could also be used, but we did not attempt to exhaustively test (or internally tune) these variants, which would be valuable future work. 

\paragraph{Batch versus sequential selection.}
Although \eqref{eq:nest-approx} supports joint batch optimization, we select points \textit{greedily}. After choosing each point, we update the posterior covariance deterministically (note that the power functions do not depend on the realized outcomes), and then re-optimize to choose the next point. This replaces a single $b_t d$-dimensional search with $b_t$ separate $d$-dimensional searches and naturally discourages near-duplicate selections, since conditioning on a chosen location reduces posterior variance around it. In our preliminary tests, greedy selection closely matched joint optimization. 

\paragraph{Step size.}
We use a standard Armijo backtracking line search on the GP mean $\mu_{\mathcal D}$ along the chosen direction. This gives a lightweight safeguard against overly aggressive steps without requiring additional function evaluations, and we cap the number of backtracking iterations using standard defaults \citep[Chapter~1]{bertsekas2016nonlinear}. There has been recent work on greedy NM \citep{shea2025greedy}, which suggests it can be advantageous to allow for step sizes larger than 1, but we do not consider that here.

\paragraph{Scale factor.}
The scale $\widehat{s}_t$ in \eqref{eq:nest-approx} controls how strongly (approximate) NeST emphasizes curvature (Hessian) uncertainty relative to gradient uncertainty. For the main experiments, we use a simple deterministic proxy $\widehat{s}_t = 1$, which performed reliably across problems and dimensions. Appendix~\ref{app:comparison-scale} provides a broader comparison over $\widehat{s}_t$ choices. We also tested a ``plug-in'' variant $\widehat{s}_t = s_{\mathcal D}(\bs x_t)$ that neglects the lookahead dependence on $(\bs Z,\bs y)$ (Appendix~\ref{app:plugin-vs-s1}); performance was very similar between these options, so we default to $\widehat{s}_t=1$ for simplicity.

\subsection{Theoretical Analysis}
\label{subsec:theoretical-analysis}

Let $\varepsilon_t = \varepsilon_{\mathcal{D}_{t+1}}(\bs{x}_t)$ denote the Newton-step error after iteration $t$.
By Theorem~\ref{thm:newton-bound}, $\varepsilon_t$ is controlled by posterior uncertainty in the \emph{local derivatives} at $\bs{x}_t$; specifically, the power functions appearing in \eqref{eq:newton-step-bound}.
This motivates a simple designability condition:

\noindent\textbf{Vanishing power-function condition (VPC).}
The NeST design objective
$\pi^{\bs{g}}_{\mathcal{D}_{t+1}}(\bs{x}_t) + \widehat{s}_t\,\pi^{\bs{H}}_{\mathcal{D}_{t+1}}(\bs{x}_t)$
can be made arbitrarily small as $b_t$ increases.

VPC is not an abstract assumption on $\varepsilon_t$; it ties the step error to \textit{concrete, design-controllable} posterior covariances.
A key subtlety is that NeST-BO cannot directly minimize the right-hand side of \eqref{eq:newton-step-bound} because the scale term $s_{\mathcal{D}_{t+1}}(\bs{x}_t)$ is only revealed \textit{after} the batch is observed.
Nevertheless, the bound applies to the realized step error under \textit{any} sampling rule, and VPC can be verified constructively for NeST.
We state the main idea informally below; the complete statement and proof are in Appendix~\ref{app:proof-vpc}.

\begin{theorem}[VPC under NeST sampling; informal]\label{thm:vpc}
Fix an iteration $t$ and the current iterate $\bs{x}_t$. For a batch $\bs{Z}=(\bs{z}_1,\dots,\bs{z}_{b_t})\in\mathcal{X}^{b_t}$, let $\Phi_t(\bs{Z}) = \widehat{\alpha}_{\mathrm{NeST}}(\bs Z | \bs x_t,\mathcal D, \widehat{s}_t)$ be short for the approximate NeST acquisition function with $\widehat{s}_t > 0$. 
Under mild domain and kernel regularity conditions, the following hold. 

Let $b^\star=d^2+d+1$. Then, in the noiseless setting, the optimal NeST design value can be driven to zero:
\[
\inf_{\bs{Z}\in\mathcal{X}^{b_t}} \Phi_t(\bs{Z}) = 0
~ \text{(as an infimum) for all } b_t\ge b^\star.
\]
Equivalently, for every $\epsilon>0$ and every $b_t\ge b^\star$, there exists a batch $\bs{Z}_\epsilon$ such that
$\Phi_t(\bs{Z}_\epsilon)\le \epsilon$.
Moreover, one explicit candidate design achieving this is a symmetric $b^\star$-point centered-difference stencil at radius $h$, for which $\Phi_t(\bs{Z}_h) \lesssim h^4$ as $h\to 0$.

With i.i.d.\ Gaussian noise of variance $\sigma^2$, using $m$ replicates per location is equivalent to observing averaged responses with variance $\sigma^2/m$.
Thus, for every $\epsilon>0$ there exist $h$ and $m$ such that the corresponding replicated design satisfies
$\Phi_t(\bs{Z}^{(m)}_h)\le \epsilon$.
\end{theorem}

An important consequence is that, once NeST-BO is operating in a neighborhood where NM is well-conditioned, VPC implies $\varepsilon_t \to 0$ as the batch size increases, and NeST-BO asymptotically inherits the \textit{local quadratic} convergence rate of NM (full statement and proof in Appendix~\ref{app:local-quadratic-conv}).
Finally, the VPC statement itself does not depend on the particular choice of scale factor beyond requiring $\widehat{s}_t>0$.
Thus, the simplified acquisition in \eqref{eq:nest-approx} retains the same asymptotic guarantee, with the caveat that the argument is inherently for large enough batches. 

\subsection{Scaling NeST-BO via Subspaces}

The main computational bottleneck for NeST-BO in high ambient dimension $d$ is the Hessian term $\pi^{\bs{H}}_{\mathcal{D} \cup \bs{Z}}(\bs{x}_t)$, which scales quadratically in $d$. We propose to address this by instantiating NeST-BO \textit{inside} embedded subspaces $\bs{v} \in \mathbb{R}^m$ of size $m \ll d$. We adopt the nested, sparse random-embedding scheme of BAxUS -- bins of input coordinates are hashed into $m$ target coordinates with random signs, and the target dimension is increased over time by \textit{splitting} bins while retaining observations across embeddings \citep{papenmeier2022increasing}. This approach preserves past data under splits and increases the probability that the subspace contains an optimizer as $m$ grows. 

We run NeST-BO in the subspace $\bs{v}$, map the candidates back to the ambient dimension $\bs{x} = \bs{S}^\top \bs{v}$ (where $\bs{S}^\top \in \mathbb{R}^{d \times m}$ is the projection matrix) for evaluation, and update the GP in the embedded subspace coordinates. This reduces the per-candidate curvature cost from $O(d^2)$ to $O(m^2)$ while preserving the step-targeting principle.
Compared to the original BAxUS algorithm, which couples its embedding with a variant of TuRBO, our version replaces TuRBO with NeST-BO; in problems where curvature matters, we find that this can can substantially accelerate optimization progress (see results in Section \ref{sec:experiments}). 

Lastly, note that our theoretical results related to VPC are proved only for NeST-BO in the \textit{original} $d$-dimensional space, and they do not directly extend to the subspace variant.
Intuitively, the VPC proof relies on constructing $d$-dimensional designs that can drive the full gradient and Hessian power functions at $\bs{x}_t$ to zero, whereas a fixed low-dimensional embedding only controls derivative information within the embedded coordinates and may not capture curvature directions relevant in $\mathbb{R}^d$. 
The BAxUS-style embedding should thus be viewed as a practical scalability heuristic that leverages the NeST sampling principle. 

\section{RESULTS}
\label{sec:experiments}

We now benchmark NeST-BO and its subspace variant, labeled NeST-BO-sub, against strong local and global BO baselines: TuRBO~\citep{eriksson2019scalable}, GIBO~\citep{muller2021local}, MPD~\citep{nguyen2022local}, MinUCB~\citep{fan2024minimizing}, BAxUS~\citep{papenmeier2022increasing}, and a ``vanilla'' GP-BO configured with dimensionally calibrated priors and LogEI (we label this \emph{D-scaled LogEI})~\citep{hvarfner2024vanilla}. We also include Sobol sampling as a non-model baseline.

Unless a global minimizer is known, we report the \textit{minimum observed value}; otherwise we show \textit{simple regret} (in log scale). Curves display the median across $10$ independent replicates with $\pm$ one standard error as the shaded band. Implementation details (e.g., models, hyperparameter updates, acquisition optimization, and the precise evaluation budgets for each task) are provided in Appendix~\ref{app:experiment-details}. In short, we used a common SE kernel across methods and standard GP training and acquisition optimizers from \texttt{BoTorch} \citep{balandat2020botorch}; hyperparameters and method-specific settings follow prior work and are held consistent across tasks to keep comparisons fair. Due to space limitations, additional ablations and diagnostics appear in Appendix~\ref{app:add-experiments}.
The code is available on GitHub: \url{https://github.com/PaulsonLab/NeST-BO}. 

\subsection{Synthetic Test Functions}

We consider two regimes relevant to our method. \textit{Moderate dimension} ($d=20$): Sphere, Griewank, and Ackley. \textit{High dimension with sparse structure} ($d=1000$ with $d_\text{eff}=30$ relevant variables): Griewank, Ackley, and Rosenbrock; the remaining coordinates are dummies and the algorithms are not told which ones are active. These are commonly chosen test problems in the BO literature; formal definitions are given in Appendix~\ref{app:synthetic-details}. Figure~\ref{fig:main results} (first six panels in the top two rows) summarizes the results.

\begin{figure*}[tb!]
  \centering
  \includegraphics[width=0.95\textwidth]{figures/Main_results2.jpg}
  \caption{Summary of performance versus evaluations for all synthetic and real-world problems and all methods. Each panel shows either simple regret (log scale; when the global minimizer is known) or the minimum observed value (otherwise). Curves are medians across $10$ runs; shading is $\pm$ one standard error. Top two rows: synthetic problems (20d and 1000d with 30 active variables). Bottom two rows: real-world tasks (control, planning, and high-dimensional model selection). See Appendix~\ref{app:experiment-details} for the full protocol and Appendix~\ref{app:add-experiments} for extended studies.}
  \label{fig:main results}
\end{figure*}

On the $d=20$ problems, NeST-BO consistently matches or outperforms all non-subspace baselines. A common pattern is an initial period with slower progress when far from a minimizer (before curvature is accurately estimated) followed by a steep drop once the iterate enters a well-behaved neighborhood. As several steps accumulate, the estimated Newton step better aligns with the true local geometry, further accelerating progress -- consistent with our Newton-step error bound, which tightens as the gradient and Hessian power functions shrink. Subspace methods start from stronger initial values by design (they restrict the initial design and acquisitions), but NeST-BO-sub ultimately achieves the lowest regret and, notably, improves over BAxUS across all three problems. These results indicate that \textit{how} one moves inside a subspace matters, i.e., curvature-aware Newton updates can be more effective than trust-region moves, even when both operate in the same embedding space.

On the $d=1000$ sparse suite, subspaces are essential. Methods that attempt to learn gradients and/or Hessians in the \textit{ambient} space require $O(d)$ queries per iteration and cannot meaningfully progress under our fixed budgets (e.g., 200 evaluations), so we do not include them here. NeST-BO-sub clearly dominates BAxUS, D-scaled LogEI, and TuRBO, achieving substantially lower regret on both Ackley and Griewank. TuRBO’s trust-region strategy is disadvantaged in very high dimensions, where large local diameters push pairwise distances into regimes that degrade GP fit and acquisition gradients; in contrast, NeST-BO-sub converts a handful of targeted samples near the iterate into accurate Newton steps inside the selected subspace (that can expand as iterations proceed), which drives fast local improvement.

\subsection{Mid- to High-Dim. Real-World Tasks}

We evaluate six real-world benchmarks spanning reinforcement learning (RL) control, robotic planning, and large-scale hyperparameter tuning. \textbf{Control:} Lunar Lander ($d=12$) and Swimmer ($d=16$) from OpenAI Gymnasium; the objective is the negative episodic return (reward sign flipped)~\citep{towers2024gymnasium}. \textbf{Planning:} Robot Pushing ($d=14$) and Rover Trajectory ($d=60$) from \citet{wang2018batched,eriksson2019scalable}, both optimized as negative reward. \textbf{Very high-dimension:} Ant ($d=888$) -- a MuJoCo quadruped with 8-dimensional action space and 111-dimensional observations, yielding a linear state-feedback policy with 888 parameters -- and Leukemia ($d=7129$), a weighted Lasso problem with 7129 hyperparameters from LassoBench~\citep{vsehic2022lassobench}. Task definitions, bounds, and initialization protocols are summarized in Appendix~\ref{app:real-world-details}. Figure~\ref{fig:main results} (last six panels in bottom two rows) reports median performance with $\pm$ one standard error bands.

In the mid-dimensional group ($d \leq 60$), NeST-BO is consistently state of the art or competitive, achieving the lowest average value in the fewest iterations for Lunar Lander, Robot Pushing, and Swimmer. On Rover Trajectory, NeST-BO-sub edges out NeST-BO, aligning with the intuition that embeddings can capture useful structure as dimensionality and coupling grow. On Lunar Lander and Swimmer, however, subspaces can slightly hurt performance, suggesting most coordinates contribute and the ambient space is preferable.

In the very high-dimensional group, NeST-BO-sub is again the best-performing method. On Ant ($d=888$), D-scaled LogEI and TuRBO show little-to-no improvement over their starting values, while BAxUS improves but plateaus well above NeST-BO-sub. The Ant landscape is both non-stationary and ill-conditioned; length-scale calibration alone (as in D-scaled LogEI) can over-smooth such objectives, and trust-region steps struggle to adapt their geometry. On Leukemia ($d=7129$), NeST-BO-sub continues to improve throughout the budget and achieves the best final objective, whereas other methods stagnate early.

\subsection{Other Methods in Same Subspace}

A natural question raised by our subspace results in Figure \ref{fig:main results} is whether NeST-BO’s gains are primarily due to the embedding, or whether targeting the local Newton step continues to matter once we restrict the search to lower-dimensional subspaces. To isolate these effects, we compare NeST-BO-sub to two strong baselines that use \textit{the same} subspace machinery: (i) GIBO-sub, i.e., GIBO run in the subspace using its length-scale-normalized gradient step \citep{muller2021local} and (ii) D-scaled LogEI-sub, i.e., standard LogEI run in the subspace using the GP prior from \citep{hvarfner2024vanilla}. For context, we also include BAxUS. Figure \ref{fig:subspace} shows the embedding is \textit{not} the whole story. NeST-BO-sub drives regret down by \textit{several orders of magnitude} relative to GIBO-sub and D-scaled LogEI-sub (the latter two plateau near $\sim 10^{0}$ on the log scale), while NeST-BO-sub continues improving throughout the budget. On Ackley, the same pattern holds: NeST-BO-sub reaches substantially lower regret and converges faster, while the other subspace methods level off earlier. 

\begin{figure}[tb!]
  \centering
  \includegraphics[width=0.9\linewidth]{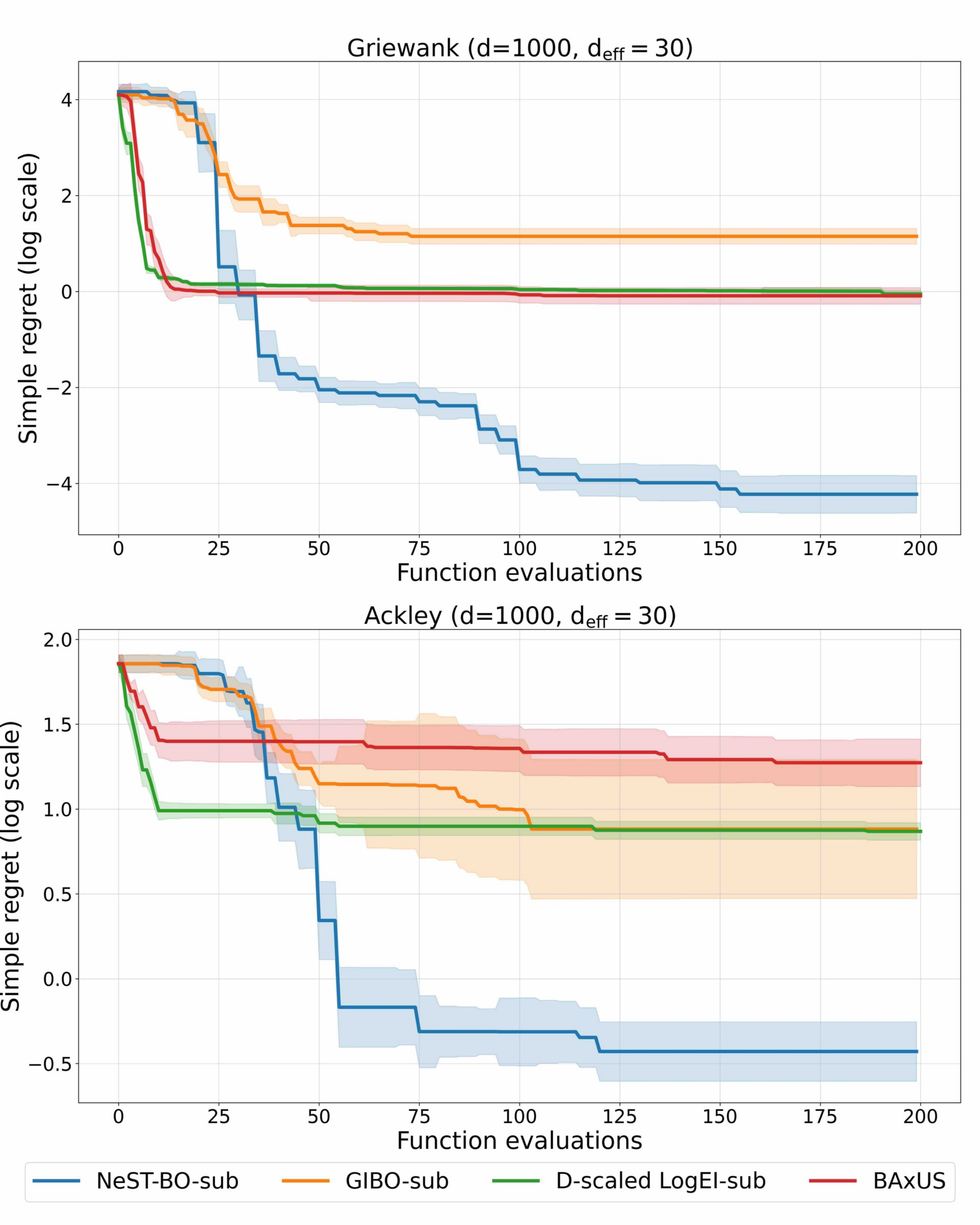}
  \caption{Optimization of $1000$-dimensional Griewank and Ackley with $30$ active variables. All -sub variants operate using the same BAxUS-style embedding approach. Median simple regret (log scale) with $\pm$ one standard error shading across 10 runs.}
  \label{fig:subspace}
\end{figure}

\section{CONCLUSIONS}
\label{sec:conclusions}

This work presents NeST-BO, a curvature-aware local Bayesian optimization (BO) method that selects samples to shrink a one-step lookahead bound on Newton-step error and then moves with a damped Newton update. Theoretical analysis establishes a vanishing power-function condition that holds under NeST-BO sampling, implying the algorithm inherits (inexact) Newton guarantees while our experiments show consistent gains over state-of-the-art local and high-dimensional BO baselines on synthetic and real-world problems, including tasks with several thousands of variables (when combined with a subspace variant to enhance scalability). Looking ahead, two promising directions for future research include investigating improved subspace embedding strategies and improving numerical efficiency of the acquisition optimization by better exploiting kernel structure and sparsity for derivative-aware GPs.



\bibliographystyle{abbrvnat}
\bibliography{reference}

@article{muller2021local,
  title={Local policy search with {B}ayesian optimization},
  author={M{\"u}ller, Sarah and von Rohr, Alexander and Trimpe, Sebastian},
  journal={Advances in Neural Information Processing Systems},
  volume={34},
  pages={20708--20720},
  year={2021}
}

@article{wu2023behavior,
  title={The behavior and convergence of local {B}ayesian optimization},
  author={Wu, Kaiwen and Kim, Kyurae and Garnett, Roman and Gardner, Jacob},
  journal={Advances in Neural Information Processing Systems},
  volume={36},
  pages={73497--73523},
  year={2023}
}

@book{wendland2004scattered,
  title={Scattered data approximation},
  author={Wendland, Holger},
  volume={17},
  year={2004},
  publisher={Cambridge university press}
}

@article{eriksson2019scalable,
  title={Scalable global optimization via local {B}ayesian optimization},
  author={Eriksson, David and Pearce, Michael and Gardner, Jacob and Turner, Ryan D and Poloczek, Matthias},
  journal={Advances in Neural Information Processing Systems},
  volume={32},
  year={2019}
}

@InProceedings{hvarfner2024vanilla,
  title = 	 {Vanilla {B}ayesian Optimization Performs Great in High Dimensions},
  author =       {Hvarfner, Carl and Hellsten, Erik Orm and Nardi, Luigi},
  booktitle = 	 {Proceedings of the 41st International Conference on Machine Learning},
  pages = 	 {20793--20817},
  year = 	 {2024},
  editor = 	 {Salakhutdinov, Ruslan and Kolter, Zico and Heller, Katherine and Weller, Adrian and Oliver, Nuria and Scarlett, Jonathan and Berkenkamp, Felix},
  volume = 	 {235},
  series = 	 {Proceedings of Machine Learning Research},
  month = 	 {21--27 Jul},
  publisher =    {PMLR},
}

@article{mania2018simple,
  title={Simple random search provides a competitive approach to reinforcement learning},
  author={Mania, Horia and Guy, Aurelia and Recht, Benjamin},
  journal={arXiv preprint arXiv:1803.07055},
  year={2018}
}

@inproceedings{wang2018batched,
  title={Batched large-scale {B}ayesian optimization in high-dimensional spaces},
  author={Wang, Zi and Gehring, Clement and Kohli, Pushmeet and Jegelka, Stefanie},
  booktitle={International Conference on Artificial Intelligence and Statistics},
  pages={745--754},
  year={2018},
  organization={PMLR}
}

@article{nguyen2022local,
  title={Local {B}ayesian optimization via maximizing probability of descent},
  author={Nguyen, Quan and Wu, Kaiwen and Gardner, Jacob and Garnett, Roman},
  journal={Advances in neural information processing systems},
  volume={35},
  pages={13190--13202},
  year={2022}
}

@article{fan2024minimizing,
  title={Minimizing UCB: a Better Local Search Strategy in Local {B}ayesian Optimization},
  author={Fan, Zheyi and Wang, Wenyu and Ng, Szu H and Hu, Qingpei},
  journal={Advances in Neural Information Processing Systems},
  volume={37},
  pages={130602--130634},
  year={2024}
}

@inproceedings{xu2024standard,
  title={Standard {G}aussian process is all you need for high-dimensional {B}ayesian optimization},
  author={Xu, Zhitong and Wang, Haitao and Phillips, Jeff M and Zhe, Shandian},
  booktitle={The Thirteenth International Conference on Learning Representations},
  year={2024}
}

@article{papenmeier2025understanding,
  title={Understanding High-Dimensional {B}ayesian Optimization},
  author={Papenmeier, Leonard and Poloczek, Matthias and Nardi, Luigi},
  journal={arXiv preprint arXiv:2502.09198},
  year={2025}
}

@inproceedings{eriksson2021high,
  title={High-dimensional Bayesian optimization with sparse axis-aligned subspaces},
  author={Eriksson, David and Jankowiak, Martin},
  booktitle={Uncertainty in Artificial Intelligence},
  pages={493--503},
  year={2021},
  organization={PMLR}
}

@article{wang2016bayesian,
  title={Bayesian optimization in a billion dimensions via random embeddings},
  author={Wang, Ziyu and Hutter, Frank and Zoghi, Masrour and Matheson, David and De Feitas, Nando},
  journal={Journal of Artificial Intelligence Research},
  volume={55},
  pages={361--387},
  year={2016}
}

@article{nesterov2006cubic,
  title={Cubic regularization of {N}ewton method and its global performance},
  author={Nesterov, Yurii and Polyak, Boris T},
  journal={Mathematical Programming},
  volume={108},
  number={1},
  pages={177--205},
  year={2006},
  publisher={Springer}
}

@article{ament2023unexpected,
  title={Unexpected improvements to expected improvement for bayesian optimization},
  author={Ament, Sebastian and Daulton, Samuel and Eriksson, David and Balandat, Maximilian and Bakshy, Eytan},
  journal={Advances in Neural Information Processing Systems},
  volume={36},
  pages={20577--20612},
  year={2023}
}

@article{kudva2025multi,
  title={Multi-Objective Bayesian Optimization for Networked Black-Box Systems: A Path to Greener Profits and Smarter Designs},
  author={Kudva, Akshay and Tang, Wei-Ting and Paulson, Joel A},
  journal={arXiv preprint arXiv:2502.14121},
  year={2025}
}

@article{tang2024beacon,
  title={{BEACON: A Bayesian optimization strategy for novelty search in expensive black-box systems}},
  author={Tang, Wei-Ting and Chakrabarty, Ankush and Paulson, Joel A},
  journal={arXiv preprint arXiv:2406.03616},
  year={2024}
}

@article{snoek2012practical,
  title={Practical bayesian optimization of machine learning algorithms},
  author={Snoek, Jasper and Larochelle, Hugo and Adams, Ryan P},
  journal={Advances in neural information processing systems},
  volume={25},
  year={2012}
}

@article{lindauer2022smac3,
  title={{SMAC3: A versatile Bayesian optimization package for hyperparameter optimization}},
  author={Lindauer, Marius and Eggensperger, Katharina and Feurer, Matthias and Biedenkapp, Andr{\'e} and Deng, Difan and Benjamins, Carolin and Ruhkopf, Tim and Sass, Ren{\'e} and Hutter, Frank},
  journal={Journal of Machine Learning Research},
  volume={23},
  number={54},
  pages={1--9},
  year={2022}
}

@inproceedings{sabbatella2024bayesian,
  title={A {B}ayesian Approach for Prompt Optimization in {LLMs}},
  author={Sabbatella, Antonio and Ponti, Andrea and Giordani, Ilaria and Archetti, Francesco},
  booktitle={International Conference on Learning and Intelligent Optimization},
  pages={348--360},
  year={2024},
  organization={Springer}
}

@article{berkenkamp2023bayesian,
  title={Bayesian optimization with safety constraints: safe and automatic parameter tuning in robotics},
  author={Berkenkamp, Felix and Krause, Andreas and Schoellig, Angela P},
  journal={Machine Learning},
  volume={112},
  number={10},
  pages={3713--3747},
  year={2023},
  publisher={Springer}
}

@inproceedings{paulson2023tutorial,
  title={A tutorial on derivative-free policy learning methods for interpretable controller representations},
  author={Paulson, Joel A and Sorourifar, Farshud and Mesbah, Ali},
  booktitle={Proceedings of the American Control Conference},
  year={2023},
  organization={IEEE}
}

@incollection{frazier2015bayesian,
  title={Bayesian optimization for materials design},
  author={Frazier, Peter I and Wang, Jialei},
  booktitle={Information Science for Materials Discovery and Design},
  pages={45--75},
  year={2015},
  publisher={Springer}
}

@inproceedings{de2021high,
  title={High-dimensional {G}aussian process inference with derivatives},
  author={De Roos, Filip and Gessner, Alexandra and Hennig, Philipp},
  booktitle={International Conference on Machine Learning},
  pages={2535--2545},
  year={2021},
  organization={PMLR}
}

@article{letham2020re,
  title={Re-examining linear embeddings for high-dimensional {B}ayesian optimization},
  author={Letham, Ben and Calandra, Roberto and Rai, Akshara and Bakshy, Eytan},
  journal={Advances in Neural Information Processing Systems},
  volume={33},
  pages={1546--1558},
  year={2020}
}

@article{jones1998efficient,
  title={Efficient global optimization of expensive black-box functions},
  author={Jones, Donald R and Schonlau, Matthias and Welch, William J},
  journal={Journal of Global Optimization},
  volume={13},
  number={4},
  pages={455--492},
  year={1998},
  publisher={Springer}
}

@inproceedings{srinivas2010gaussian,
  title        = {Gaussian process optimization in the bandit setting: No regret and experimental design},
  author       = {Srinivas, N. and Krause, A. and Kakade, S. and Seeger, M.},
  booktitle    = {Proceedings of the 27th International Conference on Machine Learning},
  pages        = {1015--1022},
  year         = {2010},
  organization = {Omnipress}
}

@article{frazier2008knowledge,
  title={A knowledge-gradient policy for sequential information collection},
  author={Frazier, Peter I and Powell, Warren B and Dayanik, Savas},
  journal={SIAM Journal on Control and Optimization},
  volume={47},
  number={5},
  pages={2410--2439},
  year={2008},
  publisher={SIAM}
}

@article{hennig2012entropy,
  title={Entropy search for information-efficient global optimization},
  author={Hennig, Philipp and Schuler, Christian J},
  journal={The Journal of Machine Learning Research},
  volume={13},
  number={1},
  pages={1809--1837},
  year={2012},
  publisher={JMLR. org}
}

@book{garnett2023bayesian,
  title={Bayesian optimization},
  author={Garnett, Roman},
  year={2023},
  publisher={Cambridge University Press}
}

@book{williams2006gaussian,
  title={Gaussian Processes for Machine Learning},
  author={Williams, Christopher KI and Rasmussen, Carl Edward},
  volume={2},
  number={3},
  year={2006},
  publisher={MIT Press, Cambridge, MA}
}

@book{nocedal2006numerical,
  title={Numerical Optimization},
  author={Nocedal, Jorge and Wright, Stephen J},
  year={2006},
  publisher={Springer}
}

@article{frazier2018tutorial,
  title={A tutorial on {B}ayesian optimization},
  author={Frazier, Peter I},
  journal={arXiv preprint arXiv:1807.02811},
  year={2018}
}

@book{bertsekas2016nonlinear,
  title={Nonlinear Programming},
  author={Bertsekas, D.P.},
  publisher={Athena Scientific},
  year={2016},
  edition={3rd},
  address={Belmont, MA}
}

@article{papenmeier2022increasing,
  title={Increasing the scope as you learn: Adaptive {B}ayesian optimization in nested subspaces},
  author={Papenmeier, Leonard and Nardi, Luigi and Poloczek, Matthias},
  journal={Advances in Neural Information Processing Systems},
  volume={35},
  pages={11586--11601},
  year={2022}
}

@article{zhu1997algorithm,
  title={{Algorithm 778: L-BFGS-B: Fortran subroutines for large-scale bound-constrained optimization}},
  author={Zhu, Ciyou and Byrd, Richard H and Lu, Peihuang and Nocedal, Jorge},
  journal={ACM Transactions on mathematical software (TOMS)},
  volume={23},
  number={4},
  pages={550--560},
  year={1997},
  publisher={ACM New York, NY, USA}
}

@article{rahimi2007random,
  title={Random features for large-scale kernel machines},
  author={Rahimi, Ali and Recht, Benjamin},
  journal={Advances in Neural Information Processing Systems},
  volume={20},
  year={2007}
}

@inproceedings{vsehic2022lassobench,
  title={Lassobench: A high-dimensional hyperparameter optimization benchmark suite for lasso},
  author={{\v{S}}ehi{\'c}, Kenan and Gramfort, Alexandre and Salmon, Joseph and Nardi, Luigi},
  booktitle={International Conference on Automated Machine Learning},
  pages={2--1},
  year={2022},
  organization={PMLR}
}

@article{balandat2020botorch,
  title={BoTorch: A framework for efficient Monte-Carlo Bayesian optimization},
  author={Balandat, Maximilian and Karrer, Brian and Jiang, Daniel and Daulton, Samuel and Letham, Ben and Wilson, Andrew G and Bakshy, Eytan},
  journal={Advances in neural information processing systems},
  volume={33},
  pages={21524--21538},
  year={2020}
}

@article{gardner2018gpytorch,
  title={Gpytorch: Blackbox matrix-matrix gaussian process inference with gpu acceleration},
  author={Gardner, Jacob and Pleiss, Geoff and Weinberger, Kilian Q and Bindel, David and Wilson, Andrew G},
  journal={Advances in neural information processing systems},
  volume={31},
  year={2018}
}

@article{siemenn2023fast,
  title={Fast {B}ayesian optimization of Needle-in-a-Haystack problems using zooming memory-based initialization {(ZoMBI)}},
  author={Siemenn, Alexander E and Ren, Zekun and Li, Qianxiao and Buonassisi, Tonio},
  journal={npj Computational Materials},
  volume={9},
  number={1},
  pages={79},
  year={2023},
  publisher={Nature Publishing Group UK London}
}

@misc{towers2024gymnasium,
      title={Gymnasium: A Standard Interface for Reinforcement Learning Environments}, 
      author={Mark Towers and Ariel Kwiatkowski and Jordan Terry and John U. Balis and Gianluca De Cola and Tristan Deleu and Manuel Goulão and Andreas Kallinteris and Markus Krimmel and Arjun KG and Rodrigo Perez-Vicente and Andrea Pierré and Sander Schulhoff and Jun Jet Tai and Hannah Tan and Omar G. Younis},
      year={2024},
      eprint={2407.17032},
      archivePrefix={arXiv},
      primaryClass={cs.LG},
      url={https://arxiv.org/abs/2407.17032}, 
}

@article{kanagawa2025gaussian,
  title={Gaussian Processes and Reproducing Kernels: Connections and Equivalences},
  author={Kanagawa, Motonobu and Hennig, Philipp and Sejdinovic, Dino and Sriperumbudur, Bharath K},
  journal={arXiv preprint arXiv:2506.17366},
  year={2025}
}

@article{firey1960remainder,
  title={Remainder formulae in {T}aylor's theorem},
  author={Firey, William J},
  journal={The American Mathematical Monthly},
  volume={67},
  number={9},
  pages={903--905},
  year={1960},
  publisher={JSTOR}
}

@article{hanzely2022damped,
  title={A Damped Newton Method Achieves Global $\mathcal{O}(1/k^2)$ and Local Quadratic Convergence Rate},
  author={Hanzely, Slavom{\'\i}r and Kamzolov, Dmitry and Pasechnyuk, Dmitry and Gasnikov, Alexander and Richt{\'a}rik, Peter and Tak{\'a}c, Martin},
  journal={Advances in Neural Information Processing Systems},
  volume={35},
  pages={25320--25334},
  year={2022}
}

@article{shea2025greedy,
  title={Greedy {N}ewton: {N}ewton’s method with exact line search},
  author={Shea, Betty and Schmidt, Mark},
  journal={Optimization Letters},
  pages={1--19},
  year={2025},
  publisher={Springer}
}

@inproceedings{martens2010hessianfree,
author = {Martens, James},
title = {Deep learning via {H}essian-free optimization},
year = {2010},
isbn = {9781605589077},
publisher = {Omnipress},
address = {Madison, WI, USA},
booktitle = {Proceedings of the 27th International Conference on International Conference on Machine Learning},
pages = {735–742},
numpages = {8},
location = {Haifa, Israel},
series = {ICML'10}
}

@article{woolson2007wilcoxon,
  title={Wilcoxon signed-rank test},
  author={Woolson, Robert F},
  journal={Wiley Encyclopedia of Clinical Trials},
  pages={1--3},
  year={2007},
  publisher={Wiley Online Library}
}

\section*{Checklist}



\begin{enumerate}

  \item For all models and algorithms presented, check if you include:
  \begin{enumerate}
    \item A clear description of the mathematical setting, assumptions, algorithm, and/or model. [Yes]
    \item An analysis of the properties and complexity (time, space, sample size) of any algorithm. [Yes]
    \item (Optional) Anonymized source code, with specification of all dependencies, including external libraries. [Yes]
  \end{enumerate}

  \item For any theoretical claim, check if you include:
  \begin{enumerate}
    \item Statements of the full set of assumptions of all theoretical results. [Yes]
    \item Complete proofs of all theoretical results. [Yes]
    \item Clear explanations of any assumptions. [Yes]     
  \end{enumerate}

  \item For all figures and tables that present empirical results, check if you include:
  \begin{enumerate}
    \item The code, data, and instructions needed to reproduce the main experimental results (either in the supplemental material or as a URL). [Yes]
    \item All the training details (e.g., data splits, hyperparameters, how they were chosen). [Yes]
    \item A clear definition of the specific measure or statistics and error bars (e.g., with respect to the random seed after running experiments multiple times). [Yes]
    \item A description of the computing infrastructure used. (e.g., type of GPUs, internal cluster, or cloud provider). [Yes]
  \end{enumerate}

  \item If you are using existing assets (e.g., code, data, models) or curating/releasing new assets, check if you include:
  \begin{enumerate}
    \item Citations of the creator If your work uses existing assets. [Yes]
    \item The license information of the assets, if applicable. [Not Applicable]
    \item New assets either in the supplemental material or as a URL, if applicable. [Yes]
    \item Information about consent from data providers/curators. [Not Applicable]
    \item Discussion of sensible content if applicable, e.g., personally identifiable information or offensive content. [Not Applicable]
  \end{enumerate}

  \item If you used crowdsourcing or conducted research with human subjects, check if you include:
  \begin{enumerate}
    \item The full text of instructions given to participants and screenshots. [Not Applicable]
    \item Descriptions of potential participant risks, with links to Institutional Review Board (IRB) approvals if applicable. [Not Applicable]
    \item The estimated hourly wage paid to participants and the total amount spent on participant compensation. [Not Applicable]
  \end{enumerate}

\end{enumerate}

\clearpage
\appendix
\numberwithin{equation}{section}
\numberwithin{figure}{section}
\numberwithin{table}{section}

\thispagestyle{empty}

\onecolumn
\aistatstitle{Supplementary Material}



\section{GAUSSIAN PROCESS DERIVATIVE EXPRESSIONS}
\label{app:gp-expressions}

\subsection{Notation Summary}

Let $\mathcal{X} \subset \mathbb{R}^d$ be a convex, compact domain and let $\mathcal{H}$ be a reproducing kernel Hilbert space (RKHS) on $\mathcal{X}$ with reproducing kernel $k(\cdot,\cdot)$, inner product $\langle \cdot,\cdot\rangle_{\mathcal H}$, and norm $\|\cdot\|_{\mathcal H}$.

For the objective $f:\mathcal{X}\to\mathbb{R}$, denote the gradient and Hessian at $\bs x\in\mathcal X$ by
\begin{align*}
\bs g(\bs x) = \nabla f(\bs x)\in\mathbb{R}^d, \qquad
\bs H(\bs x) = \nabla^2 f(\bs x)\in\mathbb{R}^{d\times d}.    
\end{align*}
A Gaussian process (GP) prior $\mathcal{GP}(\mu,k)$ is specified by mean $\mu$ and covariance $k$ functions. Conditioning on data $\mathcal{D}=\{(\bs X,\bs y)\}$ for $\bs{X} \in \mathbb{R}^{n \times d}$ and $\bs{y} \in \mathbb{R}^n$ yields a posterior GP $f | \mathcal{D}\sim\mathcal{GP}(\mu_{\mathcal D},k_{\mathcal D})$. Since differentiation is linear, $\nabla f$ and $\nabla^2 f$ are also GPs under the same conditioning. We use the shorthand
\[
\widehat{\bs g}_{\mathcal D}(\bs x)=\mathbb{E}\{\bs g(\bs x)\mid\mathcal D\}\in\mathbb{R}^d,\quad
\Sigma^{\bs g}_{\mathcal D}(\bs x)=\mathrm{cov}\{\bs g(\bs x)\mid\mathcal D\}\in\mathbb{R}^{d\times d},
\]
\[
\widehat{\bs H}_{\mathcal D}(\bs x)=\mathbb{E}\{\bs H(\bs x)\mid\mathcal D\}\in\mathbb{R}^{d\times d},\quad
\Sigma^{\bs H}_{\mathcal D}(\bs x)=\mathrm{cov}\{\mathrm{vec}(\bs H(\bs x))\mid\mathcal D\}\in\mathbb{R}^{d^2\times d^2}.
\]
We write $\|\cdot\|$ for the Euclidean norm (vectors) and the induced operator 2-norm (matrices).

\subsection{Posterior GP under Linear Operators}

GPs $\mathcal{GP}(\mu,k)$ are closed under linear operators. For linear operators $\mathcal{L},\mathcal{M}$ and $f\sim\mathcal{GP}(\mu,k)$,
\begin{align*}
\begin{bmatrix}\mathcal{L}f \\ \mathcal{M}f\end{bmatrix}
~\sim~ \mathcal{GP}\!\left(
\begin{bmatrix}\mathcal{L}\mu \\ \mathcal{M}\mu\end{bmatrix},
\begin{bmatrix}
\mathcal{L}k\,\mathcal{L}' & \mathcal{L}k\,\mathcal{M}'\\
\mathcal{M}k\,\mathcal{L}' & \mathcal{M}k\,\mathcal{M}'
\end{bmatrix}\right),
\end{align*}
where $\mathcal{L}'$ and $\mathcal{M}'$ act on the second kernel argument as adjoints of $\mathcal{L}$ and $\mathcal{M}$. We condition on noisy function observations at $\bs X\in\mathbb{R}^{n\times d}$,
$\bs y = f(\bs X)+\bs\varepsilon$, $\bs\varepsilon\sim\mathcal N(\bs0,\sigma^2\bs I)$, take $\mathcal{M}=\mathrm{Id}$ at $\bs X$, and evaluate $\mathcal{L}\in\{\mathrm{Id},\nabla,\nabla^2\}$ at a test point $\bs x$. Then, for any $\mathcal L$, the posterior
$\mathcal{L}f(\bs x)\mid\mathcal D \sim \mathcal N(\mu^\mathcal{L}_{\mathcal D}(\bs x),\,\Sigma^\mathcal{L}_{\mathcal D}(\bs x))$ with
\begin{subequations}\label{eq:posterior-gp-L}
\begin{align}
\mu^\mathcal{L}_{\mathcal D}(\bs x)
&= \mathcal{L}\mu(\bs x) + \mathcal{L}k(\bs x,\bs X)\,\big(k(\bs X,\bs X)+\sigma^2\bs I\big)^{-1}\!\big(\bs y-\mu(\bs X)\big),\\
k^\mathcal{L}_{\mathcal D}(\bs x,\bs x')
&= \mathcal{L}k(\bs x,\bs x')\,\mathcal{L}'
- \mathcal{L}k(\bs x,\bs X)\,\big(k(\bs X,\bs X)+\sigma^2\bs I\big)^{-1}\,k(\bs X,\bs x')\,\mathcal{L}',\\
\Sigma^\mathcal{L}_{\mathcal D}(\bs x)
&= k^\mathcal{L}_{\mathcal D}(\bs x,\bs x')\big|_{\bs x'=\bs x}.
\end{align}
\end{subequations}
Let $\bs K_{\bs X\bs X}=k(\bs X,\bs X)+\sigma^2\bs I$ and $\bs\alpha=\bs K_{\bs X\bs X}^{-1}\big(\bs y-\mu(\bs X)\big)$; these can be precomputed from the prior. 

\subsection{Power Functions for Gradient and Hessian}

We quantify posterior uncertainty via the (squared) \textit{power functions} (traces of derivative covariances) at $\bs x$:
\begin{align*}
& \pi^{\bs g}_{\mathcal D}(\bs x)
= \mathrm{tr}\left(\Sigma^{\bs g}_{\mathcal D}(\bs x)\right)
= \sum_{i=1}^d \left.\frac{\partial^2 k_{\mathcal D}(\bs x,\bs x')}{\partial x_i\,\partial x'_i}\right|_{\bs x'=\bs x}, \qquad \pi^{\bs H}_{\mathcal D}(\bs x)
= \mathrm{tr}\left(\Sigma^{\bs H}_{\mathcal D}(\bs x)\right)
= \sum_{i=1}^d\sum_{j=1}^d \left.\frac{\partial^4 k_{\mathcal D}(\bs x,\bs x')}{\partial x_i\,\partial x_j\,\partial x'_i\,\partial x'_j}\right|_{\bs x'=\bs x}.
\end{align*}
Using \eqref{eq:posterior-gp-L}, these admit closed forms:
\begin{subequations}\label{eq:power-functions-expanded}
\begin{align}
\pi^{\bs{g}}_\mathcal{D}(\bs{x})
&= \sum_{i=1}^d \left[
\frac{\partial^2 k(\bs{x}, \bs{x}')}{\partial x_i \partial x'_i}
- \frac{\partial k(\bs{x}, \bs{X})}{\partial x_i}\,
\bs K_{\bs X\bs X}^{-1}\,
\frac{\partial k(\bs{X}, \bs{x}')}{\partial x'_i}
\right]_{\bs{x}'=\bs{x}},\\
\pi^{\bs{H}}_\mathcal{D}(\bs{x})
&= \sum_{i=1}^d\sum_{j=1}^d \left[
\frac{\partial^4 k(\bs{x}, \bs{x}')}{\partial x_i \partial x_j \partial x'_i \partial x'_j}
- \frac{\partial^2 k(\bs{x}, \bs{X})}{\partial x_i \partial x_j}\,
\bs K_{\bs X\bs X}^{-1}\,
\frac{\partial^2 k(\bs{X}, \bs{x}')}{\partial x'_i \partial x'_j}
\right]_{\bs{x}'=\bs{x}}.
\end{align}
\end{subequations}

\subsection{Derivatives for the Squared Exponential (SE) Kernel}

We use the SE kernel with automatic relevance determination (ARD) (i.e., independent length-scales $\ell_i$ are included for each dimension to control their importance)
\[
k(\bs x,\bs x')
= \sigma_f^2\,\exp\!\left(-\tfrac{1}{2}(\bs x-\bs x')^\top \bs L\,(\bs x-\bs x')\right),
\]
with scale hyperparameter $\sigma_f^2$ that controls the expected variance of $f$ and $\bs L=\mathrm{diag}(\ell_1^{-2},\ldots,\ell_d^{-2})$. Let $L_{ii}=\ell_i^{-2}$ and $\bs r=\bs x-\bs x'$ with $r_i$ denoting its $i$-th component.

\paragraph{First derivatives.}
\[
\frac{\partial k(\bs x,\bs x')}{\partial x_i} = -L_{ii}\,r_i\,k(\bs x,\bs x'),\qquad
\frac{\partial k(\bs x,\bs x')}{\partial x'_j} = +L_{jj}\,r_j\,k(\bs x,\bs x').
\]

\paragraph{Second derivatives (mixed across arguments).}
\[
\frac{\partial^2 k(\bs x,\bs x')}{\partial x_i\,\partial x'_j}
= \Big(L_{ii}\,\delta_{ij} - L_{ii}L_{jj}\,r_i r_j\Big)\,k(\bs x,\bs x'),
\qquad
\left.\frac{\partial^2 k}{\partial x_i\,\partial x'_j}\right|_{\bs x'=\bs x}
= L_{ii}\,\delta_{ij}\,\sigma_f^2.
\]

\paragraph{Second derivatives (same argument).}
\[
\frac{\partial^2 k(\bs x,\bs x')}{\partial x_i\,\partial x_j}
= \frac{\partial^2 k(\bs x,\bs x')}{\partial x'_i\,\partial x'_j}
= \Big(-L_{ii}\,\delta_{ij} + L_{ii}L_{jj}\,r_i r_j\Big)\,k(\bs x,\bs x').
\]

\paragraph{Fourth derivatives (two in each argument).}
For the Hessian trace terms we require
\[
\frac{\partial^4 k(\bs x,\bs x')}{\partial x_i\,\partial x_j\,\partial x'_i\,\partial x'_j}
=
\begin{cases}
L_{ii}^2\big(L_{ii}^2 r_i^4 - 6 L_{ii} r_i^2 + 3\big)\,k(\bs x,\bs x') & \text{if } i=j,\\[3pt]
L_{ii}L_{jj}\big(L_{ii}L_{jj} r_i^2 r_j^2 - L_{ii} r_i^2 - L_{jj} r_j^2 + 1\big)\,k(\bs x,\bs x') & \text{if } i\neq j,
\end{cases}
\]
and at coincidence $\bs x'=\bs x$ these reduce to
\[
\left.\frac{\partial^2 k}{\partial x_i\,\partial x_j}\right|_{\bs x'=\bs x} = -L_{ii}\,\delta_{ij}\,\sigma_f^2,\qquad
\left.\frac{\partial^4 k}{\partial x_i\,\partial x_j\,\partial x'_i\,\partial x'_j}\right|_{\bs x'=\bs x}
=
\sigma_f^2\times
\begin{cases}
3\,L_{ii}^2, & i=j,\\
L_{ii}L_{jj}, & i\neq j.
\end{cases}
\]

\section{PROOF OF DATA-DEPENDENT NEWTON-STEP ERROR BOUND}
\label{app:proof-newton}

We prove Theorem \ref{thm:newton-bound} from the main text in this appendix. Notation for the GP posterior and derivative processes (including $\widehat{\bs g}_{\mathcal D}$, $\widehat{\bs H}_{\mathcal D}$, and the gradient/Hessian power functions
$\pi^{\bs g}_{\mathcal D}$, $\pi^{\bs H}_{\mathcal D}$) is summarized in Appendix \ref{app:gp-expressions}.

Throughout this work, we make the following standard assumptions about the regularity of the kernel and boundedness of the ground-truth function. 

\begin{assumption}[Kernel regularity]
\label{asmp:kernel}
The kernel $k$ is stationary and four times continuously differentiable.
\end{assumption}

\begin{assumption}[Function class]
\label{asmp:rkhs_bound}
The ground-truth $f$ is in $\mathcal{H}$ and satisfies $\|f\|_{\mathcal H} \leq B$ for some $B<\infty$.
\end{assumption}

These two assumptions directly imply pointwise boundedness of all the first- and second-order derivatives of $f$ in terms of the RKHS norm. 

\subsection{Auxiliary Bounds via Power Functions}

The following are minor extensions of standard RKHS ``power function'' bounds for derivative estimates under GP posteriors, which are a consequence of \citep[Theorem 11.4]{wendland2004scattered}. The first was presented in \citep[Lemma 1]{wu2023behavior} and the second we prove here.

\begin{lemma}[Gradient posterior error]\label{lem:grad}
For any $\bs x\in\mathcal X$ and dataset $\mathcal D$,
\[
\big\|\bs g(\bs x)-\widehat{\bs g}_{\mathcal D}(\bs x)\big\|^2
~\le~ \pi^{\bs g}_{\mathcal D}(\bs x)\,\|f\|_{\mathcal H}^2.
\]
\end{lemma}

\begin{lemma}[Hessian posterior error]\label{lem:hess}
For any $\bs x\in\mathcal X$ and dataset $\mathcal D$,
\[
\big\|\bs H(\bs x)-\widehat{\bs H}_{\mathcal D}(\bs x)\big\|^2
~\le~ \pi^{\bs H}_{\mathcal D}(\bs x)\,\|f\|_{\mathcal H}^2.
\]
\end{lemma}

\begin{proof}
    Let $\lambda : \mathcal{H} \to \mathbb{R}$ be the composition of the evaluation operator and a differential operator. \cite[Theorem 11.4]{wendland2004scattered} provides a bound on the squared error between the operator $\lambda$ applied to the true function $f$ and the posterior mean function $\mu_\mathcal{D}$, i.e.,
    \begin{align*}
        ( \lambda f(\bs{x}) - \lambda\mu_\mathcal{D}(\bs{x}) )^2 \leq \lambda^{(1)}\lambda^{(2)} k_\mathcal{D}( \bs{x}, \bs{x} ) \| f \|_\mathcal{H}^2,
    \end{align*}
    where $\lambda^{(1)}$ and $\lambda^{(2)}$ are applied to the first and second argument of $k_\mathcal{D}( \cdot, \cdot )$, respectively. Select the linear functional to be the second partial derivative $\lambda : f \mapsto \frac{\partial^2}{\partial x_i \partial x_j} f$. The left hand side of the inequality then becomes $\left( \frac{\partial^2}{\partial x_i \partial x_j} f(\bs{x}) - \frac{\partial^2}{\partial x_i \partial x_j} \mu_\mathcal{D}(\bs{x})\right)^2$, which is the error in the $n$-th element of the vectorized Hessian matrix where $n = (i-1)d + j$. The right hand side is exactly the $n$-th diagonal entry of $\Sigma_\mathcal{D}^{\bs{H}}(\bs{x})$. We can use this inequality for every element $n$ and sum over $n = 1,\ldots,d^2$ to arrive at the Frobenius norm of the error in the Hessian matrix: $\| \bs{H}(\bs{x}) - \widehat{\bs{H}}_\mathcal{D}(\bs{x}) \|_F^2 = \sum_{i=1}^d\sum_{j=1}^d \left( \frac{\partial^2}{\partial x_i \partial x_j} f(\bs{x}) - \frac{\partial^2}{\partial x_i \partial x_j} \mu_\mathcal{D}(\bs{x})\right)^2$. We can then use the standard inequality $\| \bs{A} \| \leq \| \bs{A} \|_F$ for any square matrix $\bs{A}$ to complete the proof.
\end{proof}

\subsection{Proof of Theorem~\ref{thm:newton-bound} -- Bounding the Newton-Step Error}

Recall from the definitions provided in the theorem statement that $\varepsilon_{\mathcal D}(\bs x)=\|\bs d(\bs x)-\widehat{\bs d}_{\mathcal D}(\bs x)\|$, where $\bs d(\bs x)=\bs H(\bs x)^{-1}\bs g(\bs x)$ and $\widehat{\bs d}_{\mathcal D}(\bs x)=\widehat{\bs H}_{\mathcal D}(\bs x)^{-1}\widehat{\bs g}_{\mathcal D}(\bs x)$. Suppressing the explicit dependence on $\bs x$ for readability:
\[
\bs d-\widehat{\bs d}_{\mathcal D}
= \bs H^{-1}\bs g - \widehat{\bs H}_{\mathcal D}^{-1}\widehat{\bs g}_{\mathcal D}
= \underbrace{\bs H^{-1}\big(\bs g-\widehat{\bs g}_{\mathcal D}\big)}_{(\bs{a})}
+ \underbrace{\big(\bs H^{-1}-\widehat{\bs H}_{\mathcal D}^{-1}\big)\widehat{\bs g}_{\mathcal D}}_{(\bs{b})}.
\]
For any $\bs{a},\bs{b}$ in an inner-product space,
\[
\|\bs{a}+\bs{b}\|^2
= \|\bs{a}\|^2+\|\bs{b}\|^2 + 2\langle \bs{a},\bs{b}\rangle
\le \|\bs{a}\|^2+\|\bs{b}\|^2 + 2\|\bs{a}\|\,\|\bs{b}\|
\le 2\|\bs{a}\|^2+2\|\bs{b}\|^2,
\]
where the first is the Cauchy-Schwarz inequality and the second uses Young's inequality (i.e., $2uv\le u^2+v^2$ or equivalently $(u-v)^2\ge 0$). Combining this with submultiplicativity, we get:
\[
\varepsilon_{\mathcal D}^2
\le 2\|\bs H^{-1}\|^2\,\|\bs g-\widehat{\bs g}_{\mathcal D}\|^2
 + 2\|\bs H^{-1}-\widehat{\bs H}_{\mathcal D}^{-1}\|^2\,\|\widehat{\bs g}_{\mathcal D}\|^2.
\]
Use the resolvent identity
$\bs H^{-1}-\widehat{\bs H}_{\mathcal D}^{-1}
=\bs H^{-1}\big(\widehat{\bs H}_{\mathcal D}-\bs H\big)\widehat{\bs H}_{\mathcal D}^{-1}$
to bound the second term by
\(
\|\bs H^{-1}\|^2\,\|\widehat{\bs H}_{\mathcal D}-\bs H\|^2\,\|\widehat{\bs H}_{\mathcal D}^{-1}\|^2\,\|\widehat{\bs g}_{\mathcal D}\|^2.
\)
Applying Lemmas~\ref{lem:grad}–\ref{lem:hess} and $\|f\|_{\mathcal H}\le B$ gives
\[
\varepsilon_{\mathcal D}^2
\le 2\,B^2\,\|\bs H^{-1}\|^2\!\left[
\pi^{\bs g}_{\mathcal D}(\bs x)
+ \underbrace{\|\widehat{\bs H}_{\mathcal D}(\bs x)^{-1}\|^2\,\|\widehat{\bs g}_{\mathcal D}(\bs x)\|^2}_{s_{\mathcal D}(\bs x)}
\,\pi^{\bs H}_{\mathcal D}(\bs x)\right].
\]
Equivalently,
\(
\varepsilon_{\mathcal D}(\bs x)
\le C_{\bs x}\sqrt{\pi^{\bs g}_{\mathcal D}(\bs x)+s_{\mathcal D}(\bs x)\,\pi^{\bs H}_{\mathcal D}(\bs x)}
\)
with $C_{\bs x}=\sqrt{2}B\|\bs H(\bs x)^{-1}\|$, as stated.
\hfill$\square$

\paragraph{What the bound says.} The error in the Newton step decomposes into (i) gradient uncertainty and (ii) Hessian uncertainty at $\bs x$ \textit{scaled} by a factor $s_{\mathcal D}(\bs x)=\|\widehat{\bs H}_{\mathcal D}(\bs x)^{-1}\|^2\,\|\widehat{\bs g}_{\mathcal D}(\bs x)\|^2$. This scale is large precisely when the local problem is ill-conditioned and/or when the gradient is sizeable, so the bound quantitatively formalizes when learning curvature is disproportionately valuable.

\section{EMPIRICAL STUDY OF THE SCALE FACTOR'S IMPACT}
\label{app:scale-factor}

\subsection{Details on Monte Carlo estimation of NeST acquisition}
\label{app:mc-est-scale}

For completeness, we restate the complete NeST acquisition shown in \eqref{eq:nest-rearranged} here:
\begin{align}\label{eq:nest-rearranged-app}
\tilde{\alpha}_{\mathrm{NeST}}(\bs Z \mid \bs x_t,\mathcal D)
=
\pi^{\bs g}_{\mathcal D\cup\bs Z}(\bs x_t)
+
\mathbb{E}_{\bs y \mid \mathcal D,\bs Z}\!\left[\, s_{\mathcal D\cup(\bs Z,\bs y)}(\bs x_t) \,\right]\,
\pi^{\bs H}_{\mathcal D\cup\bs Z}(\bs x_t),
\end{align}
where $\bs y$ denotes the (noisy) batch observations at the candidate locations $\bs Z$,
i.e., $\bs y = f(\bs Z) + \bs\varepsilon$ with $\bs\varepsilon\sim\mathcal N(\bs 0,\sigma^2 I)$.
The two power-function terms in \eqref{eq:nest-rearranged-app} depend only on the design $\bs Z$ (and the GP hyperparameters),
whereas the scale factor depends on the \emph{realized} post-batch posterior mean of the gradient and Hessian at $\bs x_t$, i.e., $s_{\mathcal{D}\cup(\bs Z,\bs y)}(\bs x) = \big\| \widehat{\bs H}_{\mathcal{D}\cup(\bs Z,\bs y)}(\bs x)^{-1} \big\|^2\,
\big\| \widehat{\bs g}_{\mathcal{D}\cup(\bs Z,\bs y)}(\bs x) \big\|^2$. 

Under a GP with Gaussian observation noise, the random vector $\bs y \mid \mathcal D,\bs Z$ is multivariate Gaussian,
\begin{align}\label{eq:y-posterior}
\bs y | \mathcal D,\bs Z \sim \mathcal N\!\big(\,\bs\mu_{\mathcal D}(\bs Z),\, \Sigma_{\mathcal D}(\bs Z)+\sigma^2 I\,\big),
\end{align}
where $\bs\mu_{\mathcal D}(\bs Z)$ and $\Sigma_{\mathcal D}(\bs Z)$ are the GP posterior mean and covariance of $f(\bs Z) | \mathcal D$. 
However, the mapping $\bs y \mapsto s_{\mathcal D\cup(\bs Z,\bs y)}(\bs x_t)$ is nonlinear and thus in general is non-Gaussian; this makes the expectation in \eqref{eq:nest-rearranged-app} analytically intractable, motivating the need for Monte Carlo (MC) approximation.

Let $\Sigma_y = \Sigma_{\mathcal D}(\bs Z)+\sigma^2 I$ and $\bs\mu_y = \bs\mu_{\mathcal D}(\bs Z)$.
We approximate the expectation in \eqref{eq:nest-rearranged-app} by drawing $S$ samples
\begin{align}\label{eq:mc-y-samples}
\bs y^{(s)} \sim \mathcal N(\bs\mu_y,\Sigma_y), \qquad s=1,\dots,S,
\end{align}
and computing
\begin{align}\label{eq:mc-estimate-scale}
\mathbb E\!\left[s_{\mathcal D\cup(\bs Z,\bs y)}(\bs x_t)\right]
\approx 
\frac{1}{S}\sum_{s=1}^S
s_{\mathcal D\cup(\bs Z,\bs y^{(s)})}(\bs x_t).
\end{align}
Substituting \eqref{eq:mc-estimate-scale} into \eqref{eq:nest-rearranged-app} yields the MC-approximated acquisition
\begin{align}\label{eq:nest-mc}
\tilde{\alpha}^{\mathrm{MC}}_{\mathrm{NeST}}(\bs Z | \bs x_t,\mathcal D)
=
\pi^{\bs g}_{\mathcal D\cup\bs Z}(\bs x_t)
+
\left(
\frac{1}{S}\sum_{s=1}^S
s_{\mathcal D\cup(\bs Z,\bs y^{(s)})}(\bs x_t)
\right)
\pi^{\bs H}_{\mathcal D\cup\bs Z}(\bs x_t),
\end{align}
where $s_{\mathcal D\cup(\bs Z,\bs y^{(s)})}(\bs x_t)$ can be evaluated using standard GP conditioning rules described in Appendix \ref{app:gp-expressions}. 

\subsection{Comparison of scale factor selection methods on synthetic problems}
\label{app:comparison-scale}

The practical NeST acquisition \eqref{eq:nest-approx} derived in the main text is
\[
\widehat{\alpha}_{\mathrm{NeST}}(\bs Z | \bs x_t,\mathcal D, \widehat{s}_t)
=\pi^{\bs g}_{\mathcal D\cup \bs Z}(\bs x_t)+ \widehat{s}_t\,\pi^{\bs H}_{\mathcal D\cup \bs Z}(\bs x_t),
\]
i.e., a weighted sum of the local gradient and Hessian power functions. While the one-step lookahead form \eqref{eq:nest-rearranged} includes an expectation over future observations, estimating that expectation by Monte Carlo can be costly. This appendix asks: \textit{how sensitive is performance to the choice of the scale factor $\widehat{s}_t$?}

\paragraph{Setup.} We compare four acquisition functions for \textit{sequential} design over $b_t = d$ points:
\begin{itemize}
    \item $\widehat{\alpha}_{\mathrm{NeST}}(s{=}1)$ with $\widehat{s}_t  = 1$ (our default);
    \item $\widehat{\alpha}_{\mathrm{NeST}}(s{=}\text{plugin})$ with $\widehat{s}_t = s_{\mathcal D}(\bs x_t)=\|\widehat{\bs H}_{\mathcal D}(\bs x_t)^{-1}\|^2\|\widehat{\bs g}_{\mathcal D}(\bs x_t)\|^2$;
    \item $\tilde{\alpha}_{\mathrm{NeST}}(\text{MC})$, which uses a MC estimate of the expectation (with 32 samples) in \eqref{eq:nest-rearranged}, i.e., \eqref{eq:nest-mc} with $S=32$;
    \item the GIBO gradient–information rule $\tilde{\alpha}_{\mathrm{GI}}$.
\end{itemize}
We also include a random sampling (RS) baseline.
At each iteration, we \textit{minimize} the chosen acquisition over the domain using L-BFGS \citep{zhu1997algorithm} with 20 random multistarts.
We study the Griewank, Ackley, Rosenbrock, and Sphere functions in dimensions $d \in \{2,3,4,5\}$ on $[0,1]^d$. The mathematical expressions for each function can be found in Appendix \ref{app:synthetic-details}. 
For each $d$, we evaluate the Newton-step error $\varepsilon_\mathcal{D}(\bs{x})=\|\bs d(\bs x)-\widehat{\bs d}_{\mathcal D}(\bs x)\|$ at 10 randomly chosen test locations $\bs{x}$ and report the per-iteration median across 10 replicates. The GP hyperparameters are fixed across methods: they are fit once (on a separate set of samples) and then held constant. For the Griewank function, initial designs use 5/10/20/30 points for $d{=}2/3/4/5$. For all the other synthetic functions, initial designs use 5/5/10/10 points for $d{=}2/3/4/5$.

\paragraph{Results.} Figure~\ref{fig:Acq_comparison}--\ref{fig:Acq_comparison_Sphere} illustrate two main messages: (i) $\alpha_{\mathrm{NeST}}(s{=}1)$ and the plug-in/MC variants reduce Newton-step error at similar rates across all $d$ and consistently outperform $\alpha_{\mathrm{GI}}$ and RS (though the MC has more variability due to its inherent randomness); and
(ii) this advantage translates into lower error, as seen in the final error distributions across the different runs (bottom row). The gap widens with dimension, where curvature information becomes more important.
A fixed, data-agnostic weight $s_t{=}1$ appears to be a robust and inexpensive choice: it matches the plug-in and MC versions while avoiding extra computation and hyper-sensitivity. This supports our the default selected in all experiments in the main text. It would be interesting to more carefully study, either theoretically or empirically, how the tuning of $\widehat{s}_t$ impacts optimization performance; this type of analysis was outside the scope of this initial contribution. Our results further suggest that NeST-BO's gains might come from targeting curvature at all compared to delicate tuning of $\widehat{s}_t$ -- but this yet to be rigorously formalized. 

\begin{figure}[tb!]
  \centering
  \includegraphics[width=0.9\textwidth]{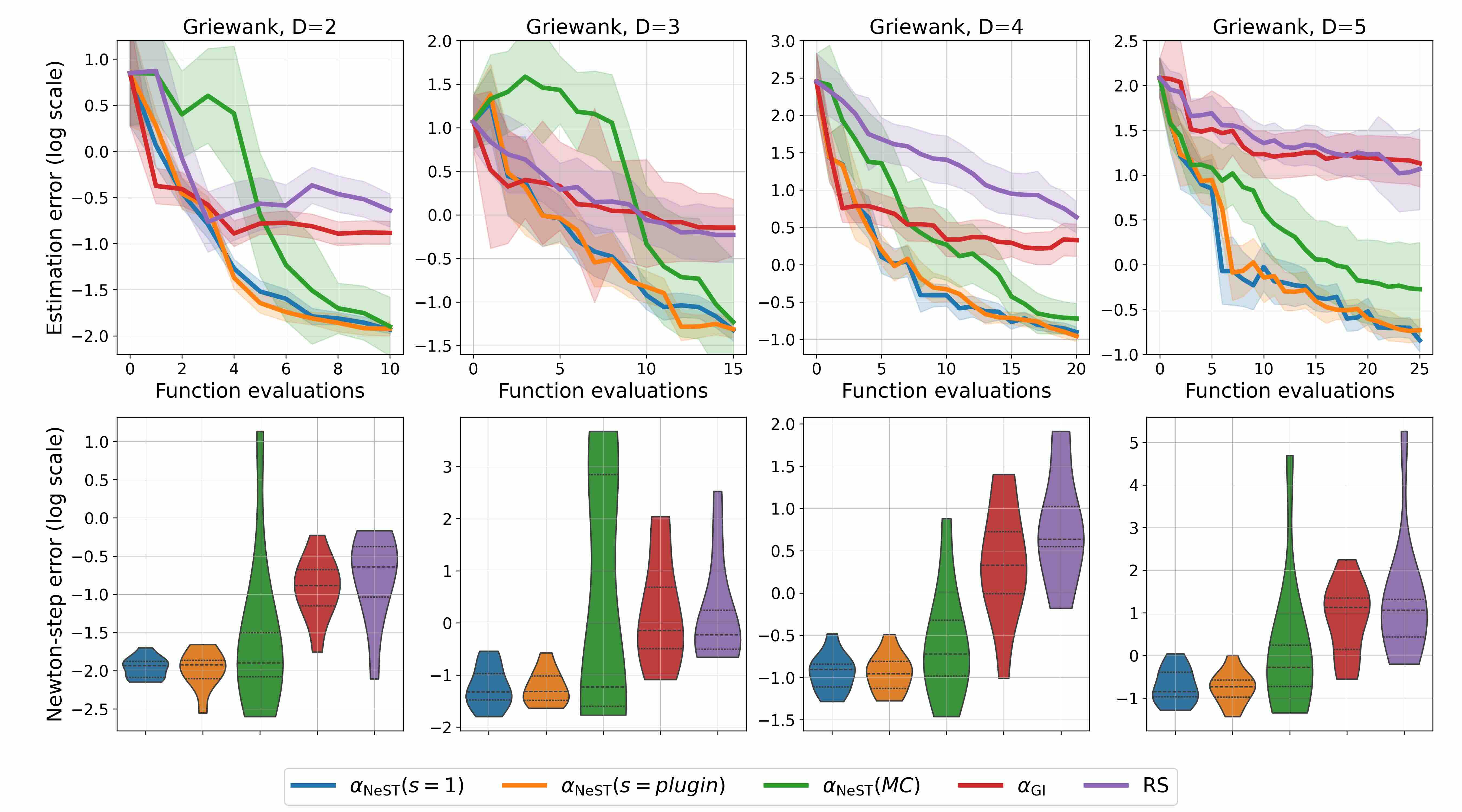}
  \caption{Empirical study of scale factor sensitivity (Griewank).
  \textbf{Top:} Median Newton-step error (log scale) versus number of function evaluations.
  \textbf{Bottom:} Distribution of Newton-step errors (log scale) at the final iteration. NeST with a fixed $s = 1$ closely tracks the plug-in and Monte Carlo (MC) sampling variants and consistently beats $\alpha_{\mathrm{GI}}$, the core acquisition underpinning the GIBO method \citep{muller2021local}, and random sampling (RS). All experiments were replicated 10 times and the shaded regions show $\pm$ one standard error.}
  \label{fig:Acq_comparison}
\end{figure}

\begin{figure}[tb!]
  \centering
  \includegraphics[width=0.9\textwidth]{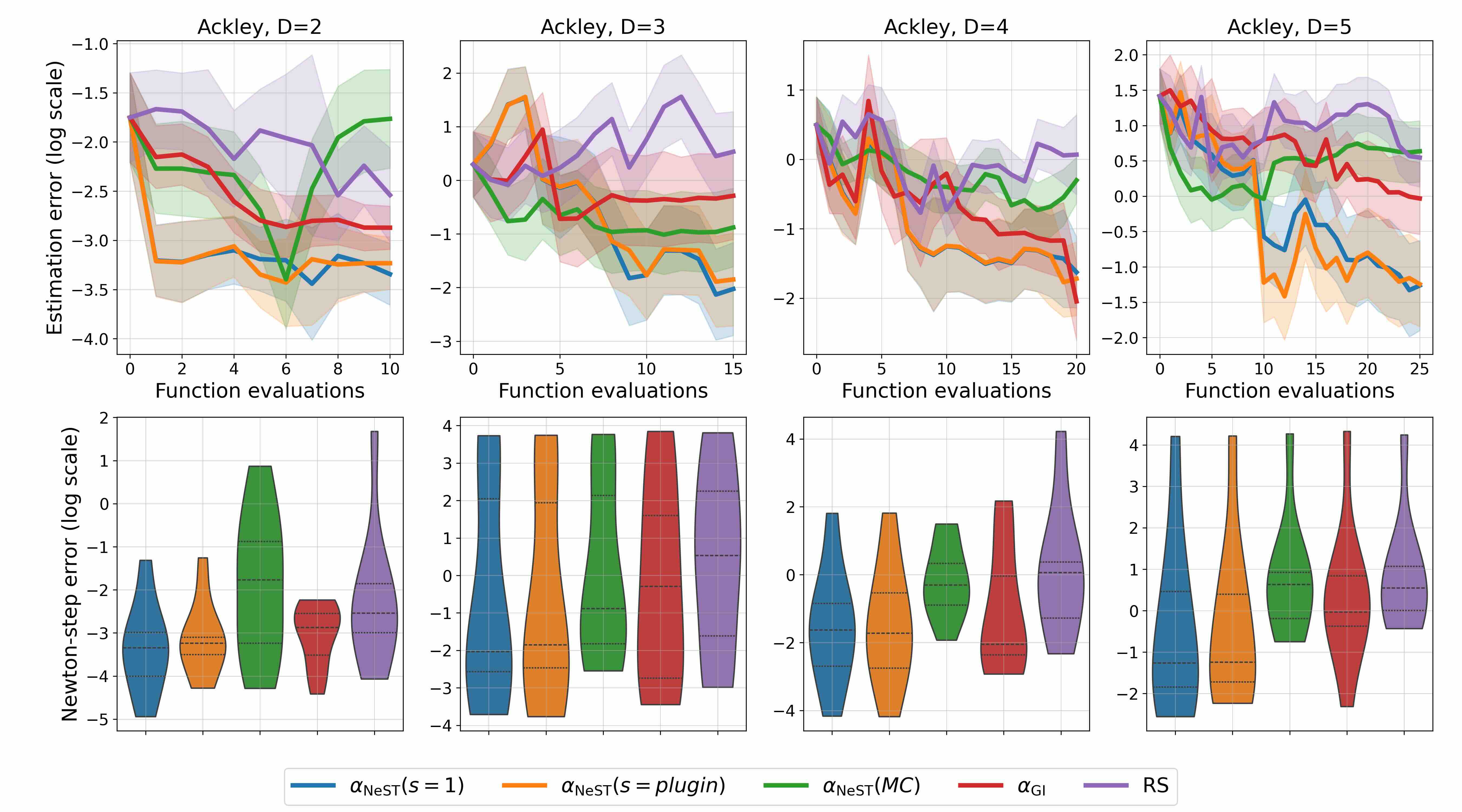}
  \caption{Empirical study of scale factor sensitivity (Ackley).
  \textbf{Top:} Median Newton-step error (log scale) versus number of function evaluations.
  \textbf{Bottom:} Distribution of Newton-step errors (log scale) at the final iteration. NeST with a fixed $s = 1$ closely tracks the plug-in and Monte Carlo (MC) sampling variants and consistently beats $\alpha_{\mathrm{GI}}$, the core acquisition underpinning the GIBO method \citep{muller2021local}, and random sampling (RS). All experiments were replicated 10 times and the shaded regions show $\pm$ one standard error.}
  \label{fig:Acq_comparison_Ackley}
\end{figure}

\begin{figure}[tb!]
  \centering
  \includegraphics[width=0.9\textwidth]{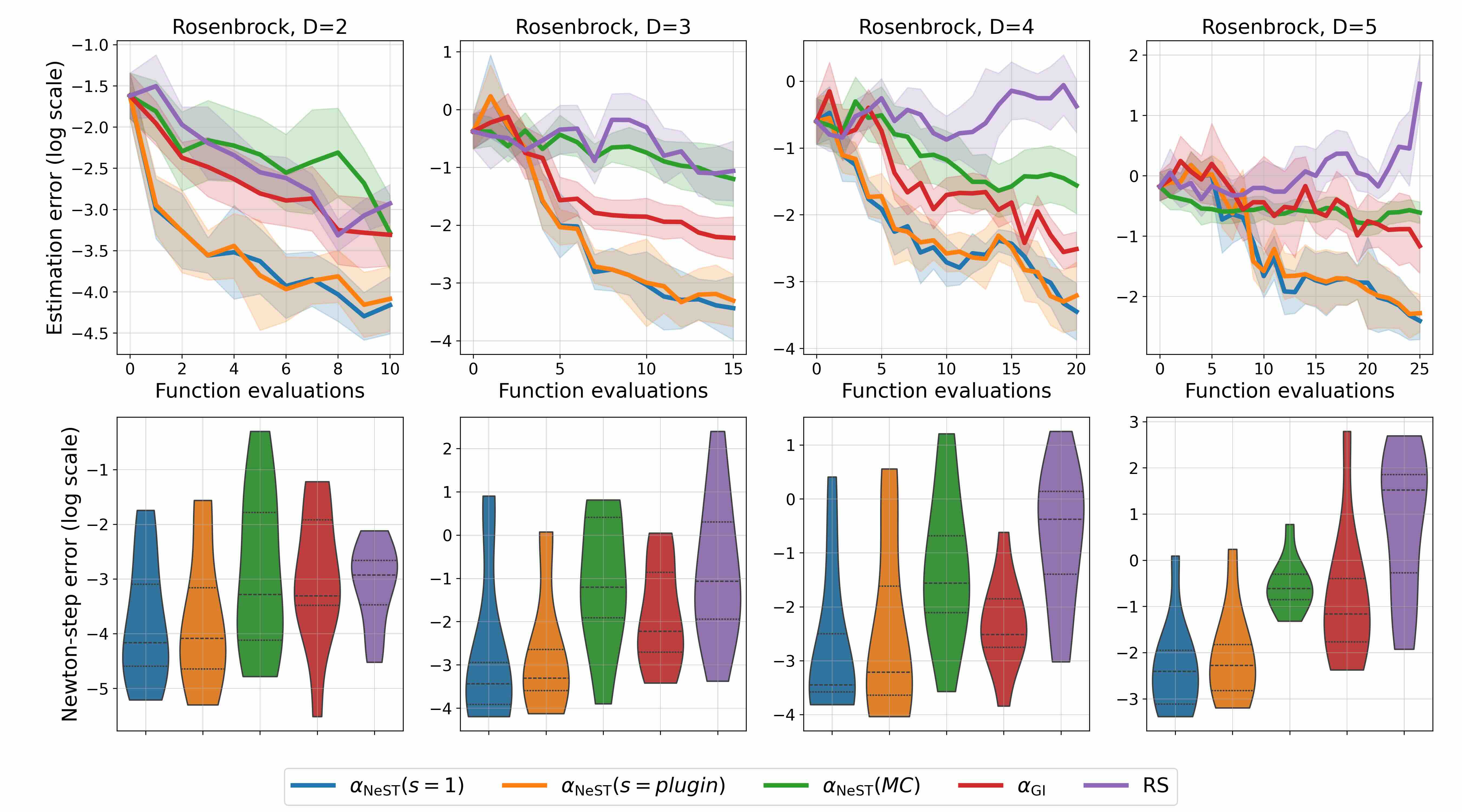}
  \caption{Empirical study of scale factor sensitivity (Rosenbrock).
  \textbf{Top:} Median Newton-step error (log scale) versus number of function evaluations.
  \textbf{Bottom:} Distribution of Newton-step errors (log scale) at the final iteration. NeST with a fixed $s = 1$ closely tracks the plug-in and Monte Carlo (MC) sampling variants and consistently beats $\alpha_{\mathrm{GI}}$, the core acquisition underpinning the GIBO method \citep{muller2021local}, and random sampling (RS). All experiments were replicated 10 times and the shaded regions show $\pm$ one standard error.}
  \label{fig:Acq_comparison_Rosenbrock}
\end{figure}

\begin{figure}[tb!]
  \centering
  \includegraphics[width=0.9\textwidth]{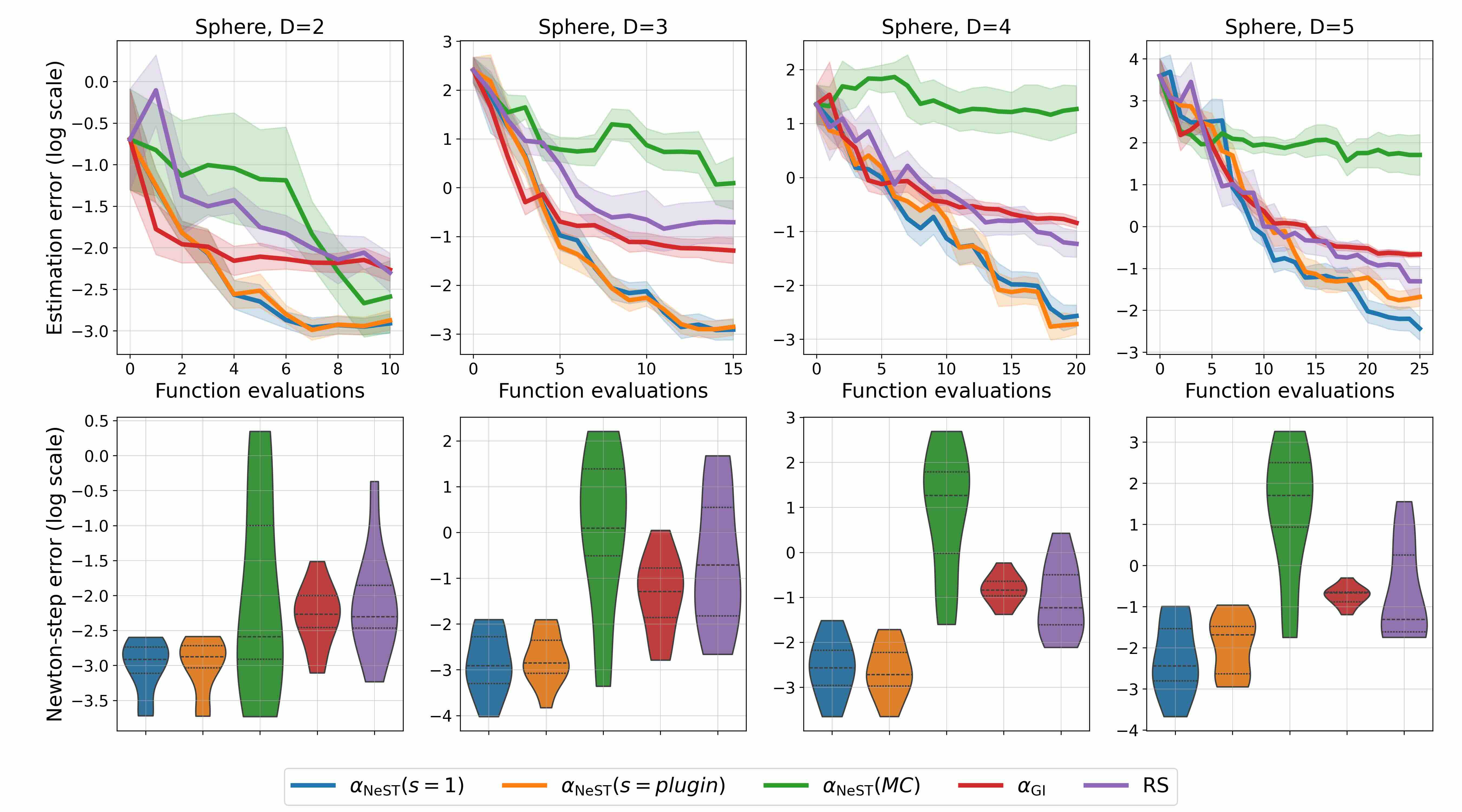}
  \caption{Empirical study of scale factor sensitivity (Sphere).
  \textbf{Top:} Median Newton-step error (log scale) versus number of function evaluations.
  \textbf{Bottom:} Distribution of Newton-step errors (log scale) at the final iteration. NeST with a fixed $s = 1$ closely tracks the plug-in and Monte Carlo (MC) sampling variants and consistently beats $\alpha_{\mathrm{GI}}$, the core acquisition underpinning the GIBO method \citep{muller2021local}, and random sampling (RS). All experiments were replicated 10 times and the shaded regions show $\pm$ one standard error.}
  \label{fig:Acq_comparison_Sphere}
\end{figure}

\subsection{Impact of scale factor on overall optimization performance}
\label{app:plugin-vs-s1}

In the main experiments we default to a fixed scale $\widehat{s}_t=1$, mainly to avoid introducing an additional moving part. As an additional sanity check, we compare this default to the ``plug-in'' alternative that performed similarly in our previous tests (Appendix \ref{app:comparison-scale} and can be more easily justified in that it replaces the expected lookahead version with its current estimate. This plug-in choice is also appealing because it is ``free'' to compute once $\widehat{\bs g}_{\mathcal D}(\bs x_t)$ and $\widehat{\bs H}_{\mathcal D}(\bs x_t)$ are available. 

Figures~\ref{fig:Griewank_plugin_s=1} and \ref{fig:Rover_plugin_s=1}, respectively, show full optimization runs on a synthetic $20$-D Griewank function and a $60$-D rover trajectory benchmark (each over $10$ random seeds, with the same evaluation budget and experimental setup as in the main text). Overall, the plug-in and $\widehat{s}_t=1$ variants behave very similarly across both problems, and both substantially outperform the Sobol baseline. In particular, these results support the choice $\widehat{s}_t=1$ as a simple default that does not appear to sacrifice performance in the regimes we tested.

\begin{figure}[tb!]
  \centering
  \includegraphics[width=0.9\textwidth]{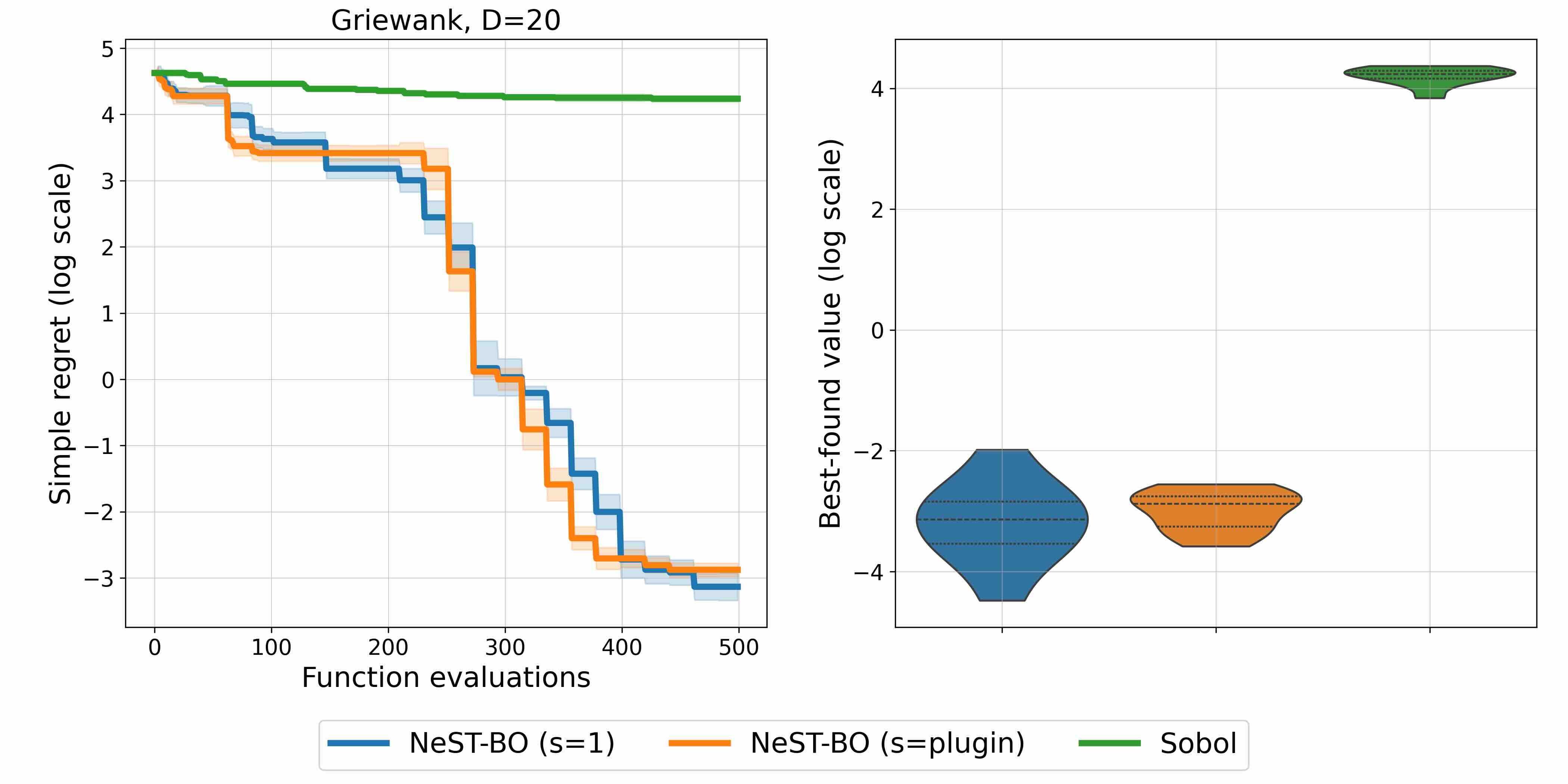}
  \caption{Scale-factor sensitivity on the $20$-D Griewank function.
  \textbf{Left:} Median simple regret (log scale) over $10$ runs, with shaded bands showing variability across runs.
  \textbf{Right:} Distribution (violin plot) of the simple regret achieved by each method at the end of the budget, across the same runs. Lower is better.}
  \label{fig:Griewank_plugin_s=1}
\end{figure}

\begin{figure}[tb!]
  \centering
  \includegraphics[width=0.9\textwidth]{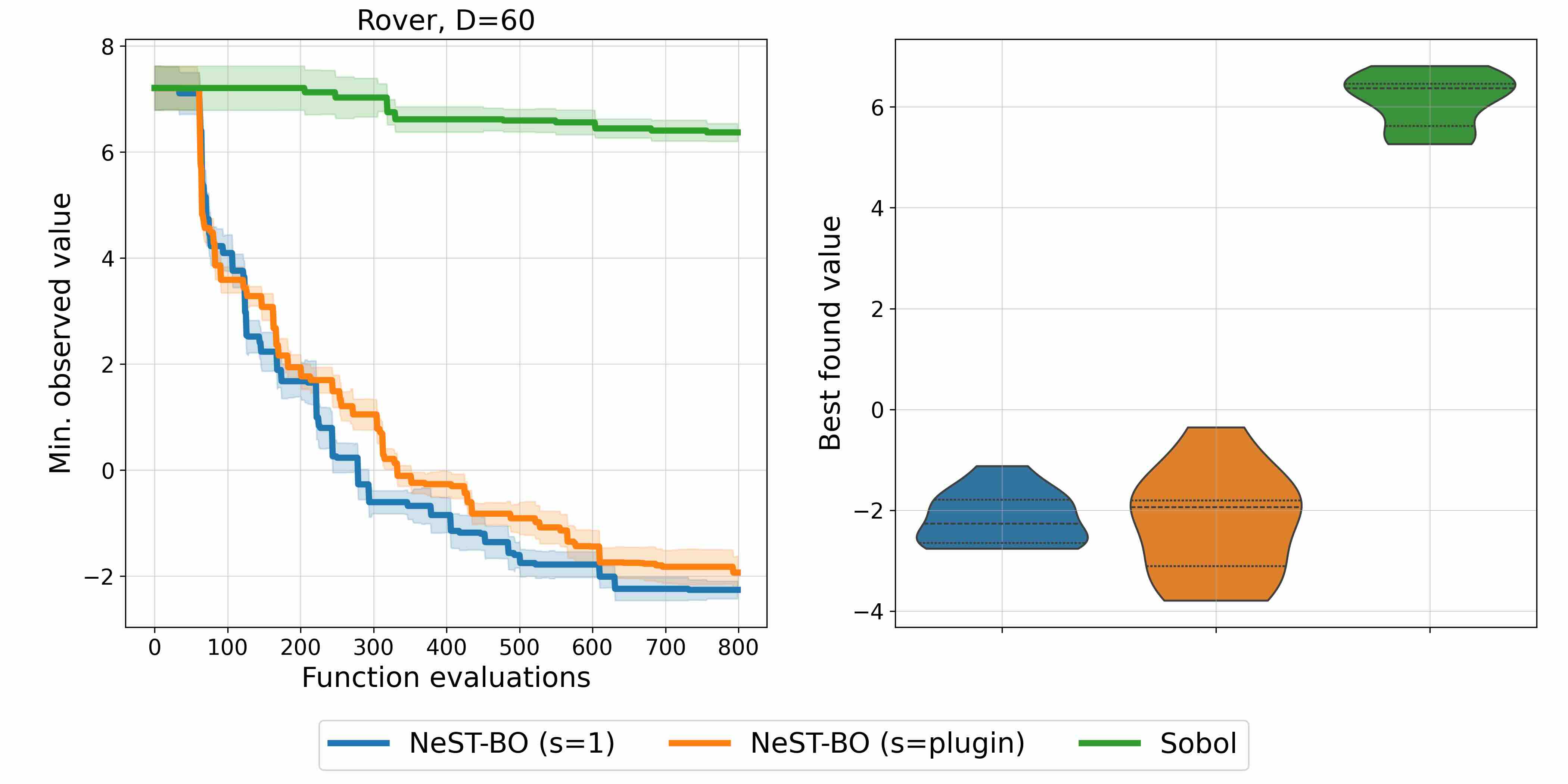}
  \caption{Scale-factor sensitivity on the $60$-D rover trajectory problem.
  \textbf{Left:} Median best-found objective value over $10$ runs, with shaded bands showing variability across runs.
  \textbf{Right:} Distribution (violin plot) of the final best objective value achieved by each method at end of budget, across the same runs. Lower is better.}
  \label{fig:Rover_plugin_s=1}
\end{figure}

\section{PRACTICAL IMPLEMENTATION OF NEST-BO}
\label{app:practical-nestbo}

This appendix describes the NeST-BO variant used in our experiments (Algorithm~\ref{alg:NeSTBO-practical}). The overall structure matches Algorithm~\ref{alg:NeSTBO} in the main body, with two implementation choices that are worth making explicit: (i) we perform greedy (sequential) selection of $M$ points instead of solving a joint $Md$-dimensional batch problem, and (ii) we only take a Newton step when the posterior-mean Hessian at the current iterate is \textit{numerically} positive definite; otherwise we fall back to a length-scale-normalized gradient step following the GIBO procedure \citep[Appendix~A.4]{muller2021local}. 

Furthermore, at the start of each outer iteration, we fit GP hyperparameters to the current dataset via a user-specified routine \textsc{FitGP}. In our experiments this is the standard marginal likelihood estimator (MLE) with restarts; when hyperparameter priors are available, a MAP variant can be used with no change to the algorithm. During the subsequent greedy inner loop, we keep these hyperparameters fixed and update the posterior as new observations arrive.
If we take a Newton step, we use an Armijo backtracking line search on the GP posterior mean $\mu_{\mathcal D}$ (standard defaults; see, e.g., \citep[Chapter~1]{bertsekas2016nonlinear}). If we fall back to a gradient step, we use the length-scale-normalized update and step-size rule from GIBO \citep{muller2021local}; we do not reproduce the full derivation here, but Algorithm~\ref{alg:NeSTBO-practical} makes clear where it enters. Lastly, $\Pi_\mathcal{X}$ denotes the projection operator onto the feasible set $\mathcal{X}$, so that we perform simple projection to ensure feasibility at every iteration (again this could be replaced with more sophisticated feasibility-preserving steps in the future). 

\begin{algorithm}[tb!]
\caption{Practical implementation of NeST-BO}
\label{alg:NeSTBO-practical}
\textbf{Inputs:} initial iterate $\bs{x}_0\in\mathcal{X}\subseteq\mathbb{R}^d$; initial dataset $\mathcal{D}_0$; outer iterations $T$; greedy inner selections per iteration $M$; scale rule $\widehat{s}_t>0$ (we use $\widehat{s}_t = 1$ unless stated otherwise); GP hyperparameter-fitting routine \textsc{FitGP}; Armijo line-search routine \textsc{Armijo}; and GIBO gradient-normalization routine \textsc{GIBOstep}.
\begin{algorithmic}[1]
\For{$t=0,\dots,T-1$}
  \State \textbf{Fit GP:} $(\theta_t,\text{posterior}) \leftarrow \textsc{FitGP}(\mathcal{D}_t)$ \Comment{hyperparameters $\theta_t$ via MLE/MAP}
  \State Set $\mathcal{D}_t^{(0)} \leftarrow \mathcal{D}_t$
  \For{$m=1,\dots,M$} \Comment{greedy (sequential) NeST selection}
    \State $\bs{z}_{t,m} \in \argmin_{\bs{z}\in\mathcal{X}} \ \widehat{\alpha}_{\mathrm{NeST}}\!\left(\bs{z}\mid \bs{x}_t,\mathcal{D}_t^{(m-1)},\widehat{s}_t;\theta_t\right)$
    \State Observe $y_{t,m}=f(\bs{z}_{t,m})+\varepsilon_{t,m}$ and set $\mathcal{D}_t^{(m)} \leftarrow \mathcal{D}_t^{(m-1)}\cup(\bs{z}_{t,m},y_{t,m})$
    \State Update GP posterior conditioned on $\mathcal{D}_t^{(m)}$ (keeping $\theta_t$ fixed)
  \EndFor
  \State Set $\mathcal{D}_{t+1}\leftarrow \mathcal{D}_t^{(M)}$
  \State Compute posterior-mean derivatives at $\bs{x}_t$:
  \State \hspace{1.5em}$\widehat{\bs g}\leftarrow \widehat{\bs g}_{\mathcal{D}_{t+1}}(\bs{x}_t)$,\quad
  $\widehat{\bs H}\leftarrow \widehat{\bs H}_{\mathcal{D}_{t+1}}(\bs{x}_t)$
  \State Symmetrize: $\widehat{\bs H}\leftarrow (\widehat{\bs H}+\widehat{\bs H}^\top)/2$
  \If{\textsc{IsPD}$(\widehat{\bs H};\tau)$} \Comment{e.g., Cholesky succeeds with tolerance $\tau>0$}
    \State \textbf{Newton direction:} solve $\widehat{\bs H}\,\bs d_t = \widehat{\bs g}$ for $\bs d_t$
    \State \textbf{Line search:} $\gamma_t \leftarrow \textsc{Armijo}\!\left(\mu_{\mathcal{D}_{t+1}},\bs{x}_t,-\bs d_t\right)$
    \State $\bs{x}_{t+1} \leftarrow \Pi_{\mathcal{X}}\!\left(\bs{x}_t-\gamma_t\,\bs d_t\right)$
  \Else
    \State \textbf{Gradient fallback:} $(\bs d_t,\gamma_t)\leftarrow \textsc{GIBOstep}\!\left(\widehat{\bs g},\theta_t\right)$
    \State \hspace{1.5em}\Comment{$\theta_t$ provides GP lengthscales used to normalize gradient step \citep{muller2021local}}
    \State $\bs{x}_{t+1} \leftarrow \Pi_{\mathcal{X}}\!\left(\bs{x}_t-\gamma_t\,\bs d_t\right)$
  \EndIf
\EndFor
\end{algorithmic}
\end{algorithm}

\section{THEORETICAL CONVERGENCE PROOFS}
\label{app:theory-conv-proof}

We continue to use the notation summarized in Appendix~\ref{app:gp-expressions} and let Assumptions~1--2 hold from Appendix~\ref{app:proof-newton}. We first provide a complete formal statement and proof of Theorem~\ref{thm:vpc} and then show how NeST-BO inherits local quadratic convergence guarantees of inexact Newton's method. 

\subsection{Complete statement and proof of Theorem~\ref{thm:vpc}}
\label{app:proof-vpc}

This section provides a complete statement and proof of Theorem~\ref{thm:vpc}, showing that the vanishing power-function condition (VPC) holds under NeST sampling in both noiseless and noisy settings. The argument combines the following steps. 
\begin{enumerate}
    \item First, we show that NeST acquisition at $\bs{x}_t$ is upper bounded by an origin-centered error function depending only on the candidate design via monotonicity under conditioning and stationarity (Lemma~\ref{lem:reduce-origin}). This reduces the problem to constructing designs that make this error arbitrarily small. 
    \item Second, we establish that derivative evaluation is a bounded linear functional on $\mathcal{H}$ with an explicit Riesz representer (Lemma~\ref{lem:linear-representer}), and that the GP posterior standard deviation for such a functional equals the optimal worst-case error among all linear rules based on function evaluations at the design points (Lemmas~\ref{lem:worst-case-error-linear}--\ref{lem:posterior-variance-linear-functional}).
    \item Third, we construct explicit centered-difference linear rules supported on a symmetric stencil of radius $h$, and use Taylor expansions with remainder plus RKHS derivative bounds to show the associated worst-case functional errors are $O(h^2)$. By the optimal worst-case-error identity, this yields $O(h^4)$ bounds on the posterior variances (power functions) for all gradient and Hessian components (Lemma~\ref{lem:pf-bound-cfd}).
    \item Lastly, we combine the reduction and the $O(h^4)$ bounds to prove the noiseless VPC result. For noisy observations, we use $m$ replicates per stencil location and show that conditioning on all replicates is equivalent to conditioning on the averaged observations with noise variance $\sigma^2/m$, which adds an explicit $O(\sigma^2/m)$ contribution that can be driven to zero by increasing $m$.
\end{enumerate}

We proceed by establishing the lemmas in this order and then formally stating and proving Theorem~\ref{thm:vpc}.

\begin{lemma}
\label{lem:reduce-origin}
By Assumption \ref{asmp:kernel}, the kernel is stationary $k(\bs{x},\bs{x}')=\varphi(\bs{x}-\bs{x}')$. Fix any dataset $\mathcal{D}_t$ and any $\widehat{s}_t>0$.
Define the error function
\begin{align} \label{eq:error-function}
E_{d,k,s,\sigma}(b)
= 
\inf_{\bs{Z}\in \mathcal{X}^b}
\Big(
\pi^{\bs{g}}_{\bs{Z}}(\bs{0}) + s\,\pi^{\bs{H}}_{\bs{Z}}(\bs{0})
\Big),
\end{align}
where the power functions are evaluated at $\bs{0}$ and conditioned only on potentially noisy observations (with noise variance $\sigma^2$) at $\bs{Z}$. 
Then, for any batch size $b_t$,
\[
\inf_{\bs{Z}\in\mathbb{R}^{b_t\times d}}
\Phi_{\mathcal{D}_t\cup\bs{Z}}(\bs{x}_t)
\leq
E_{d,k,\widehat{s}_t,\sigma}(b_t),
\]
where $\Phi_{\mathcal{D}_t\cup \bs{Z}}(\bs{x}_t) = \pi^{\bs{g}}_{\mathcal{D}_t\cup \bs{Z}}(\bs{x}_t) + \widehat{s}_t\, \pi^{\bs{H}}_{\mathcal{D}_t\cup \bs{Z}}(\bs{x}_t)$ is short for the approximate NeST acquisition function \eqref{eq:nest-approx}.
\end{lemma}

\begin{proof}
This result builds directly on two observations from \citep{wu2023behavior}. The first is monotonicity under conditioning, i.e., for any GP, conditioning on a \emph{larger} dataset cannot increase posterior variances. Here, since $\mathcal{D}_t\cup \bs{Z}$ contains $\bs{Z}$, we have for every $\bs{x}\in\mathcal{X}$
\[
\pi^{\bs{g}}_{\mathcal{D}_t\cup \bs{Z}}(\bs{x})
\leq
\pi^{\bs{g}}_{\bs{Z}}(\bs{x}),
\qquad
\pi^{\bs{H}}_{\mathcal{D}_t\cup \bs{Z}}(\bs{x})
\leq
\pi^{\bs{H}}_{\bs{Z}}(\bs{x}),
\]
hence $\Phi_{\mathcal{D}_t\cup\bs{Z}}(\bs{x})\leq \pi^{\bs{g}}_{\bs{Z}}(\bs{x})+\widehat{s}_t \pi^{\bs{H}}_{\bs{Z}}(\bs{x})$.

The second is that, due to stationarity of the kernel, we can always translate our data to the origin. Specifically, let $\bs{u}_t = -\bs{x}_t$ and define the translated design $\widetilde{\bs{Z}}=\bs{Z}+\bs{1}\bs{u}_t^\top$ (shift every candidate location by $-\bs{x}_t$ where $\bs{1}$ is the all ones vector). Stationarity implies that the joint law of $\{f(\bs{x}_t),f(\bs{Z})\}$ is identical to that of $\{f(\bs{0}),f(\widetilde{\bs{Z}})\}$ after translation; consequently,
\[
\pi^{\bs{g}}_{\bs{Z}}(\bs{x}_t)=\pi^{\bs{g}}_{\widetilde{\bs{Z}}}(\bs{0}),
\qquad
\pi^{\bs{H}}_{\bs{Z}}(\bs{x}_t)=\pi^{\bs{H}}_{\widetilde{\bs{Z}}}(\bs{0}).
\]
Therefore,
\[
\inf_{\bs{Z}} \Phi_{\mathcal{D}_t\cup\bs{Z}}(\bs{x}_t)
\leq
\inf_{\widetilde{\bs{Z}}}
\big(
\pi^{\bs{g}}_{\widetilde{\bs{Z}}}(\bs{0}) + \widehat{s}_t \pi^{\bs{H}}_{\widetilde{\bs{Z}}}(\bs{0})
\big)
=
E_{d,k,\widehat{s}_t,\sigma}(b_t),
\]
which is the claim and thus concludes the proof.
\end{proof}

\begin{lemma}\label{lem:linear-representer}
Assume $\mathcal{X}\subset \mathbb{R}^d$ is open and $k:\mathcal{X}\times\mathcal{X}\to\mathbb{R}$ is $m$-times continuously differentiable, and let $\alpha$ be a multi-index with $|\alpha|\le m$. Define
$
L_{\alpha,\bs{x}}(f) = \partial^\alpha f(\bs{x})
$
for all $f\in \mathcal{H}_k$ and $\bs{x}\in \mathcal{X}$. Then $L_{\alpha,\bs{x}}$ is a bounded linear functional on $\mathcal{H}$, and its Riesz representer is
\[
r_{\alpha,\bs{x}}=\partial^{0,\alpha} k(\cdot,\bs{x})\in \mathcal{H}_k,
\]
such that for all $f\in \mathcal{H}$,
\[
\partial^\alpha f(\bs{x})=\langle f,\partial^{0,\alpha}k(\cdot,\bs{x})\rangle_{\mathcal{H}}.
\]
Moreover, $\|r_{\alpha,\bs{x}}\|_{\mathcal{H}_k}^2=\partial^{\alpha,\alpha}k(\bs{x},\bs{x})$ and
$
|L_{\alpha,\bs{x}}(f)| \le \|f\|_{\mathcal{H}}\,\|r_{\alpha,\bs{x}}\|_{\mathcal{H}}.
$
\end{lemma}

\begin{proof}
Linearity of $L_{\alpha,\bs{x}}$ is immediate from linearity of partial derivatives.

By the differentiability assumption on $k$, the derivative evaluation functional $A(f)=\partial^\alpha f(\bs{x})$
is bounded on $\mathcal{H}_k$ and its Riesz representer is given by the kernel derivative $f_A=\partial^{0,\alpha}k(\cdot,\bs{x})\in \mathcal{H}$
(see, e.g., Example 4.5 in \citep{kanagawa2025gaussian}). Therefore, for all $f\in \mathcal{H}$,
\[
\partial^\alpha f(\bs{x}) = A(f)=\langle f,f_A\rangle_{\mathcal{H}}
= \langle f,\partial^{0,\alpha}k(\cdot,\bs{x})\rangle_{\mathcal{H}}.
\]
Boundedness follows from the Cauchy-Schwarz inequality:
\[
|\partial^\alpha f(\bs{x})|
=|\langle f,r_{\alpha,\bs{x}}\rangle_{\mathcal{H}}|
\le \|f\|_{\mathcal{H}}\,\|r_{\alpha,\bs{x}}\|_{\mathcal{H}}.
\]
Finally, the same reference gives $\|r_{\alpha,\bs{x}}\|_{\mathcal{H}}^2=\partial^{\alpha,\alpha}k(\bs{x},\bs{x})$.
\end{proof}

\begin{lemma} \label{lem:worst-case-error-linear}
    Fix a candidate design $\bs{Z} = (\bs{z}_1, \ldots, \bs{z}_b) \in \mathbb{R}^{b \times d}$ with distinct points and consider any weights $\bs{w} = (w_1, \ldots, w_b) \in \mathbb{R}^b$. Define the linear estimator:
    \begin{align*}
        \widehat{L}^{\bs{w}}_{\alpha, \bs{x}}(f) = \sum_{j=1}^b w_j f(\bs{z}_j).
    \end{align*}
    Let $e^{\bs{w}}_{\alpha, \bs{x}}(\cdot) = r_{\alpha, \bs{x}}(\cdot) - \sum_{j=1}^b w_j k(\cdot, \bs{z}_j) \in \mathcal{H}$.
    Then, for every $f \in \mathcal{H}$, we have:
    \begin{align*}
        L_{\alpha, \bs{x}}(f) - \widehat{L}^{\bs{w}}_{\alpha, \bs{x}}(f) = \langle f, e^{\bs{w}}_{\alpha, \bs{x}} \rangle_\mathcal{H}, 
    \end{align*}
    and hence the worst-case error over the unit ball is:
    \begin{align*}
        \sup_{ f \in \mathcal{H} : \| f \|_H \leq 1 } \left|  L_{\alpha, \bs{x}}(f) - \widehat{L}^{\bs{w}}_{\alpha, \bs{x}}(f) \right| = \| e^{\bs{w}}_{\alpha, \bs{x}} \|_\mathcal{H}
    \end{align*}
\end{lemma}

\begin{proof}
    We start by using the reproducing property to derive:
    \begin{align*}
        \widehat{L}^{\bs{w}}_{\alpha, \bs{x}}(f) = \sum_{j=1}^b w_j f(\bs{z}_j) = \sum_{j=1}^b w_j \langle f, k(\cdot, \bs{z}_j) \rangle_\mathcal{H} = \left\langle f, \sum_{j=1}^b w_j k(\cdot, \bs{z}_j) \right\rangle.
    \end{align*}
    Also, by Lemma \ref{lem:linear-representer}, we have $L_{\alpha, \bs{x}}(f) = \langle f, r_{\alpha, \bs{x}} \rangle$. Subtracting these two gives:
    \begin{align*}
        L_{\alpha, \bs{x}}(f) - \widehat{L}^{\bs{w}}_{\alpha, \bs{x}}(f) = \left\langle f, r_{\alpha, \bs{x}} -  \sum_{j=1}^b w_j k(\cdot, \bs{z}_j) \right\rangle = \langle f, e^{\bs{w}}_{\alpha, \bs{x}} \rangle.
    \end{align*}
    Taking the supremum over $\| f \|_\mathcal{H} \leq 1$ yields the RKHS norm by duality (the Cauchy-Schwarz inequality gives the $\leq$ and equality is achieved by $f = e / \| e \|$ when $e\neq 0$). 
\end{proof}

\begin{lemma}\label{lem:posterior-variance-linear-functional}
Consider noiseless GP interpolation with prior $f\sim \mathcal{GP}(0,k)$ on $\mathcal{X}\subset\mathbb{R}^d$, and observations $\{f(\bs{z}_j)\}_{j=1}^b$ at a design $\bs{Z}=(\bs{z}_1,\dots,\bs{z}_b)\in\mathbb{R}^{b\times d}$. Let $K = k(\bs{Z},\bs{Z})$ and assume $K$ is invertible (or interpret $K^{-1}$ as the Moore--Penrose pseudoinverse). Let $k_{\bs{Z}}$ denote the posterior covariance kernel under noiseless conditioning:
\[
k_{\bs{Z}}(\bs{x},\bs{x}')=k(\bs{x},\bs{x}')-k(\bs{x},\bs{Z})K^{-1}k(\bs{Z},\bs{x}').
\]
Fix a multi-index $\alpha$ such that the derivative functional $L_{\alpha,\bs{x}}(f)=\partial^\alpha f(\bs{x})$ is bounded on $\mathcal{H}_k$.

\begin{enumerate}
\item The posterior variance of the functional $L_{\alpha,\bs{x}}(f)$ is
\[
\mathrm{Var}\!\left(L_{\alpha,\bs{x}}(f)\mid f(\bs{Z})\right)
=
L^{(1)}_{\alpha,\bs{x}}\,L^{(2)}_{\alpha,\bs{x}}\,k_{\bs{Z}}(\bs{x},\bs{x}),
\]
where $L^{(1)}$ applies $L$ to the first argument of the kernel and $L^{(2)}$ applies it to the second.
\item The posterior standard deviation equals the optimal worst-case error among linear rules based on $\{f(\bs{z}_j)\}_{j=1}^b$:
\[
\sqrt{\mathrm{Var}\!\left(L_{\alpha,\bs{x}}(f) \mid f(\bs{Z})\right)}
=
\inf_{\bs{w}\in\mathbb{R}^b}\;
\sup_{f \in \mathcal{H} : \| f \|_H \leq 1}
\left|\,L_{\alpha,\bs{x}}(f)-\sum_{j=1}^b w_j f(\bs{z}_j)\right|.
\]
\end{enumerate}
\end{lemma}

\begin{proof}
Let $u=L_{\alpha,\bs{x}}(f)$. Since $(f(\bs{Z}),u)$ is jointly Gaussian, we have the standard conditional-variance formula
\[
\mathrm{Var}(u \mid f(\bs{Z}))=\mathrm{Var}(u)- \bs{c}^\top K^{-1}\bs{c},
\]
where $\bs{c} \in \mathbb{R}^b$ has entries $c_j=\mathrm{Cov}(f(\bs{z}_j),u)$.
By linearity of covariance for GPs and differentiability of $k$,
\[
\mathrm{Var}(u)=L^{(1)}_{\alpha,\bs{x}}L^{(2)}_{\alpha,\bs{x}}k(\bs{x},\bs{x}),
\qquad
c_j=L^{(2)}_{\alpha,\bs{x}}k(\bs{z}_j,\bs{x}).
\]
Expanding $L^{(1)}_{\alpha,\bs{x}}L^{(2)}_{\alpha,\bs{x}}k_{\bs{Z}}(\bs{x},\bs{x})$ using the definition of $k_{\bs{Z}}$ yields exactly
$\mathrm{Var}(u)-\bs{c}^\top K^{-1}\bs{c}$, proving point 1.

For point 2, applying Lemma \ref{lem:worst-case-error-linear}, we see that the right-hand side in the expression equals $\inf_{\bs{w}}\|e^{\bs{w}}_{\alpha,\bs{x}}\|_{\mathcal{H}}$.

It remains to identify $\mathrm{Var}(u\mid f(\bs{Z}))$ as $\inf_{\bs{w}}\|e^{\bs{w}}_{\alpha,\bs{x}}\|_{\mathcal{H}}^2$.
For any $\bs{w}$, the random variable $u-\widehat{L}^{\bs{w}}_{\alpha,\bs{x}}(f)$ is Gaussian and has variance
\[
\mathrm{Var}\!\left(u-\widehat{L}^{\bs{w}}_{\alpha,\bs{x}}(f)\right)
=
\left\|r_{\alpha,\bs{x}}-\sum_{j=1}^b w_j k(\cdot,\bs{z}_j)\right\|_{\mathcal{H}}^2
=
\|e^{\bs{w}}_{\alpha,\bs{x}}\|_{\mathcal{H}}^2,
\]
by expanding the squared RKHS norm and using
$\langle k(\cdot,\bs{z}_i),k(\cdot,\bs{z}_j)\rangle_{\mathcal{H}}=k(\bs{z}_i,\bs{z}_j)$,
$\langle r_{\alpha,\bs{x}},k(\cdot,\bs{z}_j)\rangle_{\mathcal{H}}=L_{\alpha,\bs{x}}^{(2)}k(\bs{z}_j,\bs{x})$,
and $\langle r_{\alpha,\bs{x}},r_{\alpha,\bs{x}}\rangle_{\mathcal{H}}=L_{\alpha,\bs{x}}^{(1)}L_{\alpha,\bs{x}}^{(2)}k(\bs{x},\bs{x})$.

Finally, for jointly Gaussian variables, the conditional mean $\mathbb{E}[u \mid f(\bs{Z})]$ is the minimum mean squared error predictor of $u$ measurable with respect to $f(\bs{Z})$, and it is linear in $f(\bs{Z})$. Therefore,
\[
\mathrm{Var}(u\mid f(\bs{Z}))
=
\min_{\bs{w}\in\mathbb{R}^b}\mathrm{Var}\!\left(u-\widehat{L}^{\bs{w}}_{\alpha,\bs{x}}(f)\right)
=
\min_{\bs{w}\in\mathbb{R}^b}\|e^{\bs{w}}_{\alpha,\bs{x}}\|_{\mathcal{H}_k}^2.
\]
Taking square roots yields the expression in point 2, which ends the proof.
\end{proof}

\begin{lemma}
\label{lem:pf-bound-cfd}
Assume $k:\mathcal{X}\times \mathcal{X}\to\mathbb{R}$ satisfies Assumption~\ref{asmp:kernel} and is sufficiently smooth so that, for every multi-index $\alpha$ with $|\alpha|\le 4$, the derivative evaluation functional
$L_{\alpha,\bs{x}}(f)=\partial^\alpha f(\bs{x})$ is bounded on $\mathcal{H}$ with Riesz representer
$r_{\alpha,\bs{x}}=\partial^\alpha_{\bs{x}}k(\cdot,\bs{x})\in\mathcal{H}$ (Lemma~\ref{lem:linear-representer}).
Define the uniform derivative-evaluation constants
\begin{align}
\kappa_{3} & = \max_{1\le i\le d}\ \sup_{\bs{x}\in\mathcal{X}} \big\| \partial^3_{x_i} k(\cdot,\bs{x})\big\|_{\mathcal{H}}, \\
\kappa_{4} & = \max_{|\alpha|=4}\ \sup_{\bs{x}\in\mathcal{X}} \big\| \partial^\alpha_{\bs{x}} k(\cdot,\bs{x})\big\|_{\mathcal{H}}.
\end{align}
Fix $\bs{x}\in\mathcal{X}$ and a radius $h>0$ small enough that all stencil points below lie in $\mathcal{X}$, and let
\begin{align}
\mathcal{Z}_h(\bs{x})
=
\{\bs{x}\}
 \cup
\{\bs{x}\pm h\bs{e}_i : 1\le i\le d\}
\cup
\{ \bs{x}\pm h(\bs{e}_i+\bs{e}_j): 1\le i<j\le d \},
\end{align}
where $\{\bs{e}_i\}_{i=1}^d$ are the standard basis vectors.
Consider noiseless GP regression with prior $f \sim \mathcal{GP}(0,k)$ and observations $f(\mathcal{Z}_h(\bs{x}))$.
Then, there exist finite constants $C_{g,k}$ and $C_{H,k}$ (depending only on $k$ and $d$ through $\kappa_3,\kappa_4$) such that, for all sufficiently small $h$,
\begin{align}
\pi^{\bs{g}}_{\mathcal{Z}_h(\bs{x})}(\bs{x}) \le C_{g,k}\, h^4,
\qquad
\pi^{\bs{H}}_{\mathcal{Z}_h(\bs{x})}(\bs{x}) \le C_{H,k}\, h^4.
\end{align}
\end{lemma}

\begin{proof}
The proof proceeds in four steps: (i) construction of Taylor expansions under centered differences with explicit remainders,
(ii) showing how we can build centered difference rules using only 
 $\mathcal{Z}_h(\bs{x})$,
(iii) bounding the remainder terms via RKHS derivative-evaluation constants,
(iv) taking a worst-case supremum over $\{f\in\mathcal{H}:\|f\|_{\mathcal{H}}\le 1\}$ and invoking the optimal worst-case error identity shown in Lemma \ref{lem:posterior-variance-linear-functional}. 

\paragraph{Step (i): Taylor series remainders under centered differences.} 
Fix any direction $\bs{u}\in\mathbb{R}^d$ and define the 1D restriction $\varphi(t) =f(\bs{x}+t\bs{u})$.
Applying Taylor's theorem, e.g., \citep{firey1960remainder} with remainder to $\varphi$ at $t=0$ gives (for some $\xi_+ \in (0,h)$ and $\xi_- \in (-h,0)$)
the standard centered-difference remainder identities are:
\begin{align}
\frac{\varphi(h)-\varphi(-h)}{2h} - \varphi'(0)
&= \frac{h^2}{12}\Big(\varphi^{(3)}(\xi_+) + \varphi^{(3)}(\xi_-)\Big), \label{eq:cd1-rem}\\
\frac{\varphi(h)-2\varphi(0)+\varphi(-h)}{h^2} - \varphi''(0)
&= \frac{h^2}{24}\Big(\varphi^{(4)}(\xi_+) + \varphi^{(4)}(\xi_-)\Big). \label{eq:cd2-rem}
\end{align}
From \eqref{eq:cd1-rem} and \eqref{eq:cd2-rem}, we obtain the bounds
\begin{align}
\left|\frac{\varphi(h)-\varphi(-h)}{2h} - \varphi'(0)\right|
&\le \frac{h^2}{6}\ \sup_{|t|\le h}\ |\varphi^{(3)}(t)|, \label{eq:cd1-bound}\\
\left|\frac{\varphi(h)-2\varphi(0)+\varphi(-h)}{h^2} - \varphi''(0)\right|
&\le \frac{h^2}{12}\ \sup_{|t|\le h}\ |\varphi^{(4)}(t)|. \label{eq:cd2-bound}
\end{align}

\paragraph{Step (ii): centered difference rules for gradient and Hessian.}
We now specialize $\bs{u}$ and map the one-dimensional quantities back to multivariate partial derivatives.

\emph{(Gradient components).}
For each $i\in\{1,\dots,d\}$, take $\bs{u}=\bs{e}_i$ and note that
$\varphi'(0)=\partial_i f(\bs{x})$ and
\[
\frac{\varphi(h)-\varphi(-h)}{2h}=\frac{f(\bs{x}+h\bs{e}_i)-f(\bs{x}-h\bs{e}_i)}{2h}
= \widehat{\partial_i f}(\bs{x};h),
\]
which is a linear rule supported on $\{\bs{x}\pm h\bs{e}_i\}\subset \mathcal{Z}_h(\bs{x})$.

\emph{(Diagonal Hessian components).}
For each $i$, take $\bs{u}=\bs{e}_i$ again and note that $\varphi''(0)=\partial_{ii}^2 f(\bs{x})$ and
\[
\frac{\varphi(h)-2\varphi(0)+\varphi(-h)}{h^2}
=\frac{f(\bs{x}+h\bs{e}_i)-2f(\bs{x})+f(\bs{x}-h\bs{e}_i)}{h^2}
= \widehat{\partial_{ii}^2 f}(\bs{x};h),
\]
a linear rule supported on $\{\bs{x},\bs{x}\pm h\bs{e}_i\}\subset \mathcal{Z}_h(\bs{x})$.

\emph{(Off-diagonal Hessian components).}
Fix $i<j$ and let $\bs{u}=\bs{e}_i+\bs{e}_j$.
The directional second derivative satisfies
\begin{align}
D_{\bs{u}}^2 f(\bs{x}) = \left.\frac{d^2}{dt^2} f(\bs{x}+t\bs{u})\right|_{t=0}
= \partial_{ii}^2 f(\bs{x}) + 2\partial_{ij}^2 f(\bs{x}) + \partial_{jj}^2 f(\bs{x}),
\end{align}
such that
\begin{align}
\partial_{ij}^2 f(\bs{x})
= \frac{1}{2}\Big(D_{\bs{e}_i+\bs{e}_j}^2 f(\bs{x}) - \partial_{ii}^2 f(\bs{x}) - \partial_{jj}^2 f(\bs{x})\Big).
\label{eq:mixed-identity}
\end{align}
Define the directional centered second-difference estimator
\[
\widehat{D_{\bs{e}_i+\bs{e}_j}^2 f}(\bs{x};h)
=
\frac{f(\bs{x}+h(\bs{e}_i+\bs{e}_j)) - 2 f(\bs{x}) + f(\bs{x}-h(\bs{e}_i+\bs{e}_j))}{h^2},
\]
which uses $\{\bs{x},\bs{x}\pm h(\bs{e}_i+\bs{e}_j)\}\subset\mathcal{Z}_h(\bs{x})$ by construction.
Then define the mixed-partial estimator by plugging centered estimators into \eqref{eq:mixed-identity}:
\begin{align}
\widehat{\partial_{ij}^2 f}(\bs{x};h)
= \frac{1}{2}\Big(\widehat{D_{\bs{e}_i+\bs{e}_j}^2 f}(\bs{x};h)
- \widehat{\partial_{ii}^2 f}(\bs{x};h)
- \widehat{\partial_{jj}^2 f}(\bs{x};h)\Big).
\label{eq:mixed-estimator}
\end{align}
The right-hand side is a linear rule supported on the union of
$\{\bs{x}\pm h(\bs{e}_i+\bs{e}_j)\}$, $\{\bs{x}\pm h\bs{e}_i\}$, $\{\bs{x}\pm h\bs{e}_j\}$, and $\{\bs{x}\}$,
all contained in $\mathcal{Z}_h(\bs{x})$.

\paragraph{Step (iii): bounding Taylor remainders via RKHS derivative constants.}
We now bound the derivatives appearing in the remainder terms in \eqref{eq:cd1-bound}--\eqref{eq:cd2-bound}
by RKHS norms. By Lemma~\ref{lem:linear-representer} and the Cauchy-Schwarz inequality, for any multi-index $\alpha$,
\begin{align} \label{eq:rkhs-deriv-bound}
|\partial^\alpha f(\bs{y})|
= |\langle f,\partial^\alpha_{\bs{y}}k(\cdot,\bs{y})\rangle_{\mathcal{H}}|
\le \|f\|_{\mathcal{H}}\ \|\partial^\alpha_{\bs{y}}k(\cdot,\bs{y})\|_{\mathcal{H}},
\qquad \forall \bs{y}\in\mathcal{X}. 
\end{align}
In particular, if $\|f\|_{\mathcal{H}}\le 1$, then $|\partial^\alpha f(\bs{y})|\le \kappa_4$ for all $|\alpha|=4$
and $|\partial_{x_i}^3 f(\bs{y})|\le \kappa_3$ for each $i$ and all $\bs{y}$.

\emph{Gradient remainder.}
With $\bs{u}=\bs{e}_i$, we have $\varphi^{(3)}(t)=\partial_{iii}^3 f(\bs{x}+t\bs{e}_i)$.
Thus, for $\|f\|_{\mathcal{H}}\le 1$,
\[
\sup_{|t|\le h}|\varphi^{(3)}(t)| \le \kappa_3.
\]
Plugging into \eqref{eq:cd1-bound} yields
\begin{align} \label{eq:grad-wce}
\sup_{\|f\|_{\mathcal{H}}\le 1}\, \left|\partial_i f(\bs{x}) - \widehat{\partial_i f}(\bs{x};h)\right|
\le \frac{h^2}{6}\kappa_3. 
\end{align}

\emph{Diagonal Hessian remainder.}
With $\bs{u}=\bs{e}_i$, we have $\varphi^{(4)}(t)=\partial_{iiii}^4 f(\bs{x}+t\bs{e}_i)$, so for $\|f\|_{\mathcal{H}}\le 1$,
$\sup_{|t|\le h}|\varphi^{(4)}(t)| \le \kappa_4$.
Plugging into \eqref{eq:cd2-bound} yields
\begin{align} \label{eq:hessdiag-wce}
\sup_{\|f\|_{\mathcal{H}}\le 1}\, \left|\partial_{ii}^2 f(\bs{x}) - \widehat{\partial_{ii}^2 f}(\bs{x};h)\right|
\le \frac{h^2}{12}\kappa_4. 
\end{align}

\emph{Mixed Hessian remainder.}
With $\bs{u}=\bs{e}_i+\bs{e}_j$, we have $\varphi^{(4)}(t)=D_{\bs{u}}^4 f(\bs{x}+t\bs{u})$.
Expanding $D_{\bs{u}}^4$ in coordinate derivatives yields
\begin{align}
D_{\bs{e}_i+\bs{e}_j}^4 f
= \partial_{iiii}^4 f + 4\partial_{iiij}^4 f + 6\partial_{iijj}^4 f + 4\partial_{ijjj}^4 f + \partial_{jjjj}^4 f,
\end{align}
so by the triangle inequality and \eqref{eq:rkhs-deriv-bound},
\begin{align}
\sup_{\|f\|_{\mathcal{H}}\le 1}\ \sup_{|t|\le h}\ |D_{\bs{e}_i+\bs{e}_j}^4 f(\bs{x}+t(\bs{e}_i+\bs{e}_j))|
\le (1+4+6+4+1)\kappa_4 = 16\kappa_4. \label{eq:dir4-bound}
\end{align}
Applying \eqref{eq:cd2-bound} with this $\bs{u}$ gives
\begin{align}
\sup_{\|f\|_{\mathcal{H}}\le 1}\, \left|D_{\bs{e}_i+\bs{e}_j}^2 f(\bs{x}) - \widehat{D_{\bs{e}_i+\bs{e}_j}^2 f}(\bs{x};h)\right|
\le \frac{h^2}{12}\cdot 16\kappa_4 = \frac{4}{3}\kappa_4\, h^2. \label{eq:dir2-wce}
\end{align}
Combining \eqref{eq:mixed-identity}--\eqref{eq:mixed-estimator} with \eqref{eq:hessdiag-wce} and \eqref{eq:dir2-wce},
then taking the supremum over $\|f\|_{\mathcal{H}}\le 1$, yields
\begin{align} \label{eq:hessmix-wce}
\sup_{\|f\|_{\mathcal{H}}\le 1}\, \left|\partial_{ij}^2 f(\bs{x}) - \widehat{\partial_{ij}^2 f}(\bs{x};h)\right|
&\le \frac{1}{2}\left(\frac{4}{3}\kappa_4 h^2 + \frac{h^2}{12}\kappa_4 + \frac{h^2}{12}\kappa_4\right)
= \frac{3}{4}\kappa_4\, h^2.
\end{align}
Specifically, this follows from
\begin{align*}
\left|\partial_{ij}^2 f(\bs{x}) - \widehat{\partial_{ij}^2 f}(\bs{x};h)\right| &= \frac{1}{2}\left( \left| D_{\bs{e}_i+\bs{e}_j}^2 f(\bs{x}) - \widehat{D_{\bs{e}_i+\bs{e}_j}^2 f}(\bs{x};h) + \partial_{ii}^2 f(\bs{x}) - \widehat{\partial_{ii}^2 f}(\bs{x};h) + \partial_{jj}^2 f(\bs{x}) - \widehat{\partial_{jj}^2 f}(\bs{x};h) \right| \right), \\
&\leq \frac{1}{2}\left( \left| D_{\bs{e}_i+\bs{e}_j}^2 f(\bs{x}) - \widehat{D_{\bs{e}_i+\bs{e}_j}^2 f}(\bs{x};h) \right| + \left| \partial_{ii}^2 f(\bs{x}) - \widehat{\partial_{ii}^2 f}(\bs{x};h) \right| + \left| \partial_{jj}^2 f(\bs{x}) - \widehat{\partial_{jj}^2 f}(\bs{x};h) \right| \right),
\end{align*}
where the second line is based on the triangle inequality and we can then substitute in the previous bounds after taking the supremum over $\| f \|_\mathcal{H} \leq 1$. 

\paragraph{Step (iv): from worst-case errors to the power functions.}
By Lemma \ref{lem:posterior-variance-linear-functional}, for each bounded linear functional $L$ (here, each $\partial_i$ and $\partial_{ij}^2$), we have:
\begin{align}
\sqrt{\mathrm{Var}\!\left(L(f)\mid f(\mathcal{Z}_h(\bs{x}))\right)}
= \inf_{\bs{w}\in\mathbb{R}^{|\mathcal{Z}_h(\bs{x})|}} \sup_{f \in \mathcal{H} : \| f \|_H \leq 1}\ \left|L(f) - \sum_{j} w_j f(\bs{z}_j)\right|.
\end{align}
Therefore, the infimum is upper bounded by the worst-case error of any specific choice of weights.
Choosing the central-difference weights constructed in Step (ii) and using
\eqref{eq:grad-wce}, \eqref{eq:hessdiag-wce}, and \eqref{eq:hessmix-wce}, we obtain the following component-wise variance bounds:
\begin{align}
\mathrm{Var}\!\left(\partial_i f(\bs{x})\mid f(\mathcal{Z}_h(\bs{x}))\right)
&\le \left(\frac{\kappa_3}{6}h^2\right)^2 = \frac{\kappa_3^2}{36}h^4, \label{eq:var-grad-comp}\\
\mathrm{Var}\!\left(\partial_{ii}^2 f(\bs{x})\mid f(\mathcal{Z}_h(\bs{x}))\right)
&\le \left(\frac{\kappa_4}{12}h^2\right)^2 = \frac{\kappa_4^2}{144}h^4, \label{eq:var-hessdiag-comp}\\
\mathrm{Var}\!\left(\partial_{ij}^2 f(\bs{x})\mid f(\mathcal{Z}_h(\bs{x}))\right)
&\le \left(\frac{3\kappa_4}{4}h^2\right)^2 = \frac{9\kappa_4^2}{16}h^4,\qquad i\neq j. \label{eq:var-hessmix-comp}
\end{align}
Finally, since $\mathrm{tr}(\Sigma)$ is the sum of marginal variances of the corresponding coordinates,
\begin{align}
\pi^{\bs{g}}_{\mathcal{Z}_h(\bs{x})}(\bs{x})
= \sum_{i=1}^d \mathrm{Var}\!\left(\partial_i f(\bs{x})\mid f(\mathcal{Z}_h(\bs{x}))\right)
\le d\cdot \frac{\kappa_3^2}{36}h^4
= \frac{d\kappa_3^2}{36}h^4,
\end{align}
proving the gradient bound with $C_{g,k}=d\kappa_3^2/36$.
Similarly,
\begin{align}
\pi^{\bs{H}}_{\mathcal{Z}_h(\bs{x})}(\bs{x})
&= \sum_{i=1}^d \mathrm{Var}\!\left(\partial_{ii}^2 f(\bs{x})\mid f(\mathcal{Z}_h(\bs{x}))\right)
 + \sum_{i\neq j} \mathrm{Var}\!\left(\partial_{ij}^2 f(\bs{x})\mid f(\mathcal{Z}_h(\bs{x}))\right) \\
&\le d\cdot \frac{\kappa_4^2}{144}h^4 + d(d-1)\cdot \frac{9\kappa_4^2}{16}h^4
= \left(\frac{d\kappa_4^2}{144} + \frac{9d(d-1)\kappa_4^2}{16}\right)h^4,
\end{align}
so $C_{H,k}$ can be chosen as stated above, which completes the proof.
\end{proof}

\setcounter{theorem}{1}
\begin{theorem}[VPC under NeST sampling; formal]
Assume the following:
\begin{enumerate}
\item (\emph{Domain}) $\mathcal{X}\subset \mathbb{R}^d$ and for each NeST iteration $t$, there exists $h_0>0$ such that the stencil proposed in Lemma \ref{lem:pf-bound-cfd}, denoted by $\bs{Z}_{h}(\bs{x}_t) \subset \mathcal{X}$, for all $h\in(0,h_0]$.
\item (\emph{Kernel regularity}) The kernel $k$ satisfies Assumption~\ref{asmp:kernel} strong enough for Lemma~\ref{lem:linear-representer} to hold for all $|\alpha|\le 4$, and for all stencil radii $h\in(0,h_0]$ the corresponding kernel matrices are invertible (e.g., $k$ is strictly positive definite and the stencil points are distinct).
\end{enumerate}
Fix any sequence $\{\widehat{s}_t\}_{t\ge 0}$ with $\widehat{s}_t>0$.

\noindent\emph{(Noiseless case, $\sigma^2=0$).}
Let $b^\star=d^2+d+1$. For any iteration $t$ and any batch size $b_t\ge b^\star$,
\begin{equation}\label{eq:thm-vpc-sandwich}
0 \le
\inf_{\bs{Z}\in\mathcal{X}^{b_t}}
\Big(\pi^{\bs{g}}_{\mathcal{D}_t\cup \bs{Z}}(\bs{x}_t)+\widehat{s}_t\,\pi^{\bs{H}}_{\mathcal{D}_t\cup \bs{Z}}(\bs{x}_t)\Big)
\le
E_{d,k,\widehat{s}_t,0}(b_t)
=0 \quad \text{(as an infimum)}.
\end{equation}
In particular, for the constructive stencil $\bs{Z}=\bs{Z}_{h}(\bs{x}_t)$ with any $h\in(0,h_0]$, we have
\begin{equation}\label{eq:thm-vpc-constructive}
\pi^{\bs{g}}_{\mathcal{D}_t\cup \bs{Z}_{h}(\bs{x}_t)}(\bs{x}_t)
+\widehat{s}_t\,\pi^{\bs{H}}_{\mathcal{D}_t\cup \bs{Z}_{h}(\bs{x}_t)}(\bs{x}_t)
\le
(1+\widehat{s}_t)\,C_k\,h^4,
\end{equation}
for some constant $C_k\in(0,\infty)$ depending only on $k$ and $d$. Hence the NeST objective can be made arbitrarily small by shrinking $h\to 0$.

\noindent\emph{(Noisy case with replication, $\sigma^2>0$).}
For any $\epsilon>0$ and any iteration $t$, there exist $h\in(0,h_0]$ and an integer $m\ge 1$ such that, using the replicated stencil with $m$ replicates per stencil location (thus batch size $b_t=m b^\star$),
\begin{equation}
\pi^{\bs{g}}_{\mathcal{D}_t\cup \bs{Z}_{h}^{(m)}(\bs{x}_t)}(\bs{x}_t)
+\widehat{s}_t\,\pi^{\bs{H}}_{\mathcal{D}_t\cup \bs{Z}_{h}^{(m)}(\bs{x}_t)}(\bs{x}_t)
\le\epsilon.    
\end{equation}
Consequently, VPC holds in both noiseless and noisy settings.
\end{theorem}

\begin{proof} We break the proof into two parts; we first prove the claimed result for noiseless function evaluations and then use this result to prove the one claimed for noisy evaluations. 

\textbf{Noiseless case.} Fix iteration $t$ and batch size $b_t\ge b^\star$.
By definition (see Lemma~\ref{lem:reduce-origin}), $E_{d,k,\widehat{s}_t,0}(b_t)$ is an infimum of a sum of posterior variances, hence $E_{d,k,\widehat{s}_t,0}(b_t)\ge 0$.
To show the infimum equals $0$, fix any $h\in(0,h_0]$ and consider the (origin-centered) central-difference stencil
$\bs{Z}_h(\bs{0})\subset \mathcal{X}$ from Lemma~\ref{lem:pf-bound-cfd}, which contains exactly $b^\star$ distinct points.
If $b_t=b^\star$, set $\bs{Z}^{(b_t)}=\bs{Z}_h(\bs{0})$.
If $b_t>b^\star$, augment $\bs{Z}_h(\bs{0})$ with any additional $(b_t-b^\star)$ distinct points in $\mathcal{X}$ and denote the resulting $b_t$-point design by $\bs{Z}^{(b_t)}$.
By monotonicity under conditioning of posterior variances (conditioning on more points cannot increase variance),
\[
\pi^{\bs{g}}_{\bs{Z}^{(b_t)}}(\bs{0}) \le \pi^{\bs{g}}_{\bs{Z}_h(\bs{0})}(\bs{0}),
\qquad
\pi^{\bs{H}}_{\bs{Z}^{(b_t)}}(\bs{0}) \le \pi^{\bs{H}}_{\bs{Z}_h(\bs{0})}(\bs{0}).
\]
Therefore, using that error function $E$ is an infimum over all $b_t$-point designs,
\[
E_{d,k,\widehat{s}_t,0}(b_t)
\le
\pi^{\bs{g}}_{\bs{Z}^{(b_t)}}(\bs{0}) + \widehat{s}_t\,\pi^{\bs{H}}_{\bs{Z}^{(b_t)}}(\bs{0})
\le
\pi^{\bs{g}}_{\bs{Z}_h(\bs{0})}(\bs{0}) + \widehat{s}_t\,\pi^{\bs{H}}_{\bs{Z}_h(\bs{0})}(\bs{0}).
\]
Applying Lemma~\ref{lem:pf-bound-cfd} at $\bs{x}=\bs{0}$ yields
\[
\pi^{\bs{g}}_{\bs{Z}_h(\bs{0})}(\bs{0}) \le C_{g,k} h^4,
\qquad
\pi^{\bs{H}}_{\bs{Z}_h(\bs{0})}(\bs{0}) \le C_{H,k} h^4.
\]
Hence, $E_{d,k,\widehat{s}_t,0}(b_t) \le \big(C_{g,k}+\widehat{s}_t C_{H,k}\big)\,h^4$.
Since $h\in(0,h_0]$ was arbitrary, letting $h\to 0$ implies $E_{d,k,\widehat{s}_t,0}(b_t)=0$ as an infimum.
Non-negativity of the left-hand side in \eqref{eq:thm-vpc-sandwich} is immediate because power functions are posterior variances. The upper bound
\[
\inf_{\bs{Z}\in\mathcal{X}^{b_t}}
\Big(\pi^{\bs{g}}_{\mathcal{D}_t\cup \bs{Z}}(\bs{x}_t)+\widehat{s}_t\,\pi^{\bs{H}}_{\mathcal{D}_t\cup \bs{Z}}(\bs{x}_t)\Big)
\le
E_{d,k,\widehat{s}_t,0}(b_t)
\]
is exactly Lemma~\ref{lem:reduce-origin} specialized to $\sigma^2=0$ and $s=\widehat{s}_t$.
Combining with the infimum result gives \eqref{eq:thm-vpc-sandwich}, and in particular shows the NeST design-objective infimum equals $0$.

To show the next bound in \eqref{eq:thm-vpc-constructive}, fix any $h\in(0,h_0]$ and set $\bs{Z}=\bs{Z}_h(\bs{x}_t)$.
By monotonicity under conditioning (i.e., conditioning on $\mathcal{D}_t$ can only reduce variances), we have
\[
\pi^{\bs{g}}_{\mathcal{D}_t\cup \bs{Z}}(\bs{x}_t)\le \pi^{\bs{g}}_{\bs{Z}}(\bs{x}_t),
\qquad
\pi^{\bs{H}}_{\mathcal{D}_t\cup \bs{Z}}(\bs{x}_t)\le \pi^{\bs{H}}_{\bs{Z}}(\bs{x}_t).
\]
By stationarity of the kernel (Assumption~\ref{asmp:kernel}), translating the entire configuration by $-\bs{x}_t$ and then
applying Lemma~\ref{lem:pf-bound-cfd} at $\bs{x}=\bs{0}$ gives
\[
\pi^{\bs{g}}_{\bs{Z}_h(\bs{x}_t)}(\bs{x}_t)\le C_{g,k}h^4,
\qquad
\pi^{\bs{H}}_{\bs{Z}_h(\bs{x}_t)}(\bs{x}_t)\le C_{H,k}h^4.
\]
Combining these and setting $C_k=\max\{C_{g,k},C_{H,k}\}$ yields \eqref{eq:thm-vpc-constructive}:
\[
\pi^{\bs{g}}_{\mathcal{D}_t\cup \bs{Z}_h(\bs{x}_t)}(\bs{x}_t)
+\widehat{s}_t\,\pi^{\bs{H}}_{\mathcal{D}_t\cup \bs{Z}_h(\bs{x}_t)}(\bs{x}_t)
\le
(1+\widehat{s}_t)\,C_k\,h^4.
\]
This concludes the noiseless part of the theorem.

\textbf{Noisy case with replication.} Let $\bs{Z} = (\bs{z}_1,\dots,\bs{z}_n)$ denote the $n=b^\star$ \emph{distinct} central-difference stencil locations
constructed around $\bs{x}_t$ with stencil radius $h>0$.
At each $\bs{z}_j$, we collect $m \ge 1$ replicate observations
\[
y_{j,\ell} \;=\; f(\bs{z}_j) + \varepsilon_{j,\ell}, 
\qquad \varepsilon_{j,\ell} \overset{\text{i.i.d.}}{\sim} \mathcal{N}(0,\sigma^2),
\qquad j=1,\dots,n,\;\ell=1,\dots,m,
\]
independent across $(j,\ell)$ and independent of the prior $f \sim \mathcal{GP}(0,k)$.
Define the sample mean at each location
\[
\bar{y}_j \;=\; \frac{1}{m}\sum_{\ell=1}^m y_{j,\ell}
\;=\; f(\bs{z}_j) + \bar{\varepsilon}_j,
\qquad
\bar{\varepsilon}_j = \frac{1}{m}\sum_{\ell=1}^m \varepsilon_{j,\ell}.
\]
Then $\bar{\varepsilon}_j \sim \mathcal{N}(0,\sigma^2/m)$ and the $\bar{\varepsilon}_j$ are independent across $j$.

We now justify that conditioning on all replicates $(y_{j,\ell})$ is equivalent (for the posterior of $f$ and of any
linear functional of $f$) to conditioning only on the averaged data $(\bar{y}_j)$.

For each $j$, define the replicate vector $\bs{y}_j = (y_{j,1},\dots,y_{j,m})^\top \in \mathbb{R}^m$ and
$\bs{\varepsilon}_j = (\varepsilon_{j,1},\dots,\varepsilon_{j,m})^\top \in \mathbb{R}^m$,
so that $\bs{y}_j = f(\bs{z}_j)\bs{1}_m + \bs{\varepsilon}_j$, where $\bs{1}_m$ is the $m$-vector of ones.
Let $T \in \mathbb{R}^{m\times m}$ be any orthonormal matrix whose first row is $\bs{1}_m^\top/\sqrt{m}$.
Define the orthogonal transform $\tilde{\bs{y}}_j = T\bs{y}_j$ and $\tilde{\bs{\varepsilon}}_j = T\bs{\varepsilon}_j$.
Since $T$ is orthonormal and $\bs{\varepsilon}_j \sim \mathcal{N}(0,\sigma^2 I_m)$,
we have $\tilde{\bs{\varepsilon}}_j \sim \mathcal{N}(0,\sigma^2 I_m)$.
Moreover
\[
\tilde{y}_{j,1} \;=\; \frac{1}{\sqrt{m}}\bs{1}_m^\top \bs{y}_j
\;=\; \sqrt{m}\, f(\bs{z}_j) + \tilde{\varepsilon}_{j,1},
\qquad
\tilde{y}_{j,r} \;=\; \tilde{\varepsilon}_{j,r}\quad (r=2,\dots,m).
\]
Thus, for each $j$, the $(m-1)$ coordinates $\tilde{y}_{j,2},\dots,\tilde{y}_{j,m}$ are \emph{pure noise}
(independent of $f(\bs{z}_j)$), and hence carry no information about $f$.
Therefore, for the posterior over $f$ (and any functionals of $f$),
conditioning on $\{\bs{y}_j\}_{j=1}^n$ is equivalent to conditioning on
$\{\tilde{y}_{j,1}\}_{j=1}^n$, i.e., on $\{\bar{y}_j\}_{j=1}^n$.
Concretely, we can treat the replicate experiment as the reduced observation model
\[
\bar{\bs{y}} \;=\; \bs{f}_{\bs{Z}} + \bar{\bs{\varepsilon}},
\qquad
\bar{\bs{\varepsilon}} \sim \mathcal{N}\!\left(0,\tfrac{\sigma^2}{m} I_n\right),
\]
where $\bs{f}_{\bs{Z}} = (f(\bs{z}_1),\dots,f(\bs{z}_n))^\top$.

Now, fix any multi-index $\alpha$ and define the derivative functional $L_{\alpha,\bs{x}_t}(f) = \partial^\alpha f(\bs{x}_t)$.
By Lemma \ref{lem:linear-representer}, $L_{\alpha,\bs{x}_t}$ is bounded and linear on $\mathcal{H}$, and
the relevant cross-covariance vector with the (distinct) sites $\bs{Z}$ is
\[
\bs{k}_{\alpha} = \big(\, \partial^\alpha_{\bs{x}} k(\bs{x}_t,\bs{z}_1),\dots,\partial^\alpha_{\bs{x}} k(\bs{x}_t,\bs{z}_n)\,\big)^\top \in \mathbb{R}^n,
\]
and the Gram matrix is $K = k(\bs{Z},\bs{Z}) \in \mathbb{R}^{n\times n}$.
Under the reduced model $\bar{\bs{y}} = \bs{f}_{\bs{Z}} + \bar{\bs{\varepsilon}}$ with
$\bar{\bs{\varepsilon}} \sim \mathcal{N}(0,(\sigma^2/m)I_n)$,
standard Gaussian conditioning gives the posterior variance
\begin{equation}\label{eq:postvar-noisy-deriv}
\mathrm{Var}\!\left( L_{\alpha,\bs{x}_t}(f) \,\middle|\, \bar{\bs{y}}\right)
\;=\;
L_{\alpha,\bs{x}_t}^{(1)} L_{\alpha,\bs{x}_t}^{(2)} k(\bs{x}_t,\bs{x}_t)
\;-\;
\bs{k}_{\alpha}^\top \Big( K + \tfrac{\sigma^2}{m} I_n \Big)^{-1} \bs{k}_{\alpha},
\end{equation}
where $L^{(1)}$ and $L^{(2)}$ indicate that the derivative operator is applied to the first and second argument, respectively.

Define the \emph{noisy} power function for this derivative as
\[
\pi_{\bs{Z},\sigma^2/m}(\alpha,\bs{x}_t)
=
\mathrm{Var}\!\left( L_{\alpha,\bs{x}_t}(f) \,\middle|\, \bar{\bs{y}}\right),
\]
and similarly let $\pi_{\bs{Z},0}(\alpha,\bs{x}_t)$ denote the noiseless version (i.e., with $\sigma^2/m=0$).

Let $\tau^2 = \sigma^2/m$. For fixed distinct $\bs{Z}$, we have $K \succ 0$ (strictly positive definite),
and therefore $(K+\tau^2 I_n)^{-1}$ is well-defined for all $\tau^2 \ge 0$.
Moreover,
\[
0 \le \pi_{\bs{Z},\tau^2}(\alpha,\bs{x}_t) - \pi_{\bs{Z},0}(\alpha,\bs{x}_t)
=
\bs{k}_{\alpha}^\top\!\Big( K^{-1} - (K+\tau^2 I_n)^{-1}\Big)\bs{k}_{\alpha}.
\]
Using the difference of inverse identity
\[
K^{-1} - (K+\tau^2 I_n)^{-1}
=
K^{-1}(\tau^2 I_n)(K+\tau^2 I_n)^{-1},
\]
we obtain the bound
\begin{equation}\label{eq:noisy-gap-bound}
0 \le \pi_{\bs{Z},\tau^2}(\alpha,\bs{x}_t) - \pi_{\bs{Z},0}(\alpha,\bs{x}_t)
\le
\tau^2 \|K^{-1}\|_2\|(K+\tau^2 I_n)^{-1}\|_2 \|\bs{k}_{\alpha}\|_2^2
\le
\tau^2 \|K^{-1}\|_2^2 \|\bs{k}_{\alpha}\|_2^2.
\end{equation}
In particular, for fixed $\bs{Z}$, this implies
$\pi_{\bs{Z},\sigma^2/m}(\alpha,\bs{x}_t) \to \pi_{\bs{Z},0}(\alpha,\bs{x}_t)$ as $m\to\infty$.

Now choose $\bs{Z}=\bs{Z}_h(\bs{x}_t)$; we can combine this with our noiseless result to see that the increase in error for each derivative-component variance is $O(\sigma^2/m)$ for fixed $h$.
Since there are finitely many gradient and Hessian components, we can sum \eqref{eq:noisy-gap-bound} over all
components and absorb the resulting finite sums into a constant $C_k'(h)$ (finite for each fixed $h$), to obtain:
\[
\pi^{\bs{g}}_{\mathcal{D}_t \cup \bs{Z}(h),\sigma^2/m}(\bs{x}_t)
\le
C_{g,k} h^4 + C_{g,k}'(h)\,\frac{\sigma^2}{m},
\qquad
\pi^{\bs{H}}_{\mathcal{D}_t \cup \bs{Z}(h),\sigma^2/m}(\bs{x}_t)
\le
C_{H,k} h^4 + C_{H,k}'(h)\frac{\sigma^2}{m}.
\]
Therefore,
\[
\pi^{\bs{g}}_{\mathcal{D}_t \cup \bs{Z}(h),\sigma^2/m}(\bs{x}_t)
+
\widehat{s}_t\,\pi^{\bs{H}}_{\mathcal{D}_t \cup \bs{Z}(h),\sigma^2/m}(\bs{x}_t)
\le
(1+\widehat{s}_t)\,C_k h^4
+
(1+\widehat{s}_t)\,C_k'(h)\frac{\sigma^2}{m}.
\]
for some $C_k, C_k'(h) \in (0, \infty)$.
Given any $\varepsilon>0$, first choose $h$ small enough so that $(1+\widehat{s}_t)C_k h^4 \le \varepsilon/2$,
and then choose $m$ large enough so that $(1+\widehat{s}_t)C_k'(h)\sigma^2/m \le \varepsilon/2$.
This proves that the NeST objective can be made arbitrarily small in the noisy case, which completes the proof.
\end{proof}

\subsection{Local Quadratic Convergence}
\label{app:local-quadratic-conv}

We now analyze the local convergence behavior of NeST-BO, which follows from relatively standard results for (inexact) Newton's method. The main difference here is that we have a data-driven Newton-step error. 

\begin{theorem}[Local quadratic convergence with NeST]\label{thm:local-quad}
Assume $f$ is twice differentiable, $\bs H$ is $\beta$-Lipschitz on a neighborhood $\mathcal N$ of a local minimizer $\bs x^\star$, and $\lambda_{\min}(\bs H(\bs x))\ge \lambda_{\min}>0$ on $\mathcal N$.
Consider the full-step NeST update $\bs x_{t+1}=\bs x_t-\widehat{\bs d}_{\mathcal{D}_{t+1}}(\bs x_t)$ with Newton-step error $\varepsilon_t=\|\bs d(\bs x_t)-\widehat{\bs d}_{\mathcal{D}_{t+1}}(\bs x_t)\|$.
Then, for any $\bs x_t\in\mathcal N$, we have
\[
\|\bs x_{t+1}-\bs x^\star\|
\le
\frac{\beta}{2\,\lambda_{\min}}\;\|\bs x_t-\bs x^\star\|^2
+ \varepsilon_t,
\qquad
\varepsilon_t^2
\le
2B^2\,\big\|\bs H(\bs x_t)^{-1}\big\|^2\,
\Phi_{\mathcal{D}_{t+1}}(\bs x_t),
\]
where $\Phi_{\mathcal{D}}(\bs x) = \pi^{\bs g}_{\mathcal{D}}(\bs x)+s_{\mathcal{D}}(\bs x)\,\pi^{\bs H}_{\mathcal{D}}(\bs x)$. 
In particular, if $\varepsilon_t \le \kappa \|\bs x_t-\bs x^\star\|^2$ eventually for $\kappa > 0$ (e.g., under VPC), the iterates enter the quadratic regime.
\end{theorem}

\begin{proof}
Write $\widehat{\bs d}_t=\widehat{\bs d}_{\mathcal{D}_{t+1}}(\bs x_t)$ and decompose the updated iterate as follows
\[
\bs x_{t+1}-\bs x^\star
=
\big[\bs x_t-\bs x^\star-\bs H(\bs x_t)^{-1}\bs g(\bs x_t)\big]
+
\big[\bs d(\bs x_t)-\widehat{\bs d}_t\big].
\]
The second bracket is $\varepsilon_t$ by definition.
For the first bracket, inexact-Newton analysis with $\beta$-Lipschitz Hessian gives the following sequence of equalities/inequalities 
\begin{align*}
\|\bs x_t-\bs x^\star-\bs H(\bs x_t)^{-1}\bs g(\bs x_t)\| &= \| \bs{H}(\bs x_t)^{-1} \left( \bs{H}(\bs x_t) (\bs x_t-\bs x^\star) - \bs g(\bs x_t) \right) \| \\
&= \| \bs{H}(\bs x_t)^{-1} \left( \bs{g}(\bs{x}^\star) - \bs g(\bs x_t) - \bs{H}(\bs x_t) (\bs x_t-\bs x^\star) \right) \|, \\
& \leq \| \bs{H}(\bs x_t)^{-1} \| \cdot \| \bs{g}(\bs{x}^\star) - \bs g(\bs x_t) - \bs{H}(\bs x_t) (\bs x_t-\bs x^\star) \|, \\
& \leq \| \bs{H}(\bs x_t)^{-1} \| \cdot \frac{\beta}{2} \| \bs{x} - \bs{x}^\star \|^2, \\
&\le \frac{\beta}{2\lambda_{\min}} \|\bs x_t-\bs x^\star\|^2,    
\end{align*}
where the first line follows from simple rearrangement, the second line follows from $\bs{g}(\bs{x}^\star)=\bs{0}$ since $\bs{x}^\star$ is a local minimizer, the third line follows from standard norm inequalities, the fourth line follows from $\beta$-Lipschitz condition on the Hessian \citep[Lemma 1]{nesterov2006cubic}, and the final line follows from $\| \bs{H}(\bs{x}_t)^{-1} \| \leq 1/\lambda_{\min}$. 
The bound on $\varepsilon_t$ follows from Theorem~\ref{thm:newton-bound}; the stated result follows from these two bounds. 
\end{proof}

\section{EXPERIMENT DETAILS}
\label{app:experiment-details}

\subsection{Implementation}
\label{app:exp-implementation}

\paragraph{Reproducibility.} The code to reproduce the main components of the numerical experiments is avialable on GitHub: \url{https://github.com/PaulsonLab/NeST-BO}. 

\paragraph{Software packages and shared settings.} Unless otherwise stated, all BO baselines are implemented using the \texttt{BoTorch} (version 0.15)\footnote{BoTorch: \url{https://botorch.org/}} \citep{balandat2020botorch} and \texttt{GPyTorch}\footnote{GPyTorch: \url{https://gpytorch.ai/}} (version 1.14) \citep{gardner2018gpytorch} packages. We use squared exponential (SE) kernels with automatic relevance determination (ARD) throughout for a controlled comparison across methods. Acquisition optimization in \texttt{BoTorch} is performed with the \texttt{optimize\_acqf} function using $\texttt{num\_restarts} = 5$ and $\texttt{raw\_samples} = 20$; for the very high-dimensional problems with $d \geq 1000$ (e.g., Ant, Leukemia), we add a 2 second timeout and reduce $\texttt{num\_restarts}$ to $3$ (only on the Leukemia problem) to limit wall-clock cost. Hyperparameters are refit at different frequencies by task class: every move for directional local methods, every $d$ iterations on our $20$d synthetic tasks and the Lunar Lander, Swimmer, Robot Pusher, and Rover Trajectory benchmarks; every iteration for $1000$d synthetic tasks; and every $10$ iterations for Ant and Leukemia (see Appendix \ref{app:synthetic-details}--\ref{app:real-world-details} for task definitions).

\paragraph{Initialization and starting location.} Initial designs use Sobol sequences over the full domain and always include the starting point (for local BO methods). Following prior local BO work, we start directional methods from the domain center on the real-world problems and from a random point on synthetic tasks, since the center can coincide with the global solution on some synthetic functions. Note that Sobol sampling uses the standard \texttt{torch.quasirandom.SobolEngine}. 

\paragraph{Local optimization of NeST.} Because the NeST acquisition \eqref{eq:nest-approx} targets the Newton step at $\bs{x}_t$ and our kernels are stationary, informative experimental designs concentrate fairly close to the current iterate. We therefore optimize $\widehat{\alpha}_{\text{NeST}}$ within a small box centered at $\bs{x}_t$ with radius $\delta_t$, i.e., search domain $[\bs{x}_t-\delta_t,\bs{x}_t+\delta_t]$. In principle, $\delta_t$ can be adapted using standard model-agreement tests from trust-region methods; in all experiments we use a fixed radius for simplicity and speed. We set $\delta = 0.2$ in most tasks and $\delta = 0.01$ on Ant due to strong non-stationarity.

\paragraph{Batch sizes for directional methods.}
For NeST-BO, GIBO, MPD, and MinUCB we query $b_t = d$ points per iteration to learn the local step/direction (or $b_t = m$ in subspace dimension $m$ for the subspace variant). This choice is supported by the ablations in Appendix \ref{app:ablation-batch}.

\paragraph{Method-specific details.} Below, we summarize specific implementation details for each method tested in our comparisons throughout this work:
\begin{itemize}
    \item \textit{NeST-BO:} We implement the practical version of NeST-BO (Algorithm \ref{alg:NeSTBO-practical}) by extending the public GIBO codebase to reuse its BO loop, GP wrappers, and acquisition optimizer, replacing GIBO's GI acquisition with our weighted power function objective in \eqref{eq:nest-approx} and adding the Newton step update with line search. The acquisition is optimized in the local box as described above; backtracking line search is applied on the GP mean.
    \item \textit{NeST-BO (subspace variant):} To study compatibility with learned embeddings, we integrate NeST-BO with the BAxUS \citep{papenmeier2022increasing} subspace machinery from the implementation provided at \url{https://botorch.org/docs/tutorials/baxus/}. As opposed to keeping it fixed in all cases, we treat the initial subspace dimension as a tunable hyperparameter. We use 4 for most problems but increased it some for the real-world problems based on some preliminary experimentation. We adopt a similar subspace expansion heuristic of the original BAxUS implementation: if no improvement is found after some number of consecutive iterations, the subspace dimension is expanded. Note that the original implementation expands the subspace dimension once the trust region side length is smaller than a pre-specified threshold value. We set the number of consective iterations to 50 for all synthetic problems, 10 for the Lunar Lander, Swimmer, and Robot Pusher problems (as they are lower dimensional), and 20 for the Rover Trajectory, Ant, and the Leukemia problems. 
    \item \textit{D-scaled LogEI:} We follow the ``vanilla BO works'' recommendation from \citep{hvarfner2024vanilla} to use LogEI \citep{ament2023unexpected} with dimension-aware length-scale priors and standardized outputs; in BoTorch this corresponds to \texttt{LogExpectedImprovement} on a \texttt{SingleTaskGP} with appropriate priors.
    \item \textit{TuRBO:} We use the (single-trust-region) TuRBO \citep{eriksson2019scalable} implementation from the BoTorch tutorial at \url{https://botorch.org/docs/tutorials/turbo_1/} with LogEI for consistency with the global baseline. TuRBO adaptively shrinks/expands a local box based on success/failure counters, providing strong anytime performance in higher dimensions. 
    \item \textit{GIBO:} The GIBO method \citep{muller2021local} selects samples that maximally reduce the posterior gradient covariance (GI acquisition), then takes a length-scale-normalized gradient step. We use the original public implementation available at \url{https://github.com/sarmueller/gibo} and its early-stopping rule for gradient learning to avoid oversampling near the iterate.
    \item \textit{MPD:} Maximum Probability of Descent (MPD) \citep{nguyen2022local} chooses directions maximizing the posterior probability that the (normalized) gradient is a descent direction; we use the reference implementation available at \url{https://github.com/kayween/local-bo-mpd} with step size $\delta = 0.01$ and probability threshold $p^\star = 0.65$.
    \item \textit{MinUCB:} The MinUCB method \citep{fan2024minimizing} minimizes a UCB-style surrogate of gradient magnitude along candidate directions; we use the reference implementation available at \url{https://github.com/chinafzy1/Minimizing-UCB} and default settings from the original paper. 
    \item \textit{BAxUS:} For the standalone BAxUS baseline, we use the BoTorch tutorial code based on \citep{papenmeier2022increasing} (same as subspace variant of NeST-BO) with LogEI to match our other baselines; BAxUS initializes a small subspace and enlarges it on stagnation. The tutorial uses a \texttt{SingleTaskGP} with log-normal lengthscale priors. We follow the original heuristic to set the subspace dimension, but increase it to 4 in the synthetic problems to avoid initial projections getting an unfair advantage of landing near the global optimum at $\bs{0}$.
    \item \textit{Sobol:} Non-adaptive Sobol sampling uses \texttt{torch.quasirandom.SobolEngine} with scrambling enabled, which is a standard baseline method considered in the BO literature. 
\end{itemize}

\subsection{Computing Resources}
\label{app:computing-resource}

All experiments were executed on the Ohio Supercomputing Center (OSC) cluster (\url{https://www.osc.edu}) using CPU nodes equipped with Intel Xeon CPU Max 9470 processors and 512 GB RAM.

\subsection{Definition of Synthetic Functions}
\label{app:synthetic-details}

We use four standard test function that jointly probe conditioning, non-convexity, and multi-modality properties that stress local learning of curvature.

\paragraph{Sphere.} The $d$-dimensional Sphere function is a convex quadratic function expressed as:
\begin{align*}
    f(\bs{x}) = \textstyle \sum_{i=1}^d x_i^2,
\end{align*}
which was used as a check for local methods, as it has well-behaved curvature and no local minima. We optimize over $\mathcal{X} = [-d^2,d^2]^d$. The global minimum is $f(\bs{x}^\star) = 0$ at $\bs{x}^\star = (0, 0, \ldots, 0)$. 

\paragraph{Rosenbrock.} The $d$-dimensional Rosenbrock function is a bowl-shaped function with a narrow, curved valley, which can be expressed as:
\begin{align*}
    f(\bs{x}) = \textstyle\sum_{i=1}^{d-1}\!\left[100(x_{i+1}-x_i^2)^2+(x_i-1)^2\right].
\end{align*}
It is a classical ``ill-conditioned'' test function that is likely to reward updates that incorporate curvature information. We optimize within the bounds $\mathcal{X} = [-5,5]^d$. The global minimum is $f(\bs{x}^\star) = 0$ at $\bs{x}^\star = (1, 1, \ldots, 1)$. This is a common benchmark in the BO literature; see, e.g., \citep{xu2024standard} for example. 

\paragraph{Griewank.} The $d$-dimensional Griewank function is a separable quadratic modulated by a product of cosines:
\begin{align*}
    f(\bs{x})=\sum_{i=1}^d \frac{x_i^2}{4000}-\prod_{i=1}^d \cos\!\left(\frac{x_i}{\sqrt{i}}\right)+1, 
\end{align*}
which features many regularly spaced local minima. We optimize within the bounds $[-300,300]^d$. The global minimum is $f(\bs{x}^\star) = 0$ at $\bs{x}^\star = (0, 0, \ldots, 0)$. This has been recently used as a benchmark problem when analyzing high-dimensional BO algorithms; see, e.g., \citep{papenmeier2025understanding}. 

\paragraph{Ackley.} The $d$-dimensional Ackley function is a highly multi-modal landscape with a flat outer region and a steep basin near the global optimum: 
\begin{align*}
    f(\bs x) = -20\text{exp}\left(-0.2\sqrt{\frac{1}{d}\sum_{i=1}^dx_i^2}\right) - \text{exp}\left(\frac{1}{d}\sum_{i=1}^d\text{cos}(2\pi x_i)\right) + 20 + \text{exp}(1).
\end{align*}
We optimize within the bounds $[-5,5]^d$. The global minimum is $f(\bs{x}^\star) = 0$ at $\bs{x}^\star = (0, 0, \ldots, 0)$. The Ackley function is another popular benchmark in both low- and high-dimensional BO works; see, e.g., \citep{siemenn2023fast, ament2023unexpected}. 

We highlight that these choices follow common practice in the BO literature and are meant to cover complementary problem aspects: Rosenbrock isolates ill-conditioning (benefiting Newton steps), Griewank/Ackley add dense local structure (testing how fast local surrogates learn gradients \textit{and} curvature), while Sphere confirms that added second-order machinery can still add value on easy, well-conditioned cases. 

\subsection{Real-World Benchmark Problems}
\label{app:real-world-details}

We include six problems spanning reinforcement learning (RL) control, robotic planning, and large-scale hyperparameter tuning. Together they cover medium to very high dimensionality, varying degrees of non-stationarity, and different noise profiles -- settings where local curvature can accelerate progress and subspaces can be useful.

\paragraph{Lunar Lander (12d).} A classic control task in the OpenAI Gymnasium (\texttt{LunarLander-v3}), where a controller with 12 parameters maps the measured state to four discrete actions. Following prior BO studies from, e.g., \citep{eriksson2019scalable}, we minimize the negative episodic return (reward sign flipped). Episodes are run for 1000 steps; we initialize the GP from 10 Sobol points.

\paragraph{Swimmer (16d).} This is a MuJoCo locomotion task in the OpenAI Gymnasium (\texttt{Swimmer-v5}) with a linear policy (16 parameters). This probelm has been considered in prior BO studies, e.g., \citep{muller2021local}; we minimize negative reward and run episodes for 1000 steps. We again initialize the GP from 10 Sobol points.

\paragraph{Robot Pushing (14d).} This is a planar manipulation benchmark where 14 controller parameters must be tuned to reduce distances to targets. We adopt bounds and setup from prior BO work, reported in \citep{eriksson2019scalable}, and run it in a small-noise regime to isolate optimization behavior.

\paragraph{Rover Trajectory (60d).} A trajectory-planning task -- originally introduced in \citep{wang2018batched} -- in which 30 two-dimensional waypoints are optimized to maximize a reward that penalizes rough terrain and constraint violations. The resulting 60-dimensional design is structured and non-stationary, which stresses local surrogates and benefits from curvature information. We minimize negative reward and follow the large-domain setting (using 200 Sobol initial points) suggested in the literature. 

\paragraph{Ant (888d).} This is a MuJoCo quadruped with an 8-dimensional action space and 111-dimensional observations; we optimize a linear state-feedback policy (888 parameters) and minimize negative reward. This benchmark has recently been used to probe high-dimensional BO with subspaces \citep{hvarfner2024vanilla}. In contrast to previous work that neglects contact forces and uses the \texttt{Ant-v2} environment, we use \texttt{Ant-v4} (in OpenAI Gymnasium) with contact forces enabled, which increases complexity. For NeST-BO-sub and BAxUS, we initialize at the center point of the subspace, which yields an initial objective (negative reward) of approximately -990.

\paragraph{Leukemia (7129d).} A weighted Lasso regression task with one weight per feature (7129 hyperparameters) on the Leukemia dataset from LassoBench \citep{vsehic2022lassobench}. We follow the standard least-squares objective with weighted $\ell_1$ regularization and evaluate test error under the LassoBench protocol. This problem exemplifies extremely high-dimensional, sparse settings where subspace methods are essential.

The RL tasks (Lunar Lander, Swimmer, Ant) expose NeST-BO to non-stationary, stochastic objectives where local Newton steps and line search stabilize progress; Robot Pushing and Rover emphasize structured geometry and curvature; Leukemia provides a sparse, ultra-high-dimensional regime. This mix lets us isolate when curvature helps (ill-conditioned valleys, non-stationary responses) and when subspaces are essential, and it explains the large empirical gains we report over purely gradient-based local BO and global BO baselines. 

\subsection{Final Performance Distributions and Summary Statistics}

To complement the performance-versus-iteration plots in Figure~\ref{fig:main results} in the main text, Figure~\ref{fig:violin-final} summarizes, for each benchmark, the empirical distribution of the \textit{final} best-found objective across replicates and methods. Each panel corresponds to one task. Within a panel, one violin plot per method shows the distribution of final outcomes; interior dashed lines denote empirical quartiles (median in the middle). All y-axes are in the native objective scale (negative is better for all tasks).
To provide a compact numerical summary consistent with the main text, Tables~\ref{tab:final-synth20}--\ref{tab:final-ant-leukemia} report the median with interquartile range (IQR) (difference between upper and lower quartile) in parentheses for the same final outcomes (10 replicates per method). 

\begin{figure}[!tb]
  \centering
  \includegraphics[width=\textwidth]{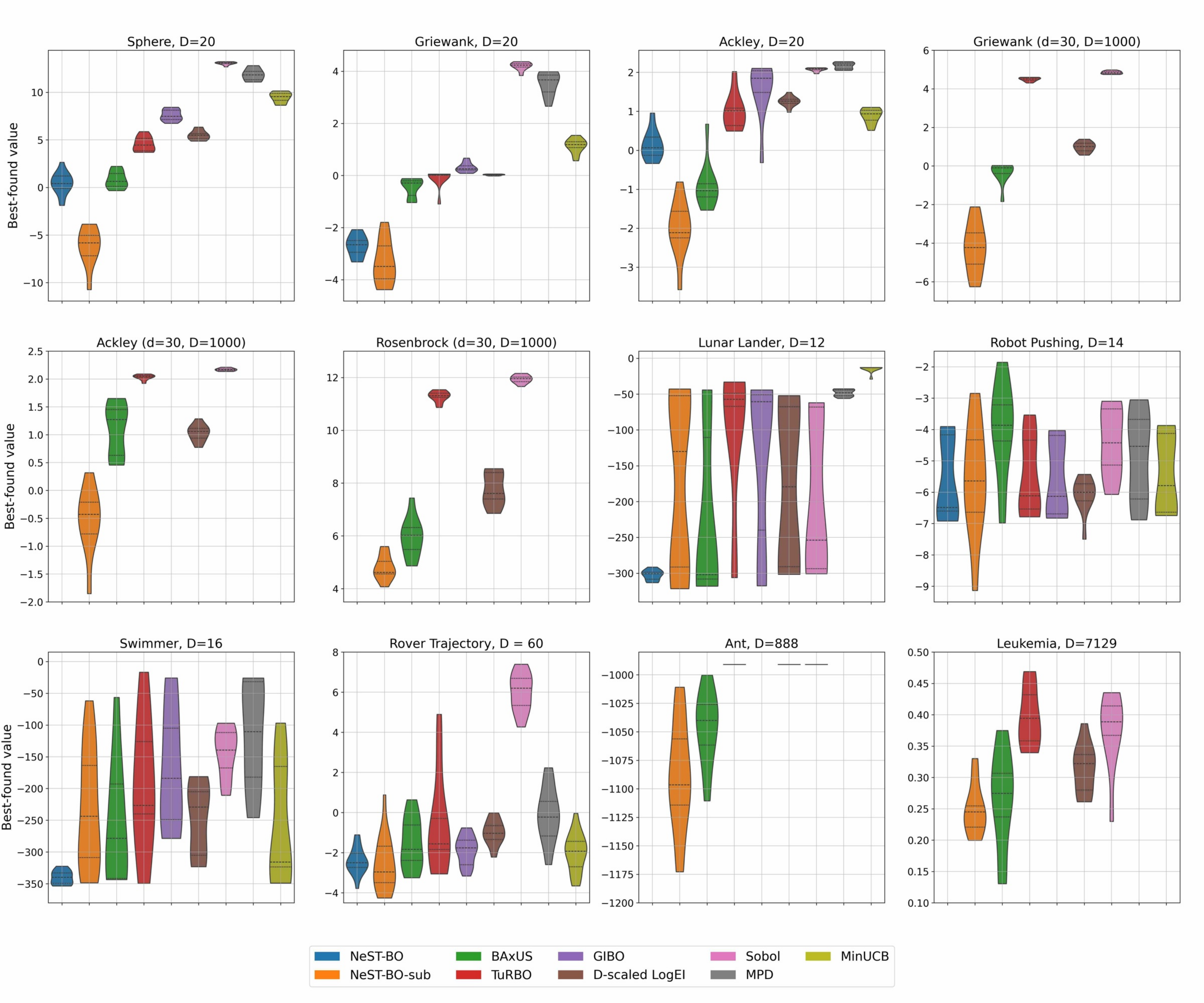}
  \caption{Final best-found values across tasks. For each test problem, violins show the distribution of the final best-found objective over repeated runs for all methods. Dashed lines mark quartiles (median centered). Lower is better in every panel. The plot provides a compact view of both central tendency and spread at termination, complementing the iteration-wise trajectories in the main text.}
  \label{fig:violin-final}
\end{figure}

\begin{table}[!tb]
\caption{Final best-found objective on 20D synthetic benchmarks. Entries are median (IQR) across 10 replicates; lower is better. Best median per column is in \textbf{bold}.}
\label{tab:final-synth20}
\centering
\small
\setlength{\tabcolsep}{6pt}
\renewcommand{\arraystretch}{1.05}
\begin{tabular}{lccc}
\toprule
Method & Griewank ($d{=}20$) & Sphere ($d{=}20$) & Ackley ($d{=}20$) \\
\midrule
NeST-BO          & 0.0704 (0.0298) & 1.37 (1.90)              & 1.07 (0.53) \\
NeST-BO-sub      & \textbf{0.0308 (0.0518)} & \textbf{0.0032 (0.0059)} & \textbf{0.12 (0.10)} \\
BAxUS            & 0.75 (0.33)     & 1.92 (3.32)              & 0.36 (0.12) \\
TuRBO            & 1.02 (0.01)     & 86.89 (127.89)           & 2.76 (1.06) \\
GIBO             & 1.30 (0.22)     & 1,791.79 (2,142.81)      & 6.38 (3.28) \\
D-scaled LogEI   & 1.03 (0.04)     & 243.39 (120.28)          & 3.55 (0.36) \\
Sobol            & 69.21 (8.78)    & 485,044.69 (62,421.03)   & 8.08 (0.32) \\
MPD              & 39.53 (22.38)   & 140,223.66 (117,641.69)  & 8.93 (1.28) \\
MinUCB           & 3.28 (0.74)     & 14,273.04 (10,468.39)    & 2.56 (0.63) \\
\bottomrule
\end{tabular}
\end{table}

\begin{table}[!tb]
\caption{Final best-found objective on high-dimensional synthetic benchmarks (ambient $D{=}1000$; intrinsic $d{=}30$). Entries are median (IQR) across 10 replicates; lower is better. Best median per column is in \textbf{bold}.}
\label{tab:final-synth1000}
\centering
\small
\setlength{\tabcolsep}{6pt}
\renewcommand{\arraystretch}{1.05}
\begin{tabular}{lccc}
\toprule
Method & Griewank ($d{=}30; D{=}1000$) & Ackley ($d{=}30; D{=}1000$) & Rosenbrock ($d{=}30; D{=}1000$) \\
\midrule
NeST-BO-sub      & \textbf{0.02 (0.05)}  & \textbf{0.65 (0.35)} & \textbf{48.82 (22.14)} \\
BAxUS            & 0.92 (0.33)           & 3.57 (2.40)          & 278.91 (107.26) \\
TuRBO            & 91.17 (13.20)         & 7.77 (0.28)          & 76,095.20 (19,735.29) \\
D-scaled LogEI   & 2.74 (1.03)           & 2.89 (0.49)          & 1,437.85 (774.60) \\
Sobol            & 126.71 (16.07)        & 8.75 (0.28)          & 138,662.22 (12,104.50) \\
\bottomrule
\end{tabular}
\end{table}

\begin{table}[!tb]
\caption{Final best-found objective on simulator benchmarks. Entries are median (IQR) across 10 replicates; lower is better. Best median per column is in \textbf{bold}.}
\label{tab:final-sim}
\centering
\small
\setlength{\tabcolsep}{6pt}
\renewcommand{\arraystretch}{1.05}
\begin{tabular}{lcccc}
\toprule
Method & Lunar Lander & Robot Pushing & Swimmer & Rover Trajectory \\
\midrule
NeST-BO          & -300.49 (9.88)          & \textbf{-6.49 (2.43)} & \textbf{-339.62 (16.53)} & -2.26 (0.87) \\
NeST-BO-sub      & -290.62 (18.51)         & -5.64 (2.31)          & -239.83 (177.65)         & \textbf{-3.08 (0.89)} \\
BAxUS            & \textbf{-301.89 (197.58)} & -3.86 (1.15)        & -278.26 (148.83)         & -1.83 (1.77) \\
TuRBO            & -57.25 (18.17)          & -6.11 (2.19)          & -226.57 (114.12)         & -1.55 (1.55) \\
GIBO             & -60.47 (189.10)         & -6.13 (2.50)          & -183.91 (144.05)         & -1.91 (0.98) \\
D-scaled LogEI   & -179.18 (223.20)        & -6.00 (0.55)          & -229.15 (100.17)         & -1.12 (0.62) \\
Sobol            & -253.77 (225.42)        & -4.42 (1.79)          & -139.29 (55.79)          & 6.36 (0.84) \\
MPD              & -48.05 (7.87)           & -4.53 (2.54)          & -110.25 (150.48)         & -0.25 (1.34) \\
MinUCB           & -12.82 (0)              & -5.79 (2.51)          & -315.82 (158.33)         & -1.95 (1.42) \\
\bottomrule
\end{tabular}
\end{table}

\begin{table}[!tb]
\caption{Final best-found objective on additional benchmarks. Entries are median (IQR) across 10 replicates; lower is better. Best median per column is in \textbf{bold}.}
\label{tab:final-ant-leukemia}
\centering
\small
\setlength{\tabcolsep}{8pt}
\renewcommand{\arraystretch}{1.05}
\begin{tabular}{lcc}
\toprule
Method & Ant & Leukemia \\
\midrule
NeST-BO-sub      & \textbf{-1099.46 (42.28)}    & \textbf{0.25 (0.03)} \\
BAxUS            & -1044.03 (35.66)             & 0.27 (0.07) \\
TuRBO            & -990.83 (0)                  & 0.39 (0.07) \\
D-scaled LogEI   & -990.83 (0)                  & 0.32 (0.06) \\
Sobol            & -990.83 (0)                  & 0.39 (0.05) \\
\bottomrule
\end{tabular}
\end{table}

To complement these descriptive statistics with a paired significance check, we additionally report one-sided paired Wilcoxon signed-rank tests, e.g., \citep{woolson2007wilcoxon} comparing NeST-BO (and NeST-BO-sub) to each baseline on the per-seed final best objective (Tables~\ref{tab:wilcox-synth20}--\ref{tab:wilcox-sub-hd-extra}). For each task and comparator, the test is applied to the matched differences in terminal outcomes (NeST variant minus baseline) and evaluates whether these paired differences are systematically negative, without assuming normality of the differences. Overall, the Wilcoxon results largely agree with the separation seen in the violin plots (Figure \ref{fig:violin-final}) and the gaps in the median (IQR) tables (Tables~\ref{tab:final-synth20}--\ref{tab:final-ant-leukemia}): most comparisons indicate statistically significant improvements for a NeST variant, with a small number of marginal or non-significant cases where methods are not statistically different at termination (beyond a standard p-value of 0.05). 

\begin{table}[!tb]
\caption{Paired one-sided Wilcoxon signed-rank tests on 20D synthetic benchmarks at termination.
Each entry reports significance for the hypothesis that the NeST variant attains lower final objective than the comparator, with p-value in parentheses (truncated at 0.001).
Symbols: \Wsig{$p<0.05$}, \Wmar{$0.05\le p\le 0.10$}, \Wns{$p>0.10$}.}
\label{tab:wilcox-synth20}
\centering
\small
\setlength{\tabcolsep}{4.5pt}
\renewcommand{\arraystretch}{1.05}
\begin{tabular}{lcccccc}
\toprule
& \multicolumn{2}{c}{Griewank ($d{=}20$)} & \multicolumn{2}{c}{Sphere ($d{=}20$)} & \multicolumn{2}{c}{Ackley ($d{=}20$)} \\
\cmidrule(lr){2-3}\cmidrule(lr){4-5}\cmidrule(lr){6-7}
Comparator & NeST-BO & NeST-BO-sub & NeST-BO & NeST-BO-sub & NeST-BO & NeST-BO-sub \\
\midrule
NeST-BO      & \Wna           & \Wmar{0.08} & \Wna         & \Wsig{0.001} & \Wna         & \Wsig{0.001} \\
NeST-BO-sub  & \Wns{0.93}     & \Wna        & \Wns{1.0}    & \Wna         & \Wns{1.0}    & \Wna \\
BAxUS        & \Wsig{0.001}   & \Wsig{0.001}& \Wns{0.25}   & \Wsig{0.001} & \Wns{1.0}    & \Wsig{0.005} \\
TuRBO        & \Wsig{0.001}   & \Wsig{0.001}& \Wsig{0.001} & \Wsig{0.001} & \Wsig{0.001} & \Wsig{0.001} \\
GIBO         & \Wsig{0.001}   & \Wsig{0.001}& \Wsig{0.001} & \Wsig{0.001} & \Wsig{0.002} & \Wsig{0.001} \\
D-scaled LogEI & \Wsig{0.001} & \Wsig{0.001}& \Wsig{0.001} & \Wsig{0.001} & \Wsig{0.001} & \Wsig{0.001} \\
Sobol        & \Wsig{0.001}   & \Wsig{0.001}& \Wsig{0.001} & \Wsig{0.001} & \Wsig{0.001} & \Wsig{0.001} \\
MPD          & \Wsig{0.001}   & \Wsig{0.001}& \Wsig{0.001} & \Wsig{0.001} & \Wsig{0.001} & \Wsig{0.001} \\
MinUCB       & \Wsig{0.001}   & \Wsig{0.001}& \Wsig{0.001} & \Wsig{0.001} & \Wsig{0.001} & \Wsig{0.001} \\
\bottomrule
\end{tabular}
\end{table}

\begin{table}[!tb]
\caption{Paired one-sided Wilcoxon signed-rank tests on simulator benchmarks at termination.
Each entry tests whether the NeST variant achieves lower final objective than the comparator, with p-value in parentheses (truncated at 0.001).
Symbols: \Wsig{$p<0.05$}, \Wmar{$0.05\le p\le 0.10$}, \Wns{$p>0.10$}.}
\label{tab:wilcox-sim}
\centering
\small
\setlength{\tabcolsep}{3.5pt}
\renewcommand{\arraystretch}{1.05}
\resizebox{\textwidth}{!}{%
\begin{tabular}{lcccccccc}
\toprule
& \multicolumn{2}{c}{Lunar Lander} & \multicolumn{2}{c}{Robot Pushing} & \multicolumn{2}{c}{Swimmer} & \multicolumn{2}{c}{Rover Trajectory} \\
\cmidrule(lr){2-3}\cmidrule(lr){4-5}\cmidrule(lr){6-7}\cmidrule(lr){8-9}
Comparator & NeST-BO & NeST-BO-sub & NeST-BO & NeST-BO-sub & NeST-BO & NeST-BO-sub & NeST-BO & NeST-BO-sub \\
\midrule
NeST-BO      & \Wna        & \Wns{0.98} & \Wna        & \Wns{0.54} & \Wna        & \Wns{1.0}  & \Wna        & \Wns{0.14} \\
NeST-BO-sub  & \Wsig{0.03} & \Wna       & \Wns{0.50}  & \Wna       & \Wsig{0.001}& \Wna      & \Wns{0.88}  & \Wna \\
BAxUS        & \Wns{0.12}  & \Wns{0.62} & \Wsig{0.01} & \Wsig{0.02} & \Wsig{0.01} & \Wns{0.54} & \Wmar{0.065} & \Wmar{0.05} \\
TuRBO        & \Wsig{0.005}& \Wsig{0.01} & \Wns{0.46} & \Wns{0.38} & \Wsig{0.003}& \Wns{0.28} & \Wsig{0.01} & \Wsig{0.002} \\
GIBO         & \Wsig{0.003}& \Wsig{0.02} & \Wns{0.69} & \Wns{0.54} & \Wsig{0.001}& \Wns{0.16} & \Wns{0.35} & \Wns{0.19} \\
D-scaled LogEI & \Wsig{0.002} & \Wmar{0.10} & \Wns{0.88} & \Wns{0.78} & \Wsig{0.001} & \Wns{0.78} & \Wsig{0.005} & \Wsig{0.007} \\
Sobol        & \Wsig{0.005}& \Wns{0.14} & \Wsig{0.001}& \Wsig{0.04} & \Wsig{0.001}& \Wsig{0.01} & \Wsig{0.001}& \Wsig{0.001} \\
MPD          & \Wsig{0.001}& \Wsig{0.001}& \Wsig{0.01} & \Wns{0.14} & \Wsig{0.002}& \Wsig{0.01} & \Wsig{0.001}& \Wsig{0.007} \\
MinUCB       & \Wsig{0.001}& \Wsig{0.001}& \Wns{0.28} & \Wns{0.50} & \Wsig{0.002}& \Wns{0.72} & \Wns{0.62} & \Wns{0.25} \\
\bottomrule
\end{tabular}}
\end{table}

\begin{table}[!tb]
\caption{Paired one-sided Wilcoxon signed-rank tests for NeST-BO-sub at termination on high-dimensional synthetic benchmarks ($D{=}1000$, $d{=}30$) and additional tasks.
Each entry tests whether NeST-BO-sub achieves lower final objective than the comparator, with p-value in parentheses (truncated at 0.001).
Symbols: \Wsig{$p<0.05$}, \Wmar{$0.05\le p\le 0.10$}, \Wns{$p>0.10$}.}
\label{tab:wilcox-sub-hd-extra}
\centering
\small
\setlength{\tabcolsep}{4.5pt}
\renewcommand{\arraystretch}{1.05}
\begin{tabular}{lccccc}
\toprule
& Griewank & Ackley & Rosenbrock & Ant & Leukemia \\
\midrule
BAxUS          & \Wsig{0.001} & \Wsig{0.001} & \Wsig{0.003} & \Wsig{0.002} & \Wns{0.25} \\
TuRBO          & \Wsig{0.001} & \Wsig{0.001} & \Wsig{0.001} & \Wsig{0.001} & \Wsig{0.001} \\
D-scaled LogEI & \Wsig{0.001} & \Wsig{0.001} & \Wsig{0.001} & \Wsig{0.001} & \Wsig{0.005} \\
Sobol          & \Wsig{0.001} & \Wsig{0.001} & \Wsig{0.001} & \Wsig{0.001} & \Wsig{0.001} \\
\bottomrule
\end{tabular}
\end{table}

\section{ADDITIONAL EXPERIMENTS AND ABLATIONS}
\label{app:add-experiments}

\subsection{GP Prior Realizations}

We consider a similar study to that in \citep{muller2021local} wherein we optimize samples drawn from a GP prior. We take a GP prior with zero mean and the squared exponential (SE) kernel with unit variance. To vary difficulty with dimension $d$, we draw the kernel length-scale $\ell(d)$ uniformly over a narrow interval centered at the heuristic used by \citep[Appendix A.5]{muller2021local}, and keep $\ell(d)$ fixed within each realization. 

Rather than fitting a surrogate to finite prior samples, we \textit{directly} sample functions from the prior using random Fourier features (RFF) \citep{rahimi2007random}. Concretely, with $n_b=1024$ features,
\begin{align*}
f(\boldsymbol{x}) \approx \sum_{i=1}^{n_b} w_i \,\phi_i(\boldsymbol{x}), 
\qquad
\phi_i(\boldsymbol{x})=\sqrt{\tfrac{2}{n_b}}\cos(\boldsymbol{\theta}_i^\top \boldsymbol{x} + \tau_i),    
\end{align*}
where $w_i\sim\mathcal{N}(0,1)$, $\tau_i\sim\mathcal{U}(0,2\pi)$, and for the SE kernel we sample $\boldsymbol{\theta}_i\sim\mathcal{N}(\boldsymbol{0},\ell(d)^{-2}\boldsymbol{I})$ by Bochner’s theorem \citep{rahimi2007random}. We treat the resulting $f$ as the ground-truth objective.

To compute simple regret, we approximate the global optimum via L-BFGS \citep{zhu1997algorithm} with 100 multistarts on each GP prior realization. We consider $d\in\{15,20\}$, generate 10 independent realizations per $d$, and report the median across runs. Since the goal here is to stress NeST-BO itself as a high-performance algorithm in the ``medium-dimensional'' regime (10 to 50 dimensions), no subspace mechanisms are used (and BAxUS is omitted to reduce confounding factors). We also include Augmented Random Search (ARS) \citep{mania2018simple} as an additional baseline. 
Figure~\ref{fig:GP-prior} shows that NeST-BO consistently achieves the lowest simple regret across both dimensions. GIBO, D-scaled LogEI, and TuRBO are competitive early on, but lag in later iterations, suggesting a benefit from explicitly targeting curvature in this regime. We also observe tighter best-value distributions for NeST-BO. Note that the absolute regret depends on the RFF approximation and the length-scale draw; all methods use the same realizations to ensure a fair comparison. 

\begin{figure}[!t]
  \centering
  \includegraphics[width=1.0\textwidth]{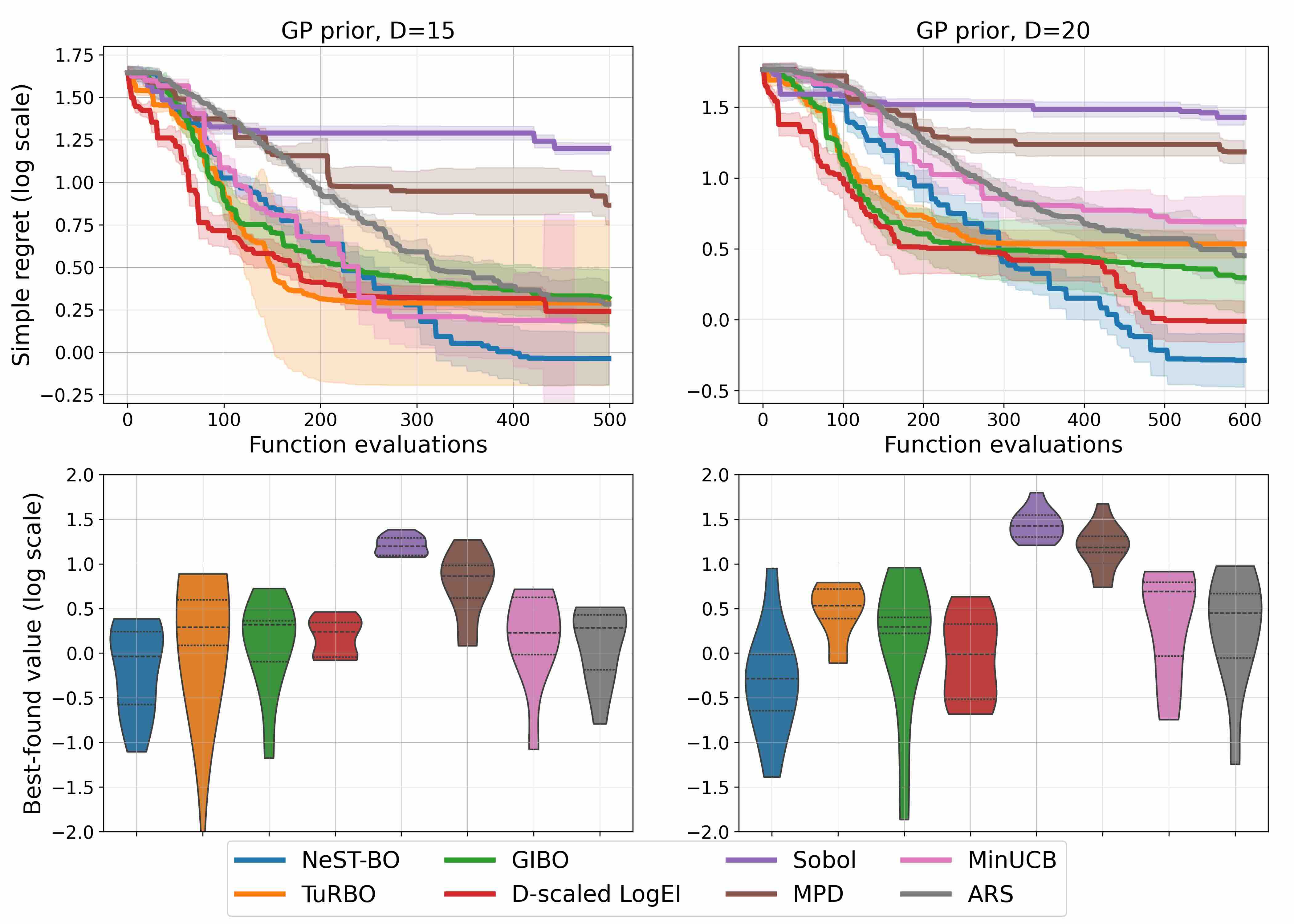}
  \caption{Optimization of GP-prior draws in $d\in\{15,20\}$ without subspaces. \textbf{Top:} Median simple regret versus number of function evaluations (shaded region corresponds to $\pm$ one standard error). \textbf{Bottom:} Distribution of best-found values across 10 realizations. NeST-BO converges faster and to lower regret than strong baselines on these medium-dimensional tasks.}
  \label{fig:GP-prior}
\end{figure}

\subsection{Alternative Ways to Learn Embeddings}

Our main results show that subspaces can be a powerful vehicle for scaling NeST-BO to high-dimensional spaces (e.g., using BAxUS-style nested subspaces). In this section, we ask a complementary question: \emph{is NeST-BO also compatible with other ways of constructing subspaces?} To answer this, we consider an adaptively learned embedding obtained using the Sparse Axis-Aligned Subspace GP (SAAS-GP) \citep{eriksson2021high}. 
In particular, we fit a SAAS-GP once at the start to select an active set of coordinates -- those whose posterior mean length-scales fall below a threshold $\gamma$ -- and then run NeST-BO \textit{only} in this learned subspace while holding the remaining coordinates fixed at their incumbent values. We denote this variant \textbf{NeST-BO-SAAS}. To assess the importance of the Newton step itself, we also run \textbf{GIBO-SAAS}, which uses the identical SAAS subspace but follows a gradient-based step rule.

We evaluate on a two-dimensional Branin function\footnote{See \url{https://www.sfu.ca/~ssurjano/branin.html}} embedded in $50$ dimensions; we use a threshold $\gamma = 10$, initialize with 30 Sobol points, and report results over 10 independent runs. We intentionally exclude BAxUS here to avoid confounding the question of \textit{which} subspace to use with \textit{how} the local step is computed inside that subspace.
As shown in Figure \ref{fig:Branin}, \textbf{NeST-BO-SAAS} delivers a sharp reduction in simple regret and substantially outperforms both \textbf{GIBO-SAAS} and the non-subspace baselines on this benchmark. This suggests that once a reasonably informative subspace is available, even from a simple axis-aligned selector, the curvature-aware Newton step provides a potentially large advantage over gradient-only updates.

\begin{figure}[!tb]
  \centering
  \includegraphics[width=0.9\textwidth]{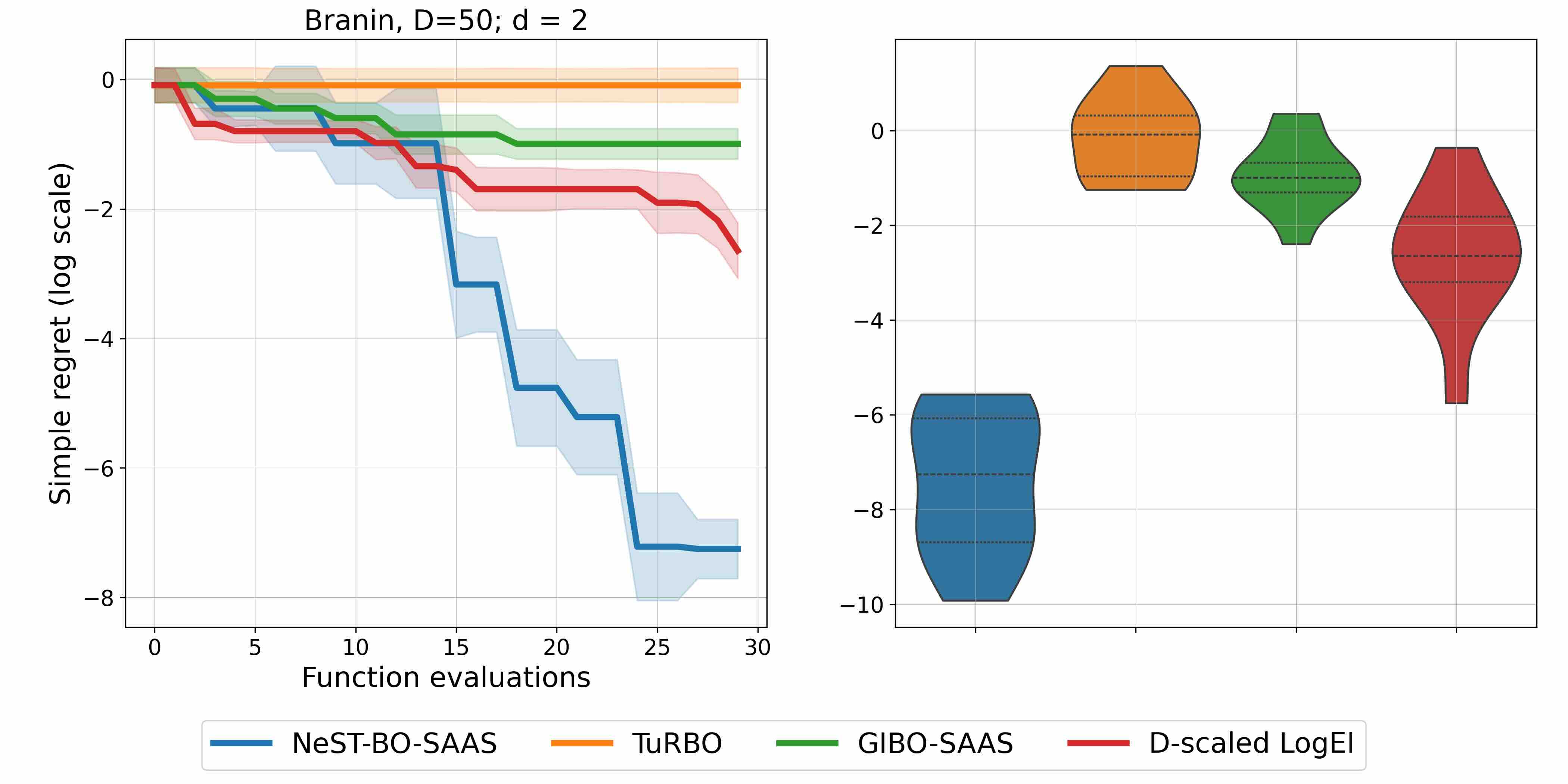}
  \caption{Branin embedded in $D = 50$ dimensions with a learned SAAS subspace ($d{=}2$ active dimensions); shading = $\pm$ one standard error.
  \textbf{Left:} Median simple regret (log scale) over 10 runs.
  \textbf{Right:} Distribution of the best simple regret across runs.
  NeST-BO-SAAS uses a one-time SAAS-GP to select active coordinates, then runs NeST-BO only in that subspace; GIBO-SAAS uses the identical subspace but uses gradient-only steps.}
  \label{fig:Branin}
\end{figure}

\subsection{Impact of Batch Size on NeST-BO}
\label{app:ablation-batch}

A defining feature of NeST-BO is its explicit use of gradient and Hessian information to construct a local Newton step. This creates a natural trade‑off: in each iteration we can either devote more samples to accurately estimating the step, or spend fewer samples per step and move on more quickly. In other words, the batch size $b$ controls how well the Newton direction is learned relative to how frequently it can be updated. Intuitively, very small $b$ risks moving along a poorly estimated (heavily biased or noisy) direction, which can slow convergence or even push the search off course; large $b$ improves the estimation quality but consumes budget so quickly that only a few iterations of actual movement occur.

To examine this tradeoff, we ran NeST‑BO on two standard $d = 10$ benchmarks -- Griewank\footnote{See \url{https://www.sfu.ca/~ssurjano/griewank.html}} and Ackley\footnote{See \url{https://www.sfu.ca/~ssurjano/ackley.html}} -- using three batch sizes: $b\in\{0.2d,\,d,\,2d\}$. Each run started from 10 Sobol points and was repeated 10 times. This setup lets us isolate how the per-iteration sampling budget influences the quality of the learned Newton step and the resulting optimization trajectory.
The results in Figure~\ref{fig:ablation_batchsize} reveal a clear pattern. When the batch size is too small ($b=0.2d$), the algorithm learns an imprecise Newton direction: simple regret decreases slowly and often plateaus at higher values. In contrast, moving from $b=d$ to $b=2d$ produces only marginal gains in early-iteration slope but nearly identical final performance, indicating diminishing returns once the local power functions for $\bs{g}$ and $\bs{H}$ are already reasonably small. In practice, this suggests that, beyond a moderate batch size, further increasing $b$ yields only a slight benefit in direction accuracy while substantially reducing the number of outer iterations.

Overall, these experiments show that NeST‑BO benefits from a sufficiently large batch size to accurately estimate the Newton step but does not require very large batches to converge effectively. This supports our default choice of $b=d$ in the main experiments as a balanced setting between step-accuracy and iteration budget.

\begin{figure}[!tb]
  \centering
  \includegraphics[width=1.0\textwidth]{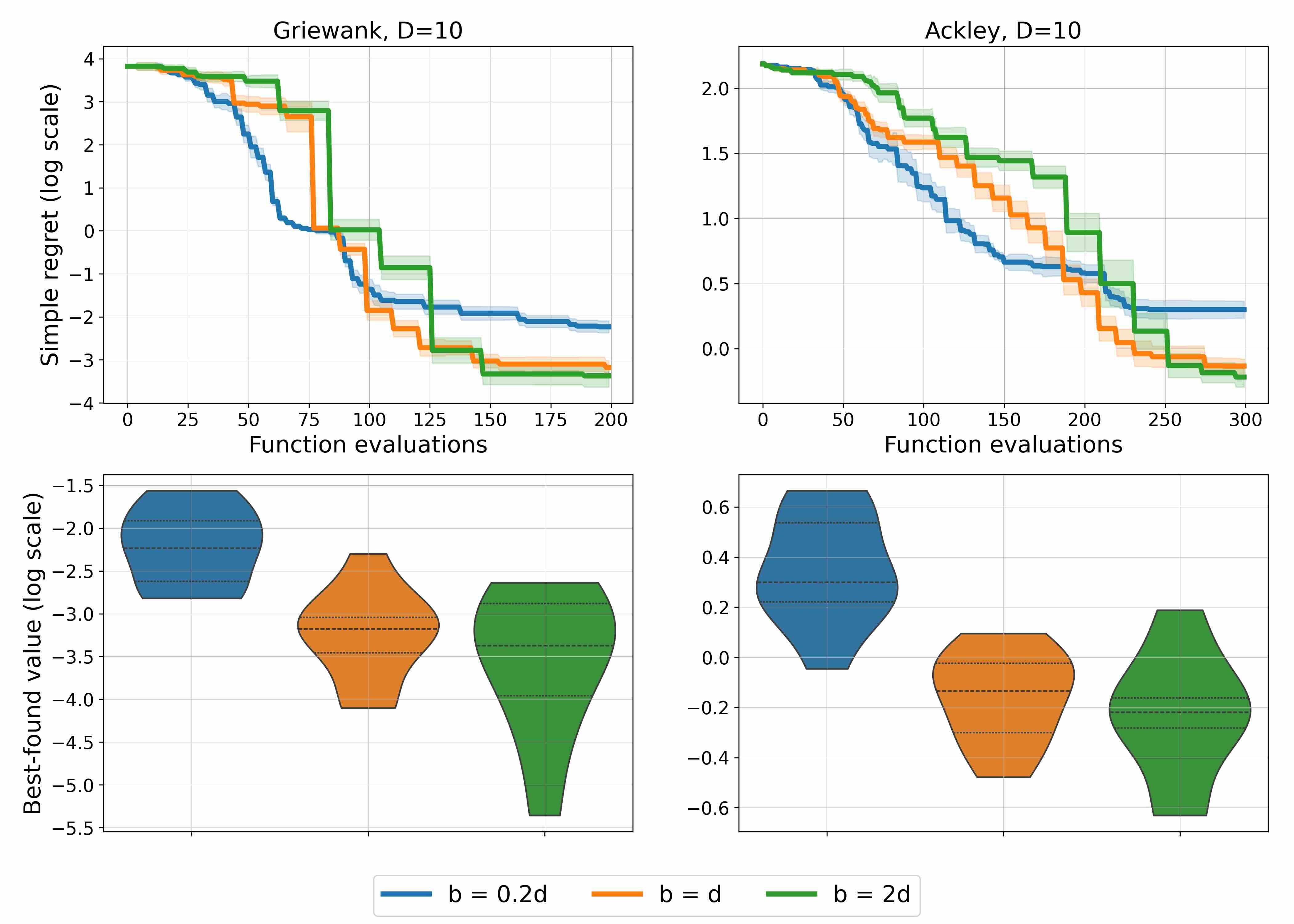}
  \caption{Effect of sampling budget on Newton‑step learning. NeST‑BO with $b \in \{0.2d,d,2d\}$ on Griewank and Ackley ($d = 10$). \textbf{Top:} Median simple regret (log scale) across 10 runs, shading = $\pm$ one standard error. \textbf{Bottom:} Distribution of best‑seen values. Small $b$ underestimates the step and slows convergence; $b = d$ and $b = 2d$ give comparable performance, indicating diminishing returns beyond $b \approx d$.}
  \label{fig:ablation_batchsize}
\end{figure}

\subsection{Impact of Step Size on GIBO}

First-order local Bayesian optimization methods (such as GIBO) build steps using only gradient information. This raises an important question: \textit{how sensitive is their performance to the choice of step size?} A step that is too aggressive can overshoot narrow valleys or oscillate around the optimum, while a step that is too conservative may crawl slowly toward the solution. In contrast, NeST-BO augments gradient information with curvature and employs an automatic backtracking line search, which potentially makes it less sensitive to such manual tuning.

To examine this issue, we compared NeST-BO with GIBO on the four-dimensional Rosenbrock function\footnote{See \url{https://www.sfu.ca/~ssurjano/rosen.html}}, a classical ill-conditioned test problem. GIBO used their length-scale-normalized gradient update with fixed step sizes $\eta \in \{1.0,\,0.5,\,0.1\}$. NeST-BO used its default line search. All methods started from 10 Sobol points, and we averaged results over 10 independent runs to reduce sensitivity to the initial data.

The results in Figure~\ref{fig:Rosenbrock_stepsize} highlight a striking difference. GIBO’s performance is highly step-size dependent: with $\eta = 0.1$, it converges relatively well, but with larger steps ($\eta = 0.5$ or $1.0$), the algorithm's performance deteriorates, often stalling or oscillating near Rosenbrock’s curved valley. This behavior is consistent with overshooting under ill-conditioned curvature. NeST-BO, by contrast, maintains steady progress without any step-size tuning, leveraging its line search and curvature scaling to automatically adjust the step length. Even on this challenging landscape, NeST-BO achieves competitive regret compared with the best-tuned GIBO setting.
Overall, this study underscores the practical advantage of NeST-BO’s Newton-based update: by removing the need for manual step-size selection, it improves robustness and reduces the burden of hyperparameter tuning relative to first-order local BO methods like GIBO.

\begin{figure}[!tb]
  \centering
  \includegraphics[width=0.9\textwidth]{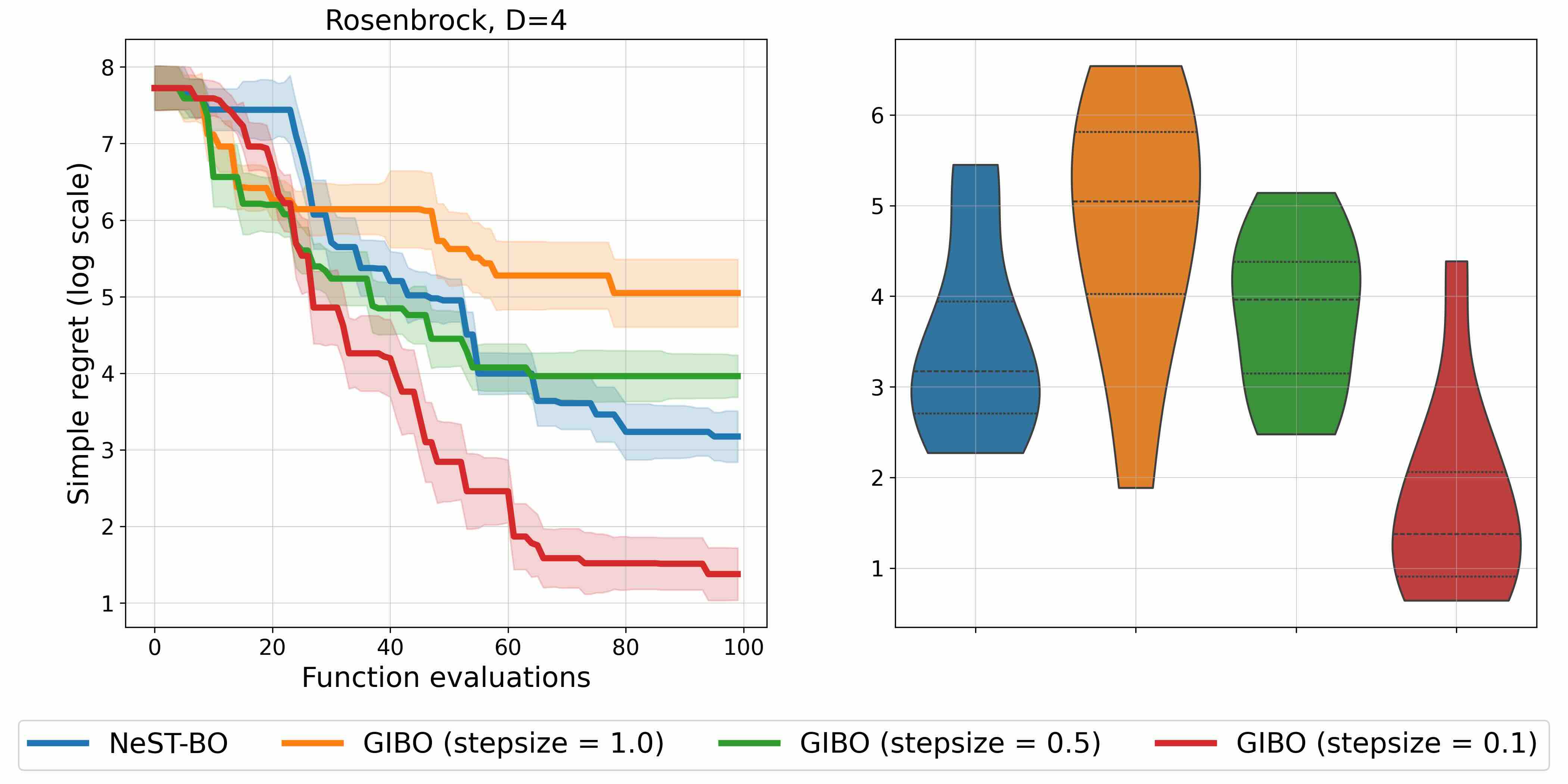}
  \caption{Step-size sensitivity of GIBO on Rosenbrock ($d=4$). \textbf{Left:} Median simple regret (log scale) across 10 runs; shading = $\pm$ one standard error. \textbf{Right:} Distribution of best-seen values. GIBO requires careful step-size tuning ($\eta = 0.1$ works best here); larger $\eta$ harms stability. NeST-BO’s line search removes this sensitivity while leveraging curvature.}
  \label{fig:Rosenbrock_stepsize}
\end{figure}

\subsection{Runtime Comparison}

In this section, we analyze the runtime of NeST-BO compared to GIBO and D-scaled LogEI under controlled conditions. 
To perform a fair comparison, we run each method on the same shared CPU cluster (see Appendix \ref{app:computing-resource}) limited to 4 cores and evaluate the cumulative time required to complete 200 iterations, which is reported in Table \ref{tab:runtime-comparison}. Each method is initialized with 10 Sobol points, and we report average CPU time across 10 independent replicates on both 12- and 20-dimensional test functions.

\begin{table}[!tb]
\centering
\caption{Average cumulative CPU time (in seconds) to complete 200 BO iterations across 10 replicates for various algorithms. Standard deviation across replicates is shown in parentheses.}
\vspace{0.5em}
\begin{tabular}{lcc}
\hline
\textbf{Method} & \textbf{12-dim function} & \textbf{20-dim function} \\
\hline
NeST-BO & 167.4 (14.2) & 260.1 (16.3) \\
GIBO & 138.2 (12.5) & 183.6 (15.7) \\
D-scaled LogEI & 150.7 (10.2) & 189.6 (15.7) \\
\hline
\end{tabular}
\label{tab:runtime-comparison}
\end{table}

Empirically, we find that NeST-BO is modestly more expensive than the other methods, with runtime increases of roughly 15--40\% compared to GIBO depending on the dimension. This difference is expected: NeST-BO must construct and invert GP posteriors over gradient and Hessian quantities during each update, and evaluating the posterior variance of the Newton step is the most costly step. 
In contrast, GIBO uses only first-order information and LogEI computes acquisition values from scalar posteriors.

Despite these additional costs, we argue that the tradeoff is often worthwhile. First, NeST-BO consistently provides lower regret than the alternatives across synthetic and real-world benchmarks, as shown throughout the main paper and Appendix. Second, the marginal CPU time increase is negligible in most practical BO applications, where each black-box evaluation may take minutes to hours (or longer). Finally, the current NeST-BO implementation uses exact kernel derivatives with dense covariance matrices and no real numerical acceleration. We believe this leaves ample room for improvement, especially if recent advances in scaling GPs with derivatives (for certain common kernel classes) are leveraged, e.g., \citep{de2021high}

\end{document}